\documentclass[aoas]{arxiv}

\RequirePackage{amsthm,amsmath,amsfonts,amssymb}
\RequirePackage[authoryear]{natbib}
\RequirePackage[colorlinks,citecolor=blue,urlcolor=blue]{hyperref}
\RequirePackage{enumerate}
\RequirePackage{bbm}
\usepackage{graphicx}
\RequirePackage{dsfont}
\RequirePackage{mathtools}
\usepackage{booktabs}
\usepackage{subfig}
\RequirePackage[ruled,linesnumbered]{algorithm2e}
\SetKwProg{Init}{Initialization:}{}{}

\startlocaldefs
\theoremstyle{plain}
\newtheorem{theorem}{Theorem}[section]
\newtheorem{lemma}{Lemma}[section]

\newtheorem{proposition}[theorem]{Proposition}
\theoremstyle{remark}
\newtheorem{definition}{Definition}[section]
\newtheorem{assumption}{Assumption}
\newtheorem{remark}{Remark}[section]
\allowdisplaybreaks
\endlocaldefs
\usepackage{xcolor}
\newcommand{\xueqing}[1]{
   {\color{black}  {#1}}
  }
\begin{document}

\begin{frontmatter}
\title{Thompson sampling for zero-inflated count outcomes with an application to the Drink Less mobile health study}
\runtitle{Thompson sampling for count outcomes}

\begin{aug}
\author[A]{\fnms{Xueqing}~\snm{Liu}\ead[label=e1]{xueqing\_liu@u.duke.nus.edu}},
\author[B]{\fnms{Nina}~\snm{Deliu}\ead[label=e2]{nina.deliu@uniroma1.it
}},
\author[C]{\fnms{Tanujit}~\snm{Chakraborty}\ead[label=e3]{tanujit.chakraborty@sorbonne.ae}},
\author[D]{\fnms{Lauren}~\snm{Bell}\ead[label=e4]{L.M.Bell@leeds.ac.uk}},
\and
\author[A]{\fnms{Bibhas}~\snm{Chakraborty}\ead[label=e5]{bibhas.chakraborty@duke-nus.edu.sg}}

\address[A]{Centre for Quantitative Medicine, Duke-NUS Medical School \printead[presep={,\ }]{e1,e5}}

\address[B]{Department of Methods and Models for Economics, Territory and Finance, Sapienza University of Rome \printead[presep={,\ }]{e2}}


\address[C]{Department of Science and Engineering, Sorbonne University Abu Dhabi  \printead[presep={,\ }]{e3}}


\address[D]{Clinical Trials Research Unit, Leeds Institute of Clinical Trials Research, University of Leeds \printead[presep={,\ }]{e4}}


\end{aug}

\begin{abstract}
Mobile health (mHealth) interventions often aim to improve distal outcomes, such as clinical conditions, by optimizing proximal outcomes through just-in-time adaptive interventions. Contextual bandits provide a suitable framework for customizing such interventions according to individual time-varying contexts. However, unique challenges, such as modeling count outcomes within bandit frameworks, have hindered the widespread application of contextual bandits to mHealth studies. The current work addresses this challenge by leveraging count data models into online decision-making approaches. Specifically, we combine four common offline count data models (Poisson, negative binomial, zero-inflated Poisson, and zero-inflated negative binomial regressions) with Thompson sampling, a popular contextual bandit algorithm. The proposed algorithms are motivated by and evaluated on a real dataset from the Drink Less trial, where they are shown to improve user engagement with the mHealth platform. The proposed methods are further evaluated on simulated data, achieving improvement in maximizing cumulative proximal outcomes over existing algorithms. Theoretical results on regret bounds are also derived. The \texttt{countts} R package provides an implementation of our approach.
\end{abstract}


\end{frontmatter}

\section{Introduction}
\label{sec:intro}
Mobile health (mHealth) interventions use mobile devices to manage health and wellness by continuously observing contextual data such as location, step count, or other real-time individual information. The just-in-time adaptive intervention (JITAI) is a unique intervention design for delivering mHealth support, which can customize the content and delivery of support based on a user's changing state and environment \citep{nahum-shani2018}. This intervention design can be conceptualized as an online decision-making problem, where the mobile phone delivers an intervention based on contextual information and monitors a proximal outcome at each decision time. The aim is to improve an ultimate distal outcome, such as reducing alcohol consumption, by maximizing the cumulative proximal outcome, such as engagement with an alcohol reduction app \citep{bell2020notifications}. 

Contextual bandits provide an ideal framework for customizing JITAIs, where an agent continually observes users' real-time contexts and selects from a range of actions or interventions based on available knowledge to maximize their cumulative proximal outcome or reward over a given time horizon \citep{tewari2017a}. To achieve this, the agent must balance between exploiting the actions associated with the current highest expected outcomes and exploring other unexplored actions that may further improve the expected outcomes \citep{dumitrascu2018}.

The Thompson sampling (TS) approach, originally proposed by \citet{thompson1933likelihood} in the context of clinical trials, represents a Bayesian strategy for addressing the above exploration-exploitation trade-off. The underlying principle of TS is to select an action according to its posterior probability of being optimal, that is, leading to the highest expected outcome. Several empirical studies revealed that TS outperforms different bandit alternatives in various simulated and real data settings \citep{scott2010modern,chapelle2011a}. Recent theoretical research has also proved regret bounds for TS methods under linear bandit, generalized linear bandit, and full reinforcement learning settings \citep{agrawal2013thompson, abeille2017, russo2014learning, russo2016information}. Moreover, TS-based algorithms have been used in mHealth and clinical studies, including DIAMANTE \citep{aguilera2020a}, HeartSteps II \citep{spruijt-metz2022}, and StratosPHere 2 \citep{deliu2023stratosphere}. Improved TS algorithms have also been developed to properly personalize JITAIs in mHealth settings \citep{liao2020a, tomkins2021intelligentpooling, kumar2023using, NIPS2017_4fa177df, huch2023debiased}. 

Contextual bandits have gained popularity in mHealth applications, but their use requires careful attention to modeling count outcomes in an online setting, which has not been sufficiently addressed. Many mHealth studies collect proximal outcomes as counts, e.g., the number of dialogue steps \citep{beinema2022} and the number of screen views within an hour following a notification \citep{bell2020notifications}. Count models may enhance the personalization of JITAIs in such instances. Poisson distribution, assuming equal mean and variance, is frequently used in count data analysis. However, this assumption is often violated in real data with overdispersion, where the variance is greater than the mean, and with zero-inflation phenomena \citep{lambert1992zero}. Notably, overdispersion and zero-inflation can result in an invalid estimation if not appropriately addressed \citep{mccullagh2019generalized}. 

In offline statistical analysis, several approaches have been developed to address these issues \citep{payneApproachesDealingVarious2017, feng2021comparison}.  Among the alternatives to Poisson regression, the negative binomial (NB) regression is frequently utilized, especially given its variance being a quadratic function of the mean, thereby facilitating the modeling of overdispersed count data \citep{ver2007quasi}. In scenarios with excess zeros, the zero-inflated Poisson (ZIP) regression provides a robust framework. This model posits a latent mixture of a zero component and a Poisson component, i.e., with probability $p$ the only possible observation is $0$, and with probability $1-p$, a Poisson random variable is observed \citep{lambert1992zero}.  To address additional overdispersion in the Poisson component, the zero-inflated negative binomial (ZINB) model is employed \citep{garay2011estimation}. Although these advancements have significantly improved the analysis of count data in offline settings, the adaptation and application of these models to online learning frameworks are not well studied.

The current paper aims to address the present lack of effective algorithms for contextual bandits with count outcomes in mHealth. One potential approach is to extend the existing generalized linear bandit framework to account for count outcomes \citep{filippi2010,li2017provably,kveton2020randomized}. However, these studies mainly focus on logistic bandits and lack discussion on count data. Moreover, many count regressions, such as negative binomial regression with an unknown dispersion parameter, zero-inflated Poisson regression, and zero-inflated negative binomial regression, do not belong to the generalized linear model framework \citep{cameron2013regression}. \xueqing{
In the mHealth domain, the problem of zero inflation has been tackled in \citet{trella2022designing,trella2023reward}, where a contextual Bayesian Poisson model is investigated in comparison to a Bayesian linear model (of a similar flavor to TS). However, their work focuses on the practical implementation of the algorithm with no insights on its theoretical properties and overdispersion is not taken into account.}

In this work, both Bayesian and frequentist regret bounds have been derived to assess the theoretical performances of the proposed algorithms. To evaluate their empirical performances, we have conducted simulation studies under a wide range of scenarios. In addition, we have applied our approaches to the \textit{Drink Less} data, obtained from a micro-randomized trial (MRT) designed to optimize the content and sequencing of push notifications in a behavior change app that helps people reduce alcohol consumption \citep{bell2020notifications}. Contextual bandits can enhance the utility of an MRT in several ways. For example, by adapting randomization probabilities in favor of the most effective interventions, they can offer a benefit to current users and directly translate clinical trial efficacy to real-world effectiveness by continuously learning from accumulating data. 

The structure of this article is as follows. Section \ref{data} details our motivating example, the Drink Less study, while Section \ref{prem} outlines the problem setup and explores various count models. Section \ref{method} introduces the TS-Count strategy, which integrates four count models, and discusses the approximation method for posterior sampling as well as the metric used to evaluate their performances.  Theoretical results supporting our proposed algorithms are provided in Section \ref{theory}. Section \ref{simu} assesses the performance of these algorithms via simulations, and Section \ref{app} illustrates their application using the Drink Less MRT dataset. Finally, Section \ref{conc} concludes the article with a comprehensive discussion.

\section{Motivating Example: the Drink Less Study}
\label{data}
Our work is motivated by a behavior change app called Drink Less. This app targets UK adults drinking alcohol at increasing and higher risk levels in reducing hazardous and harmful alcohol consumption  \citep{garnett2019development,garnett2021refining}. The study is based on an MRT design: a useful experimental design that provides high-quality data to optimize the timing, content, and sequencing for mHealth interventions, intending to sustain engagement \citep{klasnja2015microrandomized,qian2022microrandomized,liu2023microrandomized}. 
A 30-day MRT involving 349 participants was conducted to refine the push notification strategy for the Drink Less app \citep{bell2020notifications,bell2023}. Participants with a baseline Alcohol Use Disorders Identification Test (AUDIT) score of 8 or higher \citep{bohn1995alcohol,saunders1993development}, who resided in the UK, were at least 18 years old, and desired to drink less, were recruited into the trial. During the MRT, participants were randomly assigned each day at 8 p.m. to either receive a push notification (with a probability of 60\%) or not (with a probability of 40\%). The study collected both baseline and time-varying contextual data, such as age, AUDIT score, gender, and occupation type.  An exploratory proximal outcome was the overall number of screen views from 8 p.m. to 9 p.m., serving as a count measure of the depth of engagement with the app. Figure \ref{fig:normality} illustrates the empirical distribution and the corresponding Q-Q plot for this outcome, defined as the number of screen views recorded between 8 p.m. and 9 p.m. following the receipt of a push notification or no intervention. The analysis reveals that the proximal outcome deviates from a normal distribution, characterized notably by a high frequency of zero values. Consequently, this article introduces innovative contextual bandit algorithms designed to optimize the delivery of interventions with count outcomes.

\begin{figure}[ht]
	\centering
    \subfloat[Histogram]
{\includegraphics[width=10cm,height=6cm,keepaspectratio]{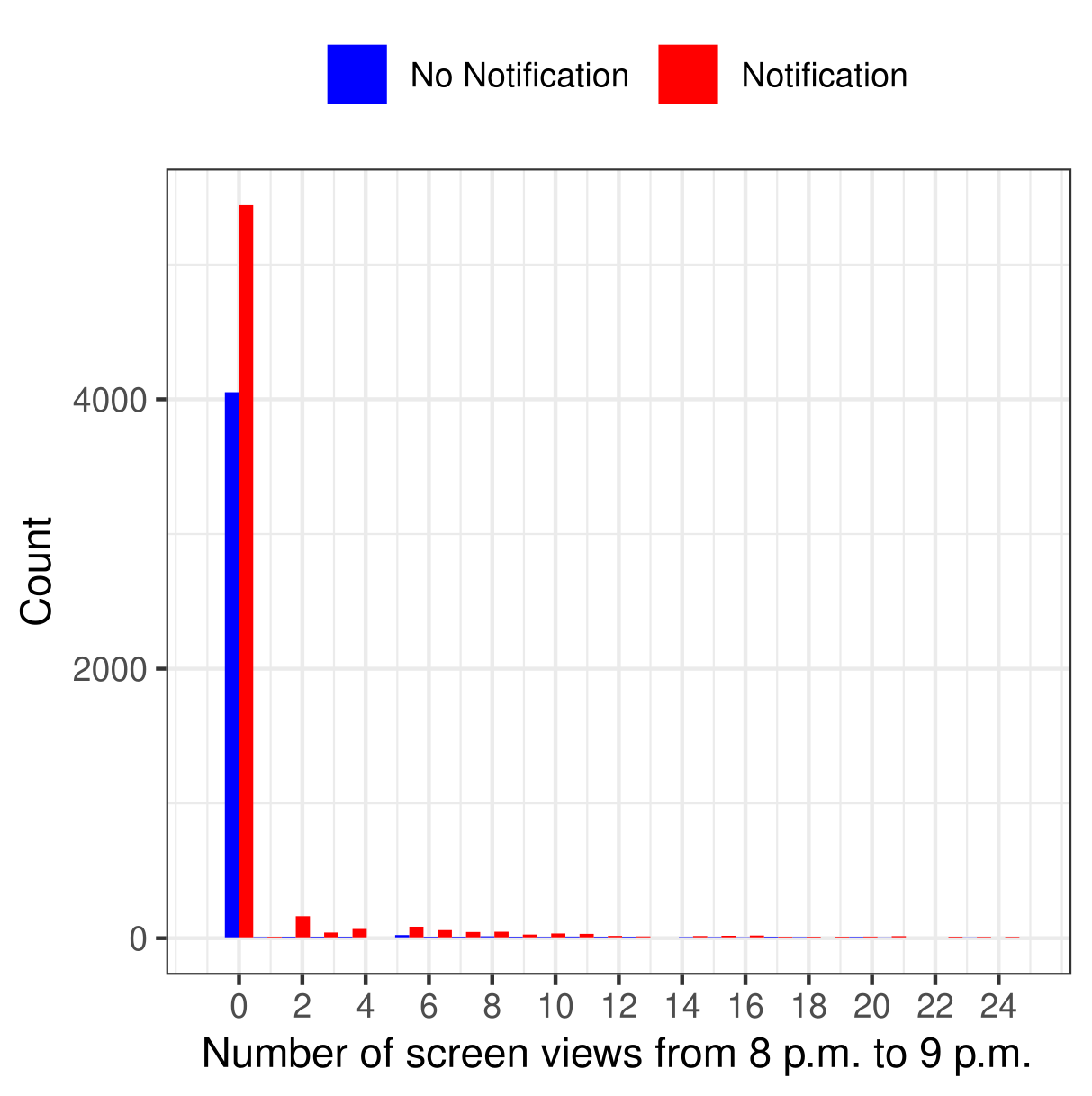}} 
	\subfloat[Q-Q plot]
 {\includegraphics[width=10cm,height=6cm,keepaspectratio]{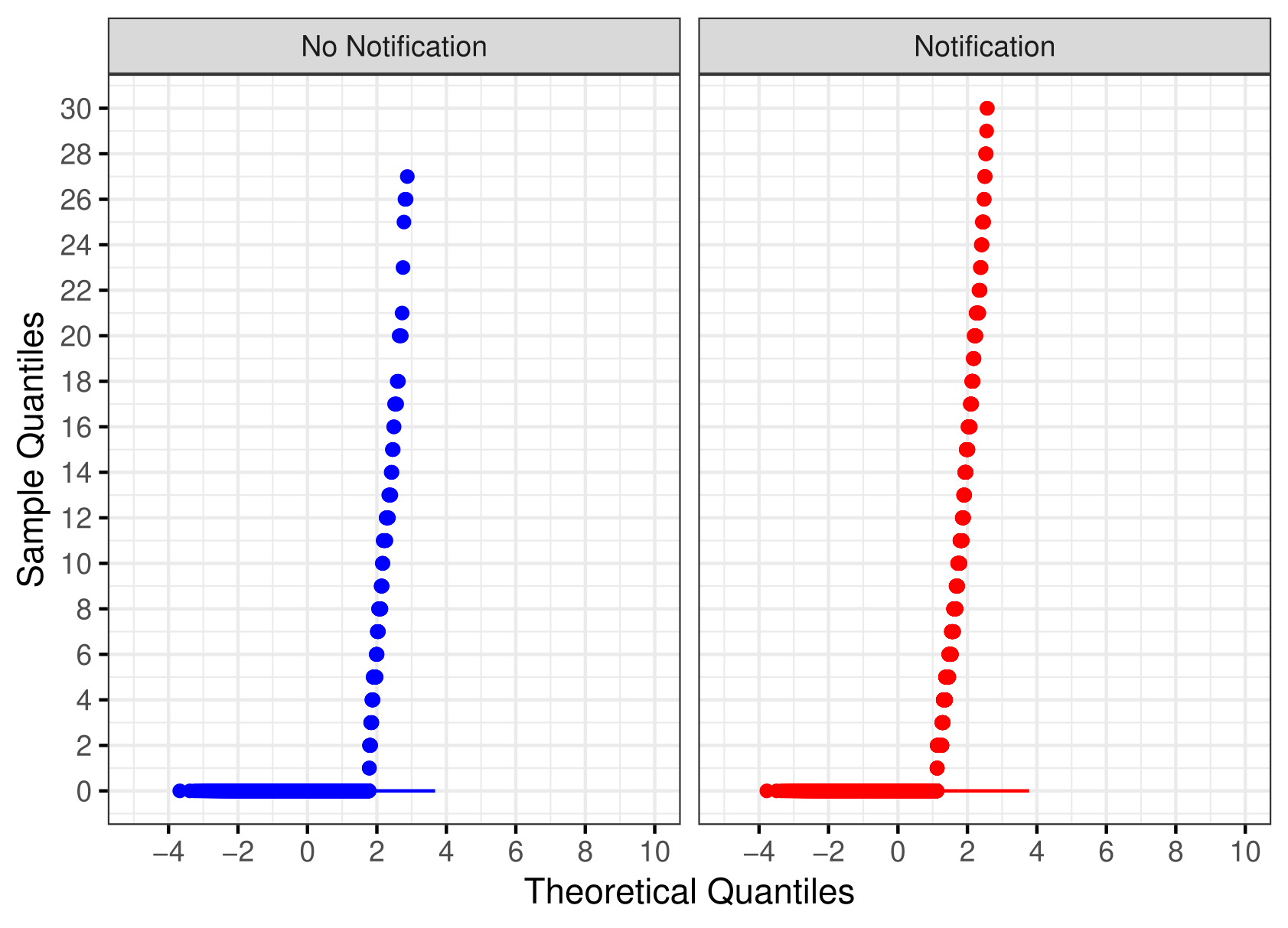}} 
	\caption{Graphical inspection of the distribution and normality check of the number of screen views from 8 p.m. to 9 p.m. after each intervention.}
	\label{fig:normality}
\end{figure}

\section{Preliminaries}
\label{prem}
\subsection{Notations and Problem Setup}
We consider the setting where the learning algorithm embedded in a mobile app must decide whether to deliver an intervention and what type of intervention to provide to a given user over a predefined follow-up period. This period is defined on a time horizon $T$, with the decision occurring at discrete time points $t= 1,2,..., T$. Let $\mathcal{A}$ denote the discrete set of available interventions at any time $t$ with cardinality $|\mathcal{A}|=K$. We denote the personal and contextual information of an individual at time $t$ by $X_t $. In the Drink Less study, for example, $X_t$ represents the real-time information of a user, 
\xueqing{typically selected based on domain knowledge.}  After an intervention $A_t \in \mathcal{A}$ is delivered, we collect a stochastic proximal outcome $Y_{t}$. We assume that $Y_t$ is count data, i.e., non-negative integers, with potential overdispersion and zero-inflation throughout the analysis. \xueqing{Further, let $\phi(A_t,X_t) \in \mathbb{R}^d$, with $d\geq 1$, denote the feature vector, which is a pre-specified function of the intervention $A_t$  and the context $X_t$. Typically, it includes the context, interventions, and interactions between intervention and context (see Section \ref{app} for further illustration). }


As the environment for generating the proximal outcomes is unknown, the algorithm within the mobile app must estimate the expected outcome for each intervention. This presents a trade-off between exploiting the best intervention so far versus exploring alternative interventions that might yield better results. Contextual bandits offer a framework for addressing this trade-off in online decision-making problems \citep{tewari2017a}. At each decision time $t$, contextual bandit algorithms utilize historical observations, denoted as $D_{t} = \left\{X_i, A_i, Y_i; i=1,\cdots, t-1\right\}$, to update the estimates of model parameters using approaches such as maximum likelihood estimation (MLE) method. A naive strategy is the \textit{greedy} algorithm \citep{sutton2018reinforcement}, which always exploits by selecting the intervention with the highest estimated proximal outcome at each decision time $t$.
However, this can lead to unreliable algorithms that do not explore sufficiently and get stuck at suboptimal interventions. There are two alternative bandit approaches recognized for their ability to deal with this issue: the upper confidence bound method, which relies on ``optimism in the face of uncertainty'' \citep{filippi2010,li2017provably}, and the TS approach, rooted in the Bayesian framework \citep{russo2018tutorial}. We have chosen to implement the TS strategy for two reasons: first, it has a valid theoretical guarantee and excellent empirical performance \citep{agrawal2013thompson, abeille2017, russo2014learning, russo2016information}; second, it is a randomized strategy that assigns probabilities to each intervention option, making it suitable for MRTs and allowing for valid post-study causal inference \citep{liao2020a, liu2023incorporating}.


\subsection{Count Data Models: A Review}
In this section, we provide a brief overview of two widely used regression models for count data, namely, Poisson regression and NB regression. In addition, we also discuss two variations of these basic models, namely, ZIP regression and ZINB regression, which are particularly useful in cases with zero-inflation.

\subsubsection{Poisson Regression}
First, let us consider a Poisson regression model. 
Mathematically, the conditional distribution of $Y_t$ can be written as 
\begin{eqnarray*}
    \operatorname{Pr}\left(Y_t= y | A_t, X_t \right)=\frac{\exp (-\mu_{t}) \mu_{t}^y}{y !}, \quad y=0,1,2, \ldots
\end{eqnarray*}
where $\mu_{t} = \exp(\phi(A_t,X_t)^{\top}\beta^*)$ and $\beta^*\in \mathbb{R}^d$ denotes the true model parameter. 

The parameter $\beta^*$ can be estimated by maximizing the log-likelihood function given by
\begin{eqnarray}
\label{lik_pois}
\ln L^{\text{Poisson}}(\beta) = \sum_{i=1}^{t-1} \left[-\exp \left(\phi(A_i,X_i)^{\top} \beta\right)+\left(\phi(A_i,X_i)^{\top} \beta\right) Y_i-\ln \left(Y_i !\right)\right].
\end{eqnarray}

It is also worth noting that the only condition for consistency is the correct specification of the conditional mean model \citep{gourieroux1984pseudo}. However, due to the common issues with count data, i.e., overdispersion and zero-inflation, other models may be preferred over the basic Poisson regression. 

\subsubsection{Negative Binomial Regression}
The NB model is a popular alternative to handle overdispersion in count data analysis. The NB model is defined by the probability mass function
\begin{eqnarray*}
     \operatorname{Pr}\left(Y_t= y | A_t, X_{t}\right) = \frac{\Gamma(Y_i+r)}{Y_i!\Gamma(r)} \left(\frac{\mu_{t}}{\mu_{t} + r}\right)^{Y_i} \left(1+\frac{\mu_{t}}{r}\right)^{-r},
\end{eqnarray*}
where $r>0$ denotes the inverse dispersion parameter and $\Gamma(\cdot)$ denotes the gamma function. The estimates of the regression coefficients $\beta^*$ can be obtained by maximizing the log-likelihood function given by
\begin{eqnarray}
\label{lik_nb}
\ln L^{\text{NB}}(\beta,r) &=& \sum_{i=1}^{t-1} \Bigg\{ \log \left[\Gamma\left(r + Y_i\right)\right]-\log \left[\Gamma\left(Y_i+1\right)\right]-\log [\Gamma(r)] \nonumber  \\
&+&  Y_i \log \left[\frac{\exp \left(\phi(A_i,X_i)^{\top} \beta\right)}{r+\exp \left(\phi(A_i,X_i)^{\top} \beta\right)}\right]+r \log \left[\frac{r}{r+\exp \left(\phi(A_i,X_i)^{\top} \beta\right)}\right] \Bigg\}. 
\end{eqnarray}

While the MLEs for the Poisson model can be obtained using the Newton-Raphson or the iteratively reweighted least square technique \citep{cameron2013regression}, for the NB model, additionally, the estimation of the overdispersion parameter needs to be incorporated into the iteration procedure.

\subsubsection{Zero-inflated Poisson Regression}
Count data with extra zeros is common in many real-world applications. The ZIP model can handle such data by combining a distribution that is degenerate at zero with a Poisson distribution. In ZIP regression, the outcome $Y_t$ has a probability mass function  given by
\begin{eqnarray*}
    \operatorname{Pr}\left(Y_t=y | A_t, X_t\right) = \begin{cases}p_{t}+\left(1-p_{t}\right)e^{-\mu_{t}}, & y=0 \\ \left(1-p_{t}\right) \frac{e^{-\mu_{t}} \mu_{t}^y}{y!}, & y=1,2, \ldots\end{cases}
\end{eqnarray*}
where $0 \leq p_{t} \leq 1$ represents the probability that the zero value comes from the zero state, and 
\begin{eqnarray*}
     \mu_{t}= \exp\left(\phi(A_t,X_t)^{\top} \beta^* \right)\text { and } p_{t}= \operatorname{sigmoid} \left(\phi(A_t,X_t)^{\top} \gamma^* \right).
\end{eqnarray*}
Here, $\gamma^* \in \mathbb{R}^d$ and $\beta^* \in \mathbb{R}^d$ denote the true model parameters. 

The log-likelihood function is given by 
\begin{eqnarray*}
    \ln L^{\text{ZIP}}(\beta,\gamma) &=& \sum_{i=1}^{t-1} \log \left[p_{i} + (1-p_{i})\exp(-\mu_{i})\right]I_{(Y_i = 0)} \nonumber \\
    &+& \sum_{i=1}^{t-1} \log \left[(1-p_{i})\frac{\exp(-\mu_{i}) \mu_{i}^{Y_i}}{Y_i!}\right]I_{(Y_i > 0)}.
\end{eqnarray*}

Typically the expectation-maximization (EM) algorithm is used to derive the MLEs, which is stable and straightforward. To this end, an additional variable $Z_t$ is needed to indicate which zeros came from the Poisson distribution and which came from the zero state. Specifically, $Z_t = 1$ when $Y_t$ is from the zero state and $Z_t = 0$ otherwise. The complete-data log-likelihood is
\begin{eqnarray}
\label{like_zip}
    \ln L^{\text{ZIP}}_{c}(\beta, \gamma) &=& \sum_{i=1}^{t-1} Z_i\phi(A_i,X_i)^{\top}\gamma - \log(1+\exp(\phi(A_i,X_i)^{\top}\gamma)) \nonumber \\
    &+& (1-Z_i) \left(-\exp \left(\phi(A_i,X_i)^{\top} \beta\right)+\left(\phi(A_i,X_i)^{\top} \beta\right) Y_i-\ln \left(Y_i !\right)\right).
\end{eqnarray}
The EM algorithm proceeds by iteratively alternating between estimating $Z_t$ by its expectation under the current estimates of $\beta$ and $\gamma$ (E step) and maximizing Equation \eqref{like_zip} with $Z_t$ fixed at the values from the E step (M step). We refer readers to \citet{lambert1992zero} for more details on the EM algorithm. 

\subsubsection{Zero-inflated Negative Binomial Regression}
Suppose the data suggests additional overdispersion after accounting for zero-inflation. In that case, it may be necessary to consider the ZINB model that mixes a distribution degenerate at zero with an NB distribution. This is a more flexible approach compared to the ZIP model. In ZINB regression, $Y_t$ has a probability mass function given by 
\begin{eqnarray*}
    \operatorname{Pr}\left(Y_t=y | A_t, X_t\right)= \begin{cases}p_{t}+\left(1-p_{t}\right)\left(\frac{r}{\mu_{t}+r}\right)^r, & y=0 \\ \left(1-p_{t}\right) \frac{\Gamma\left(r+y\right)}{\Gamma\left(y+1\right) \Gamma(r)}\left(\frac{\mu_{t}}{\mu_{t}+r}\right)^{y}\left(\frac{r}{\mu_{t}+r}\right)^r, & y=1,2, \ldots\end{cases}
\end{eqnarray*}
where $\mu_t$ and $p_t$ are defined as before. Similar to ZIP, one can derive the complete-data log-likelihood function of ZINB, which is given by
\begin{equation}
\label{lik_zinb}
    \ln L^{\text{ZINB}}_{c}\left(\beta, \gamma, r\right)=\sum_{i=1}^{t-1} \left\{Z_i \phi(A_i,X_i)^{\top}\gamma-\log (1+\exp \left(\phi(A_i,X_i)^{\top} \gamma\right))+\left(1-Z_i\right)  g\left(Y_i;\beta, r\right)\right\}
\end{equation}
where $g\left(Y_i;\beta, r\right)$ is the same as Equation \eqref{lik_nb}. We also use the EM algorithm to compute the MLEs; see \citet{garay2011estimation} for more details. 

\section{The TS-Count Strategy}
\label{method}
In this section, we present the TS-Count strategy for optimizing the delivery of mHealth interventions. The TS strategy is closely connected to Bayesian inference, which selects the optimal intervention according to its posterior probability of being optimal. Algorithm \ref{alg1} presents the pseudocode for TS-Count, specialized to the case of a Poisson or NB bandit, while Algorithm \ref{alg2} displays TS-Count in the context of a ZIP or ZINB bandit. Depending on the model used for $Y_t$, the strategy can be referred to as TS-Poisson, TS-NB, TS-ZIP, or TS-ZINB. 


\begin{algorithm}[ht]
\KwIn{prior distribution $Q_0(\beta)$.}
\For{$t=1,\cdots, T$}{
Update the posterior distribution $Q_t(\beta)$ using data $D_{t}$

Observe context $X_t$ 

Sample $\tilde{\beta}_t$ from the updated posterior distribution $Q_t(\beta)$

Compute $A_t = \arg \max_{a \in \mathcal{A}} \exp\left\{\phi(a,X_t)^{\top} \tilde{\beta}_t\right\}$

Offer intervention $A_t$ and observe the associated outcome $Y_t$
}
  \caption{TS-Count (Poisson / NB)}
  \label{alg1}
\end{algorithm}

\begin{algorithm}[ht]
\KwIn{prior distribution $Q_0(\gamma)$ and $Q_0(\beta)$.}
\For{$t=1,\cdots, T$}{
Update the posterior distribution $Q_t(\gamma)$ and $Q_t(\beta)$ using data $D_{t}$

Observe context $X_t$  

Sample $\tilde{\gamma}_t$ from the updated posterior distribution $Q_t(\gamma)$

Sample $\tilde{\beta}_t$ from the updated posterior distribution $Q_t(\beta)$

Compute $A_t = \arg \max_{a \in \mathcal{A}}\left[1-\operatorname{sigmoid}(\phi(a,X_t)^{\top} \tilde{\gamma}_t)\right] \exp\left\{\phi(a,X_t)^{\top} \tilde{\beta}_t\right\}$

Offer intervention $A_t$ and observe the associated outcome $Y_t$
}
  \caption{TS-Count (ZIP / ZINB)}
  \label{alg2}
\end{algorithm}

These algorithms operate through a sequential five-step process at each decision time $t$. First, a prior distribution $Q_0(\beta)$ is assumed for the underlying model parameter $\beta^*$.  Subsequently, the algorithm proceeds to update the posterior distribution $Q_t(\beta)$ by integrating the data collected up to time $t$, say $D_{t}$. Following this, the algorithm derives a Bayesian estimate $\tilde{\beta}_t$ for $\beta^*$ by sampling from the posterior distribution $Q_t(\beta)$. This process is grounded in the Bayesian framework, which treats the true parameter $\beta^*$ as a random variable; thus, at time $t$, both the Bayesian estimate $\tilde{\beta}_t$ and the model parameter $\beta^*$ are characterized by a posterior distribution $Q_t(\beta)$. The subsequent phase involves selecting an intervention $A_t$ that maximizes the estimated expected proximal outcome under the Bayesian estimate $\tilde{\beta}_t$. Finally, $A_t$ is offered to the user, leading to the observation of a proximal outcome $Y_t$.  It is important to note that for both TS-ZIP and TS-ZINB, an additional step is included wherein $\tilde{\gamma}_t$ is sampled from its respective posterior distribution to aid in the intervention selection process, thereby introducing an extra layer of randomness into the algorithm.

The proposed algorithms, which update the posterior distribution $Q_t(\gamma)$ and $Q_t(\beta)$ through a Bayesian framework, benefit significantly from prior assumptions about the underlying model parameter $\gamma^*$ and $\beta^*$. While a normal prior is frequently utilized in the estimation of regression coefficients due to its desirable properties \citep{vanzwet2019}, its application within count data models presents computational challenges. Specifically, the normal prior is non-conjugate for such models, thereby complicating the derivation of closed-form posterior distributions. This inherent complexity necessitates the adoption of advanced computational methods. Markov chain Monte Carlo (MCMC) techniques become indispensable in such scenarios to approximate the posterior distribution effectively. However, MCMC can be computationally demanding and rather time-consuming, particularly in online settings where the algorithm updates very frequently \citep{russo2018tutorial}. As such, we turn to more efficient methods for approximate sampling from posterior distributions.

\subsection{Laplace Approximation for Tractability}
In this section, we address the computational complexity associated with the posterior distributions in the TS-Count framework. Our approach involves adopting the \textit{Laplace approximation}, a technique that simplifies the posterior distribution by approximating it with a normal distribution \citep{russo2018tutorial}. As \citet{abeille2017} have indicated, when a distribution possesses appropriate \textit{concentration} and \textit{anti-concentration} characteristics, it ensures minimal regret in the approximation.  Specifically, the stochastic estimate $\tilde{\beta}_t$ is designed to strike a balance between being sufficiently far from the MLE $\hat{\beta}_t$ (anti-concentration) to ensure enough exploration, and being adequately concentrated around $\hat{\beta}_t$ (concentration) to ensure adequate exploitation. 

At each decision time $t$, a Bayesian estimate $\tilde{\beta}_t$ is sampled from a Gaussian distribution $\mathcal{N}(\hat{\beta}_t, \alpha_\beta^2 \mathcal{I}_t^{-1}(\hat{\beta}_t))$. Here, $\hat{\beta}_t$ denotes the MLE, and $\mathcal{I}_t(\hat{\beta}_t)$ represents the Fisher information, with $\alpha_\beta$ serving as a tuning parameter. Similarly, for our implementations of TS-ZIP and TS-ZINB, we additionally sample $\tilde{\gamma}_t$ from a Gaussian distribution $\mathcal{N}(\hat{\gamma}_t, \alpha_\gamma^2 \mathcal{I}_t^{-1}(\hat{\gamma}_t))$, where $\alpha_\gamma$ is also a tuning parameter. The pseudocode for applying TS-Poisson or TS-NB with Laplace approximation is detailed in Algorithm \ref{alg3}. In contrast, Algorithm \ref{alg4} presents the pseudocode of TS-ZIP or TS-ZINB with Laplace approximation. \xueqing{The implementation of Laplace approximation in these algorithms facilitates efficient sampling of the stochastic estimates $\tilde{\beta}_t$ and $\tilde{\gamma}_t$.} The tuning parameters, $\alpha_\beta$ and $\alpha_\gamma$, are critical in controlling the exploration-exploitation tradeoff and must be chosen carefully (see the proofs of Theorems \ref{regret-nb} - \ref{regret-zip} in the Supplement for more details).

\begin{algorithm}[ht]
\KwIn{Tuning parameter $\alpha_\beta$, initial random exploration time $\tau$}
\textbf{Initialization:}
Randomly choose $A_t \in \mathcal{A}$ for $t\leq \tau$ and collect the observations $D_{\tau+1}$

\For{$t=\tau+1,\cdots, T$}{

Compute the MLE $\hat{\beta}_{t}$ by maximizing the log-likelihood function using data $D_t$

\tcc{TS-Poisson: Equation \eqref{lik_pois}}
\tcc{TS-NB: Equation \eqref{lik_nb}}

Get $\mathcal{I}_t(\hat{\beta}_{t})$ as a byproduct of step 3

Sample $\tilde{\beta}_t$ from $\mathcal{N}(\hat{\beta}_t, \alpha_\beta^2\mathcal{I}_t^{-1}(\hat{\beta}_{t}))$

Observe context $X_t$  and compute $A_t = \arg \max_{a \in \mathcal{A}} \exp\left\{\phi(a,X_t)^{\top} \tilde{\beta}_t\right\}$

Offer intervention $A_t$ and observe the associated outcome $Y_t$
}
\caption{TS-Count with Laplace approximation (Poisson / NB)}
\label{alg3}
\end{algorithm}

\begin{algorithm}[ht]
\KwIn{Tuning parameters $\alpha_\beta$ and $\alpha_\gamma$, initial random exploration time $\tau$}
\textbf{Initialization:}
Randomly choose $A_t \in \mathcal{A}$ for $t\leq \tau$ and collect the observations $D_{\tau+1}$

\For{$t=\tau+1,\cdots, T$}{

Compute the MLEs $\hat{\gamma}_{t}$ and $\hat{\beta}_{t}$ by maximizing the log-likelihood function using data $D_t$

\tcc{TS-ZIP: Equation \eqref{like_zip}}
\tcc{TS-ZINB: Equation \eqref{lik_zinb}}

Get $\mathcal{I}_t(\hat{\gamma}_{t})$ and $\mathcal{I}_t(\hat{\beta}_{t})$ as a byproduct of step 3

Sample $\tilde{\gamma}_t$ from $\mathcal{N}(\hat{\gamma}_t, \alpha_\beta^2\mathcal{I}^{-1}_t(\hat{\gamma}_{t}))$

Sample $\tilde{\beta}_t$ from $\mathcal{N}(\hat{\beta}_t, \alpha_\gamma^2\mathcal{I}^{-1}_t(\hat{\beta}_{t}))$

Observe context $X_t$  and compute $A_t = \arg \max_{a \in \mathcal{A}}\left[1-\operatorname{sigmoid}(\phi(a,X_t)^{\top} \tilde{\gamma}_t)\right] \exp\left\{\phi(a,X_t)^{\top} \tilde{\beta}_t\right\}$

Offer intervention $A_t$ and observe the associated outcome $Y_t$
}
\caption{TS-Count with Laplace approximation (ZIP / ZINB)}
\label{alg4}
\end{algorithm}

We define $G_t=\sum_{i=1}^{t-1} \phi(A_i,X_i) \phi(A_i,X_i)^{\top}$ as the unweighted Hessian matrix at decision time $t$. TS-Count with normal approximation starts with an initial exploration time period $\tau$, which consists of delivering interventions $A_1,A_2,\cdots,A_{\tau}$ with uniform randomization probabilities. This can ensure the invertibility of $G_t$ for all $t \geq \tau$. The necessity is motivated by the classical likelihood theory: as the sample size $t \rightarrow \infty$, the MLE $\hat{\beta}_t$ is consistent and asymptotically normal, i.e., $\hat{\beta}_t -\beta^* \rightarrow \mathcal{N}(0, \mathcal{I}_t^{-1}(\beta^*))$. The invertibility of $G_t$ can ensure that $\mathcal{I}_{t}$ is invertible and a unique MLE is attainable. Proposition \ref{cor1} provides a way for theoretically specifying the length of the initial exploration time.

 \begin{proposition}
 \label{cor1}
Suppose $X_{t}$ is drawn i.i.d. from some distribution with support in the unit ball, i.e., $||X_{t}|| \leq 1$. Furthermore, let $\Sigma:=\mathbb{E}\left[\phi(A_t,X_t)\phi(A_t,X_t)^{\top}\right]$ be the second moment matrix, and $B$ be a positive constant. Then, there exist positive, universal constants $D_1$ and $D_2$ such that $\lambda_{\min }\left(G_{\tau}\right) \geq B$ with probability at least $1-T^{-1}$, as long as
$$
\tau \geq\left(\frac{D_1 \sqrt{d}+D_2 \sqrt{\log (T)}}{\lambda_{\min }(\Sigma)}\right)^2+\frac{2 B}{\lambda_{\min }(\Sigma)}.
$$
 \end{proposition}
This is an adaptation of Proposition 1 in \citet{li2017provably} and \citet{oh2019} to TS-Count, and the proof can be found in the same sources. Note that the i.i.d. assumption on $X_t$ is only required during the initial exploration period to ensure  $G_t$ is invertible. Subsequently, $X_t$ can be chosen adversarially, without assuming any structure on it, as long as $||X_{t}||$ is bounded. 

\xueqing{
\begin{remark}
    In terms of the Drink Less app, Proposition \ref{cor1} provides guidance on the number of initial observations required to ensure the validity of the learning algorithm for intervention delivery.  Typically, we select $\tau$ such that the minimum eigenvalue of $G_{\tau}$ satisfy $\lambda_{\min}(G_{\tau}) \geq 1$, aligning with setting $B=1$ in Proposition \ref{cor1}. Additionally, when designing a real mobile health study, it is important to ensure that the length of the study is longer than the required initial exploration length. This consideration needs to be taken into account before starting the study to ensure sufficient data collection beyond the exploration phase. 
    
    An alternative approach involves the use of regularization techniques. Specifically, we suggest modifying Line 3 in Algorithms \ref{alg3} and \ref{alg4} to include a ridge penalty in the objective function, $\lambda \|\beta\|_2^2$, mirroring the effect of a normal prior within a Bayesian context. This is of particular interest during the implementation of algorithms for the Drink Less case study, as discussed in Section \ref{app}. 
\end{remark}
}

\subsection{Performance Metrics}
To quantify the theoretical performances of the proposed contextual bandit algorithms, we introduce the following metrics common in the bandit literature.

\begin{definition}[Frequentist Regret] The frequentist regret of a learning algorithm is defined as the cumulative expected difference between the expected proximal outcome of the optimal intervention and that of the intervention chosen by the learning algorithm, i.e.,
    \begin{eqnarray*}
   \mathcal{R}(T,\eta^*) &=&\sum_{t=1}^T \mathbb{E}\left[\Delta_{A_t} | \eta^*\right] \\
   &=& \sum_{t=1}^T \mathbb{E}\left[\max_{a\in \mathcal{A}}h(\phi(a,X_t)^{\top};\eta^*) - h(\phi(A_t,X_t)^{\top};\eta^*) \Big| \eta^*\right],
   \end{eqnarray*}
    where $h(\phi(a,X_t)^{\top};\eta^*)$ denotes the expected proximal outcome for intervention $a$ at time $t$, $\eta^*$ denotes the vector of the true model parameters. For TS-Poisson and TS-NB, $h(\phi(a,X_t)^{\top};\eta^*) = \exp\left\{\phi(a,X_t)^{\top} \beta^*\right\}$, while for TS-ZIP and TS-ZINB, $h(\phi(a,X_t)^{\top};\eta^*) = \left[1-\operatorname{sigmoid}(\phi(a,X_t)^{\top} \gamma^*)\right] \allowbreak \exp\left\{\phi(a,X_t)^{\top} \beta^*\right\}$. Note that, the outer expectation is taken over random variables $A_t$ and $X_t$, which may depend on the randomness in rewards and in the algorithm.
\end{definition}

\begin{definition}[Bayesian Regret] The Bayesian regret of a learning algorithm is defined as $\mathbb{E}[\mathcal{R}(T,\eta^*)]$ 
where $\eta^*$ is considered as a random vector having a Bayesian flavor. Mathematically, we write
   \begin{eqnarray*}
    \mathcal{R}_{Bayes}(T) &=& \sum_{t=1}^T \mathbb{E}\left[\mathbb{E}\left[\Delta_{A_t}|\eta^*\right]\right] \\
    &=&  \sum_{t=1}^T \mathbb{E}\left[\mathbb{E}\left[\max_{a\in \mathcal{A}}h(\phi(a,X_t)^{\top};\eta^*)  -h(\phi(A_t,X_t)^{\top};\eta^*)\Big|\eta^*\right]\right],
    \end{eqnarray*}
   where the outer expectation is taken with respect to the prior distribution over $\eta^*$. 
\end{definition}

\xueqing{The concept of regret is central to measuring the performance of a contextual bandit algorithm. It quantifies the expected difference between the expected proximal outcome generated by the optimal intervention and the intervention selected by the learning algorithm. A lower regret value indicates a better performance of the learning algorithm, with zero regrets being achieved by an oracle algorithm that always selects the optimal intervention. 

Bayesian regret is less informative than its frequentist counterpart because a bound on $\mathcal{R}_{Bayes}(T)$ does not necessarily translate to a meaningful bound on $\mathcal{R}(T,\eta^*)$. However, the converse is true. Frequentist regret, by focusing on the worst-case or minimax regret, offers a robust measure that bounds the maximum possible regret for any specific instance of the problem, thus providing valuable insights into the algorithm's performance under the most challenging conditions \citep{agra2017}. }

\section{Regret Analysis}
\label{theory}
In this section, we present the theoretical guarantees, i.e., regret bounds, for the proposed algorithms. Previous studies, such as \citet{russo2014learning}, have shown that regret analysis for generalized linear bandits can be extended to TS with a Poisson model (TS-Poisson). However, such an extension is not straightforward for TS-NB, TS-ZIP, and TS-ZINB. The complexity arises primarily due to the undetermined overdispersion parameter in NB regression and the parameters governing the zero components in ZIP and ZINB. \xueqing{In what follows, we first establish Bayesian regret bounds for Algorithms \ref{alg1} and \ref{alg2}, assuming the posterior distributions $Q_t(\beta)$ and $Q_t(\gamma)$ are known. We then transition to examining Algorithms \ref{alg3} and \ref{alg4}, where we derive more robust frequentist regret bounds. It is important to note that the concentration results obtained from the Bayesian regret analysis are crucial for formulating the frequentist regret bounds in Algorithms \ref{alg3} and \ref{alg4}.}

To facilitate our analysis, we introduce some new notations.  The norm of a vector $x$ is denoted as $|x| = \sqrt{x^\top x}$, and $\tilde{O}$ denotes the big-$O$ notation up to logarithmic factors in time $T$.  For any positive semi-definite matrix $M$, $\lambda_{\min}(M) \geq 0$ denotes the minimum eigenvalue of $M$. Additionally, $\Delta_{\max}$ represents the maximum gap across all actions $a$,  
namely, $\Delta_{\max} = \max_{a\in \mathcal{A}}\Delta_a$. The regret analysis depends on several key assumptions stated below.

\begin{assumption}
\label{ass1}
    For time $t$ and any intervention $a$, we have $\left\|\phi(a,X_t)\right\| \leq 1$. In addition, we have  $\|\gamma^*\| \leq 1$ and $\|\beta^*\| \leq 1$. 
\end{assumption}
\begin{assumption}
\label{ass2}
    There exists $Y_{\max} >0$ such that for any $t \geq 1$, $0 \leq Y_t \leq Y_{\max}$ holds almost surely. 
\end{assumption}
\begin{assumption}
\label{ass3} The overdispersion parameter $r$ satisfies $0 < r_{\min} \leq r \leq r_{\max}$.
\end{assumption}


Assumption \ref{ass1} is used to make the regret bounds scale-free for convenience; otherwise, we may assume $\left\|\phi(a,X_t)\right\| \leq C$, $\|\gamma^*\| \leq C$, and $\|\beta^*\| \leq C$ instead, and the regret bound would increase by a factor of $C$. In practice, this assumption can also be satisfied by scaling the data. Assumption \ref{ass2} restricts the magnitude of the proximal outcome, which is adapted from the generalized linear bandit literature \citep{filippi2010}. \xueqing{In Lemma S1.1, S1.3, and S1.4, presented across Sections S3.1, S3.3, and S3.4 of the Supplement, we demonstrate that Assumption 2 validates the sub-Gaussian nature of the noise term in each model.} We make an additional Assumption \ref{ass3} to facilitate the regret analysis for TS-NB and TS-ZINB.

\xueqing{
\begin{theorem}
\label{bayes-nb}
Suppose we run TS-NB for a total of $T$ time points. Under Assumptions \ref{ass1}-\ref{ass3} (also satisfying certain mild regularity conditions defined in the Supplement), the Bayesian regret of TS-NB is upper-bounded as
\begin{eqnarray*}
\mathcal{R}_{\text{Bayes}}(T) &\leq&   \sqrt{d}+ \mathcal{O}(1) 
+ \left[2s Y_{\max}
\kappa_{\min}^{-1} \sqrt{d \log (T/ d)+4 \log T}\right]\sqrt{2 d T \log \left(\frac{T}{d }\right)},
\end{eqnarray*}
where $s$ represents the Lipschitz constant for the exponential function, and $\kappa_{\min}$ is a constant (for further details see the Supplement).
\end{theorem}
}

Theorem \ref{bayes-nb} establishes $\tilde{O}(d\sqrt{T})$ Bayesian regret bound, which coincides with the Bayesian regret bound for generalized linear TS \citep{russo2014learning}. A detailed proof of Theorem \ref{bayes-nb} is deferred to Section S1.1 of the Supplement.

\xueqing{
\begin{theorem}
\label{bayes-zip}
Suppose we run TS-ZIP for a total of $T$ time points. Under Assumptions \ref{ass1}-\ref{ass3} (also satisfying certain mild regularity conditions defined in the Supplement), the Bayesian regret of TS-ZIP is upper-bounded as
\begin{eqnarray*}
\label{eqzip}
\mathcal{R}_{\text{Bayes}}(T) &\leq&   \sqrt{d} + \mathcal{O}(1) 
+\left[2 q (Y_{\max } \kappa_{1,\min}^{-1} + \kappa_{2,\min}^{-1} )\sqrt{d \log (T / d)+4 \log T}\right] \sqrt{2 d T \log \left(\frac{T}{d}\right)},
\end{eqnarray*}
where $q = \sqrt{2}e$, $\kappa_{1,\min}$ and $\kappa_{2,\min}$ are constants (for further details see the Supplement).  
\end{theorem}
}

Theorem \ref{bayes-zip} also establishes $\tilde{O}(d\sqrt{T})$ Bayesian regret bound for TS-ZIP. The proof is deferred to Sections S1.2 of the Supplement. The Bayesian regret bound for TS-ZINB has the same form as TS-ZIP, but under slightly different conditions, and the proof can be found in Section S1.3 of the Supplement. 




In the second part of our analysis, we turn our focus to the evaluation of frequentist regret associated with the proposed algorithms.  Frequentist regret, in contrast to Bayesian regret, provides a measure of performance under the assumption that true parameters of the model are fixed but unknown, rather than being treated as random variables \citep{hamidi2020worst}. 

\xueqing{
\begin{theorem}
\label{regret-nb}
Suppose we run TS-NB for a total of $T$ time points. Under the Assumptions \ref{ass1}-\ref{ass3} (also satisfying certain mild regularity conditions defined in the Supplement), the regret of TS-NB is upper-bounded as
    \begin{eqnarray*}
        \mathcal{R}(T,\beta^*) &\leq& (\tau+ 2)\Delta_{\max}+ sY_{\max} \kappa_{\min }^{-1}C\left(1+\sqrt{2 \kappa_{\min }^{-1} \kappa_{\max } \log (K T)}\right)\\
    &\times&\sqrt{d \log (T / d)+2 \log T}  \sqrt{2 dT \log (T / d)},
    \end{eqnarray*}
    where the number of exploration rounds $\tau$ satisfies $\lambda_{\min}(G_\tau) \geq 1$. The constant $C$ is in the proof. 
\end{theorem}
}

Theorem \ref{regret-nb} establishes $\tilde{O}(d\sqrt{T\log K})$ frequentist regret bound. We defer the proof of Theorem \ref{regret-nb} to Section S2.1 of the Supplement.

\xueqing{
\begin{theorem}
\label{regret-zip}
Suppose we run TS-ZIP for a total of $T$ time points. Under Assumptions \ref{ass1}-\ref{ass3} (also satisfying certain mild regularity conditions defined in the Supplement), the regret of TS-ZIP is upper-bounded as
    \begin{eqnarray*}
    \label{regretzip}
        \mathcal{R}(T,\gamma^*, \beta^*) &\leq& (\tau+4-2/T)\Delta_{\max}+ qD\sqrt{d \log (T / d)+2 \log T}\sqrt{2 dT \log (T / d)}\\
        &\times&
         \Bigg[Y_{\max} \kappa_{1,\min }^{-1}\left(1+\sqrt{2 \kappa_{1,\min }^{-1} \kappa_{1,\max } \log (K T) }\right)\\
        &+& \kappa_{2,\min }^{-1}\left(1+\sqrt{2 \kappa_{2,\min }^{-1} \kappa_{2,\max } \log (K T)}\right)  \Bigg],
    \end{eqnarray*}
where the number of exploration rounds $\tau$ satisfies $\lambda_{\min}(G_\tau) \geq 1$. The constant $D$ is in the proof. 
\end{theorem}
}

Theorem \ref{regret-zip} also establishes $\tilde{O}(d\sqrt{T\log K})$ frequentist regret bound for TS-ZIP. A comprehensive proof can be found in Section S2.2 of the Supplement. Notably, the frequentist regret bound for TS-ZINB aligns with Theorem \ref{regret-zip}, albeit under slightly different conditions, as detailed in Section S2.2 of the Supplement. 

\xueqing{
\begin{remark}
    The regret bound $\tilde{O}(d\sqrt{T\log K})$ aligns with the regret bounds observed in generalized linear bandits with finite interventions \citep{kveton2020randomized, filippi2010, li2017provably}, up to a factor of $\sqrt{\log K}$. Importantly, \citet{li2017provably} suggests that a better frequentist regret bound for TS-Count can be established by imposing an additional condition on the minimum eigenvalue of $G_t$, that is, $\sum_{\tau=1}^T \lambda_{\min}^{-1/2}(G_t) \leq z \sqrt{T}$. The arguments presented by \citet{li2017provably} are also applicable in this context. 
\end{remark}
}

\section{Simulation Studies}
\label{simu}
In this section, we demonstrate the performance of the proposed TS-Count variants with a normal approximation under a wide range of simulation scenarios. Specifically, we compare the performances of the proposed TS-Poisson, TS-NB, TS-ZIP, and TS-ZINB with the existing Linear TS algorithm proposed by \citet{agrawal2013thompson}. To ensure a fair comparison, we apply a log transformation to the proximal outcome and subsequently utilize the Linear TS algorithm.

\subsection{Setup}
For the simulation setup, we consider a scenario with $K=20$ intervention options over a time horizon $T=1000$. For each intervention $a$ at every time point $t$, the feature vector $\phi(a,X_t)$ is directly sampled from a multivariate normal distribution, specifically $\operatorname{MVN}(\mathbf{0},\mathbf{I}_{4\times 4})$, where $\mathbf{I}_{4\times 4}$ denotes the diagonal unit matrix.  Moreover, we generate the true parameters $\gamma^*$ and $\beta^*$ from a multivariate normal distribution $\operatorname{MVN}(\mathbf{0},\mathbf{I}_{4\times 4})$.  We normalize the feature and true parameters to satisfy $\|\phi(a,X_t)\| \leq 1$, $\|\gamma^*\| \leq 1$ and $\|\beta^*\| \leq 1$. 

Our simulation studies encompass a total of eight settings, with the stochastic proximal outcome generated from (1) a Poisson distribution, (2) an overdispersed Poisson (OP) distribution with a low dispersion level, (3) an OP distribution with a moderate dispersion level, (4) an OP distribution with a high dispersion level, (5) a ZIP distribution, (6) a zero-inflated and overdispersed Poisson (ZIOP) distribution with a low dispersion level, (7) a ZIOP distribution with a moderate dispersion level, and (8) a ZIOP distribution with a high dispersion level.

For Settings (1) - (4), the underlying model for generating count outcomes $Y_t$ follows the structure:
\begin{eqnarray}
\label{gen1}
    Y_t \sim \operatorname{Poisson}(\lambda_t \mu_t), ~\ \lambda_t \sim \operatorname{Gamma}(\omega, 1/\omega),~\ \mu_t = \exp\left(\phi(A_t, X_t)^{\top}\beta^*\right).
\end{eqnarray}
In this case, the outcomes are generated from an OP model, where $\lambda_t$ introduces extra overdispersion, and $\omega$ controls the level of dispersion. We induce high, moderate, and low levels of overdispersion by setting $\omega = 0.25, 1, $ and $25$, respectively. When setting $\lambda_t=1$, we recover the Poisson distribution. 

For Settings (5) - (8), the underlying model for generating count outcomes $Y_t$ is defined as follows:
\begin{eqnarray}
\label{gen2}
    Z_t &\sim& \operatorname{Bernoulli} (1-p_t), ~\ p_t = \operatorname{sigmoid}\left(\phi(A_t, X_t)^{\top}\gamma^*\right), \nonumber\\
    C_t &\sim& \operatorname{Poisson}(\lambda_t \mu_t), ~\ \lambda_t \sim \operatorname{Gamma}(\omega, 1/\omega),~\ \mu_t = \exp\left(\phi(A_t, X_t)^{\top}\beta^*\right), \nonumber\\
    Y_t &=& Z_tC_t.
\end{eqnarray}
Here, the outcomes are generated from a ZIOP model. We also consider scenarios without zero-inflation and overdispersion by setting both $Z_t$ and $\lambda_t$ to $1$. 

In all simulations, we set the tuning parameter $\alpha=1$, which implements the Laplace approximation. We designated the initial random exploration period as $\tau=20$, during which each intervention was selected once to establish the initial estimators. Performance was assessed by averaging frequentist regret over 200 replications.

\xueqing{To further evaluate performance, we compare the proposed algorithms with their TS counterparts, which use MCMC methods to approximate posterior distributions. For MCMC-based TS algorithms, we use priors $\beta^* \sim \mathcal{N}(0,I)$  and  $\gamma^*\sim \mathcal{N}(0,I)$. We assess the total running time and regret over  $200$ replications. The stochastic outcomes are generated under Setting (7). 
Due to computational constraints, we set $T=200$ for these evaluations.}

\subsection{Results}
The results under Settings (1) - (4) are presented in Figure \ref{fig:mainfig1}. Our findings reveal that the TS-Count variants exhibit superior performance compared to the Linear TS algorithm even with log transformation, which demonstrates consistently high and linear regret across all settings. Furthermore, our results suggest that TS-Poisson outperforms other algorithms, including TS-NB, particularly in settings with varying degrees of overdispersion. This is attributed to the reliability of Poisson regression in yielding consistent estimates when the conditional mean model is appropriately specified, meeting the requirement for sublinear regret, as demonstrated in the proof of regret. 

Figure \ref{fig:mainfig2} provides a comprehensive summary of the simulation results for Settings (5) - (8), where the underlying outcome model incorporates zero-inflation. Similarly, Linear TS with log transformation exhibits inferior performance when compared to other algorithms. Notably, our observations indicate that TS-Poisson and TS-NB demonstrate linear regret across various settings due to not addressing zero-inflation. In contrast, TS-ZIP consistently exhibits superior performance in scenarios involving zero-inflated outcomes, regardless of the presence or absence of overdispersion.

\xueqing{In Figure \ref{fig:mainfig3}, we present the comparison results between the proposed algorithms and their MCMC counterparts. Figure \ref{fig:subfig9} demonstrates that the proposed TS-Count algorithms significantly reduce computation time compared to their MCMC counterparts. Figure \ref{fig:subfig10} reveals that the proposed algorithms achieve similar levels of regret as their MCMC versions. Overall, the proposed algorithms offer substantial computational efficiency without compromising performance, making them practically advantageous.
}

\begin{figure}[ht]
    \centering
    \subfloat[No overdispersion]{\includegraphics[width=0.5\textwidth]{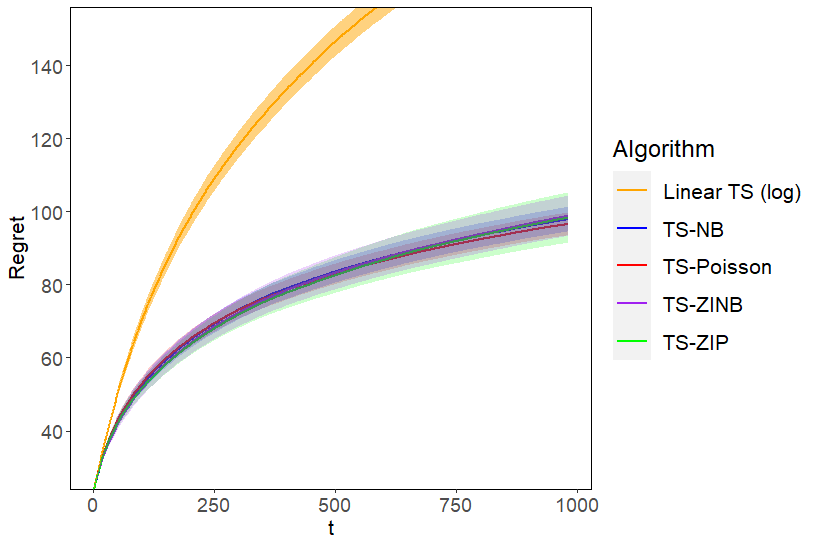}\label{fig:subfig1}}
    \hfill
    \subfloat[Low overdispersion]{\includegraphics[width=0.5\textwidth]{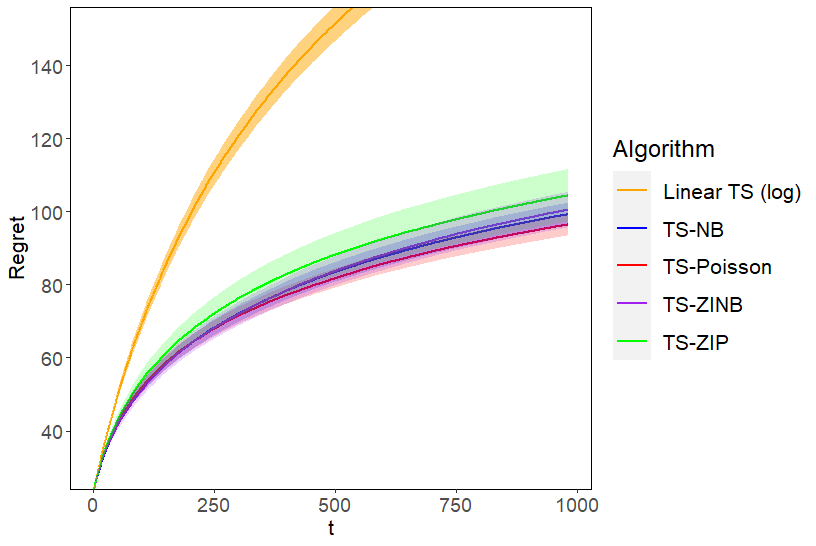}\label{fig:subfig2}}
    
    \subfloat[Moderate overdispersion]{\includegraphics[width=0.5\textwidth]{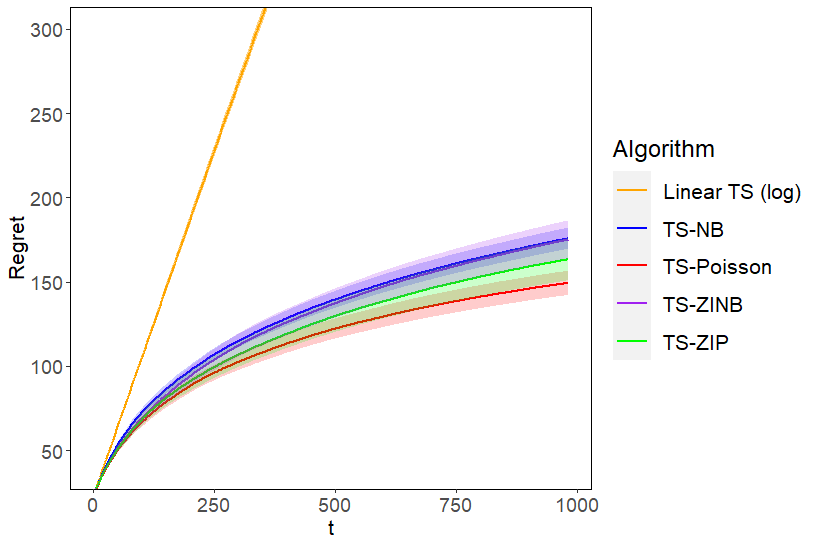}\label{fig:subfig3}}
    \hfill
    \subfloat[High overdispersion]{\includegraphics[width=0.5\textwidth]{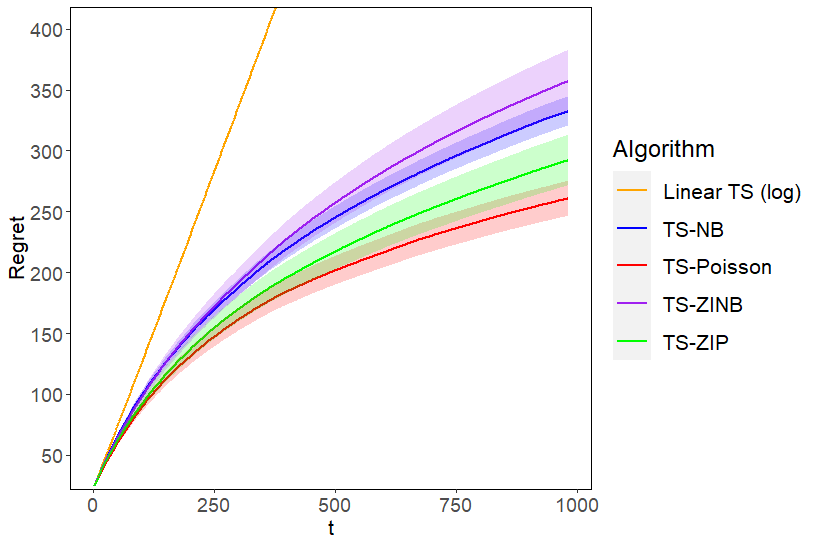}\label{fig:subfig4}}
    
    \caption{Simulation results of the compared algorithms under Settings (1) - (4). The results are calculated from $200$ replications of the experiment. The solid lines indicate the mean
values, while the shaded bands represent the standard error bounds across the independent replications.}
    \label{fig:mainfig1}
\end{figure}

\begin{figure}[ht]
    \centering
    \subfloat[No overdispersion]{\includegraphics[width=0.5\textwidth]{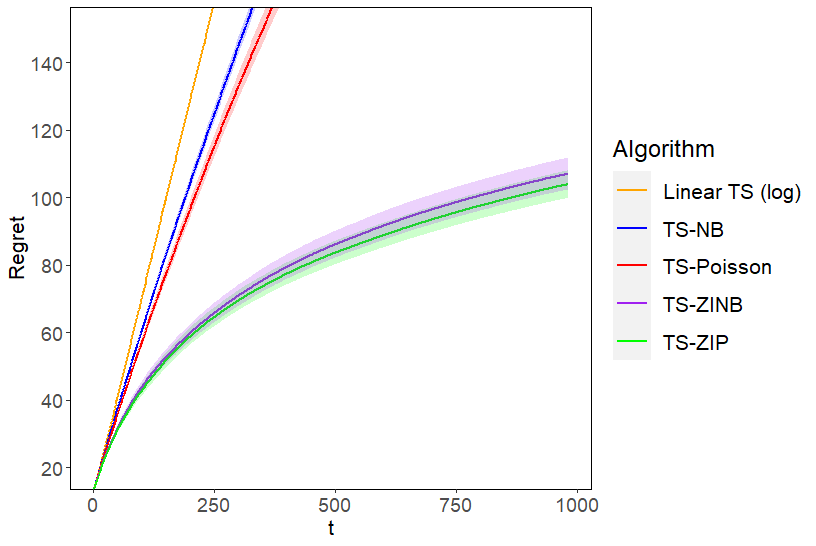}\label{fig:subfig5}}
    \hfill
    \subfloat[Low overdispersion]{\includegraphics[width=0.5\textwidth]{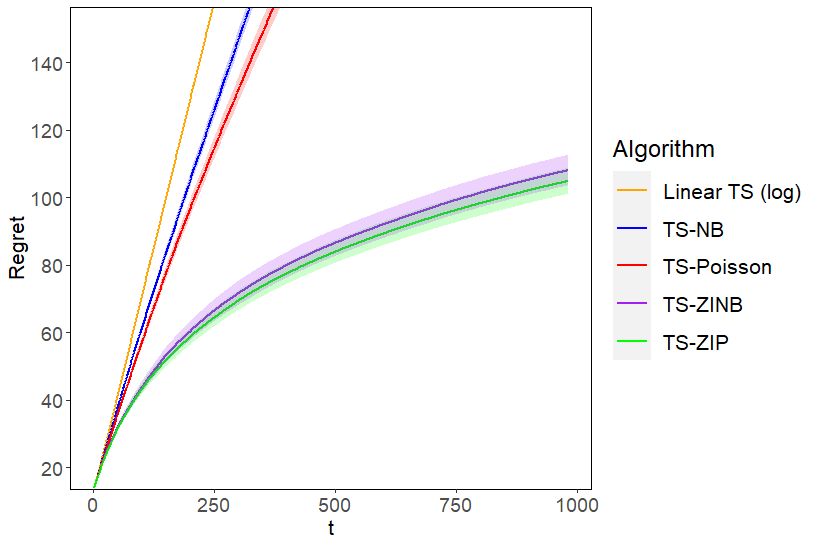}\label{fig:subfig6}}
    
    \subfloat[Moderate overdispersion]{\includegraphics[width=0.5\textwidth]{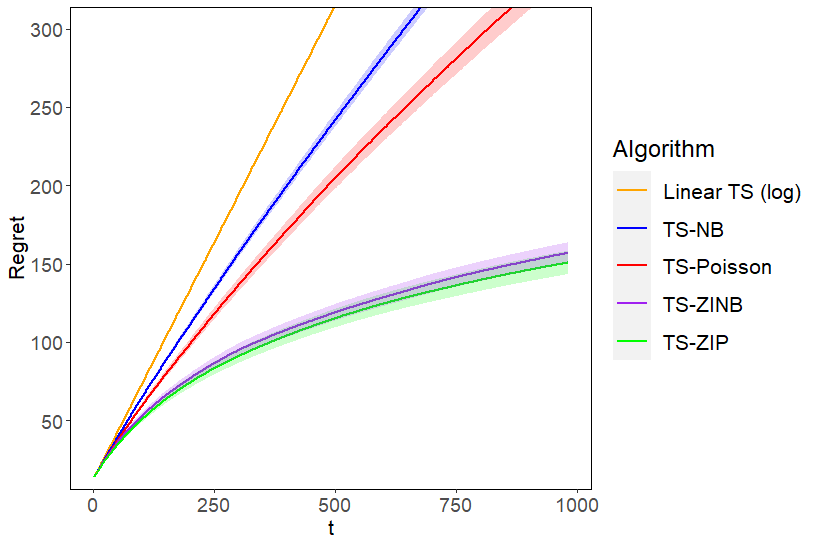}\label{fig:subfig7}}
    \hfill
    \subfloat[High overdispersion]{\includegraphics[width=0.5\textwidth]{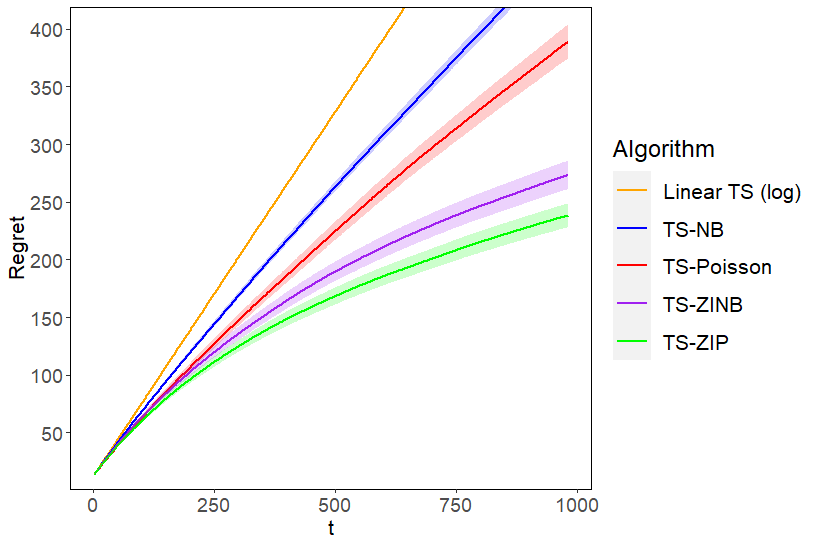}\label{fig:subfig8}}
    
    \caption{Simulation results of the compared algorithms under Settings (5) - (8). The results are calculated from $200$ replications of the experiment. The solid lines indicate the mean
values, while the shaded bands represent the standard error bounds across the independent replications.}
    \label{fig:mainfig2}
\end{figure}

\begin{figure}[ht]
    \centering
    \subfloat[Computation time]{\includegraphics[width=0.4\textwidth]{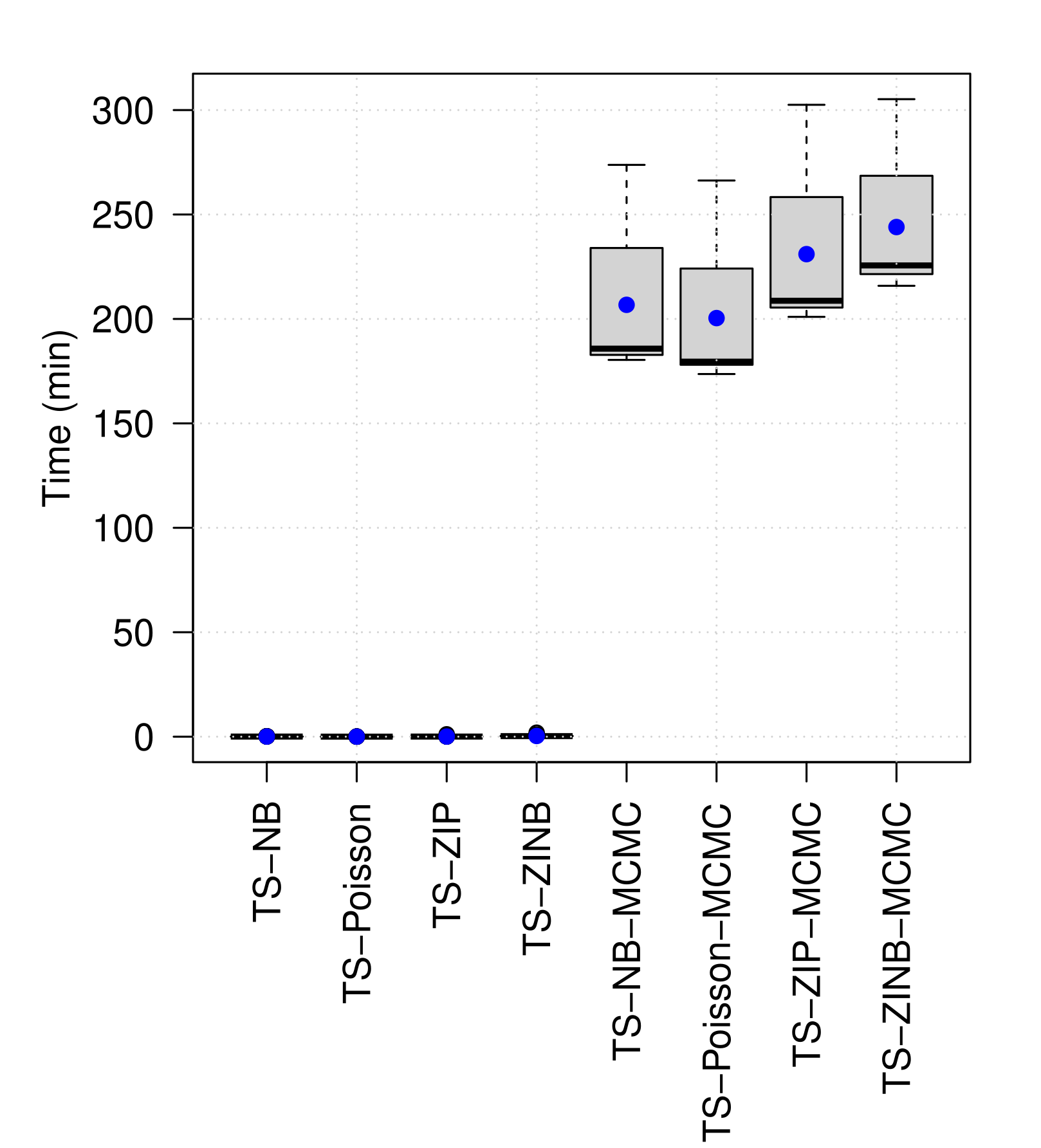}\label{fig:subfig9}}
    \hfill
    \subfloat[Regret]{\includegraphics[width=0.6\textwidth]{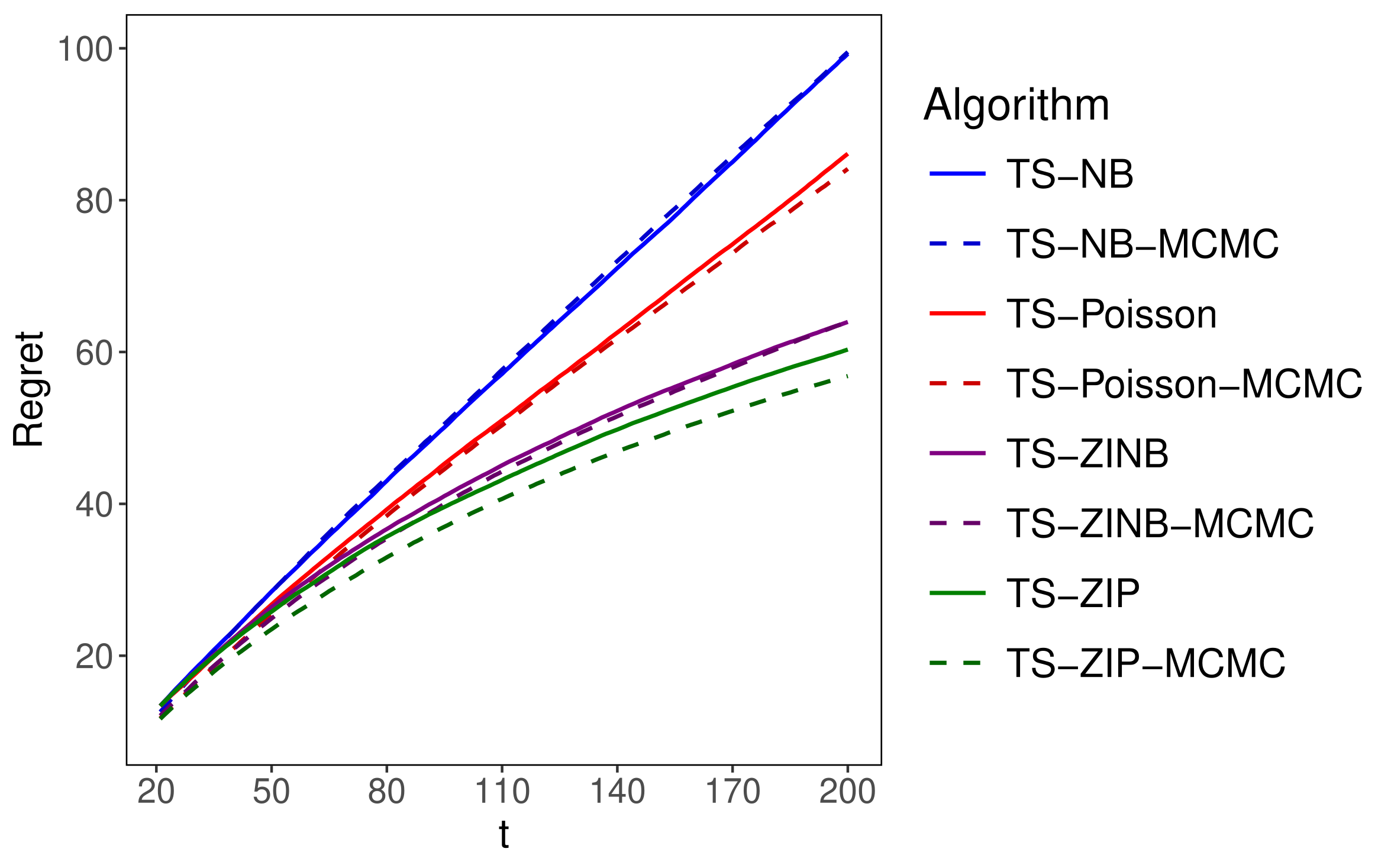}\label{fig:subfig10}} 
    \caption{Comparison of the proposed algorithms and their MCMC counterparts focuses on computation time and regret. The results are calculated from $200$ replications of the experiment.}
    \label{fig:mainfig3}
\end{figure}

\section{Application to the Drink Less Case Study}
\label{app}
In this section, we implement the proposed algorithms in a real case study - the Drink Less MRT - to evaluate and benchmark their performances. We consider not only the TS-Count variants and the Linear TS algorithm that incorporates a logarithmic transformation of the outcome but also the original static strategy from Drink Less, which assigns interventions at a constant probability $0.6$. 
\xueqing{Our application is divided into two parts: (1) a simulation study guided by the Drink Less data, and (2) an off-policy evaluation directly applied to the Drink Less data.}

\subsection{Implementation of Algorithms}
\xueqing{
Up to this point, our discussion has centered on an individual user. However, within the framework of mHealth studies or MRTs, the paradigm typically involves $N$ users, each receiving interventions at $T$ decision times. We introduce the subscript $n$ to denote each user. At any decision time $t$,  we have a collection of historical observations denoted by $D_{n, t} = \{X_{n,i},A_{n,i},Y_{n,i};i=1,\cdots, t-1\}_{n=1}^N \cup \{X_{n,t}\}_{n=1}^N$. The overall objective is to customize interventions to cater to the specific demands of each user. This scenario can be envisioned as $N$ separate contextual bandit problems, where each participant is associated with a unique model parameter $\eta_n^*$. We implement the proposed algorithms individually for each user, consistent with previous research that has shown the effectiveness of personalized approaches \citep{liao2020a,pmlr-v149-yao21a}.

Recall that $\phi(A_{n,t}, X_{n,t})$ represents the feature vector, which transforms the intervention $A_{n,t}$ and the context $X_{n,t}$ as follows:
\begin{align*}
    \phi(A_{n,t}, X_{n,t})^\top =(X_{n,t}^\top, A_{n,t}S_{n,t}^\top), 
\end{align*}
where $X_{n,t}$ is the context containing baseline and time-varying information of user $n$, and $S_{n,t}$ is a subset of $X_{n,t}$ that may interact with intervention (also known as moderators in the causal inference literature). Table S1 in Section S4.1 of the Supplement describes the context available in the Drink Less study.

Another critical aspect of applying the proposed algorithms to MRTs is ensuring robust post-study statistical analyses. It is important to note that algorithms designed solely to minimize regret may not ensure error control or provide sufficient statistical power to draw reliable inferences about treatment effects using standard methods, nor may they effectively support off-policy evaluations \citep{liao2020a,pmlr-v149-yao21a}. To address this issue, the randomization probabilities must remain strictly between $0$ and $1$, avoiding extremes. Specifically, we constrain the randomization probability within bounds $p_{\min}$ and $p_{\max}$, i.e., $\tilde{p}_t(X_{n,t}) = \max(p_{\min}, \min(p_{\max}, p_t(X_{n,t}))$. This method is known as ``clipping'' in mHealth \citep{liao2020a,pmlr-v149-yao21a}. For our simulations, we set $p_{\min} = 0.01$ and $p_{\max} = 0.99$. 

Moreover, we apply the proposed algorithms to the Drink Less Data by incorporating a ridge penalty of the form $\lambda \|\beta_n - \beta_{0}\|_2^2$ and $\lambda \|\gamma_n - \gamma_{0}\|_2^2$, rather than incorporating an explicit initial exploration stage of length $\tau$. This modification is enabled by the availability of parallel data from the Drink Less MRT \citep{bell2020notifications}, allowing us to form prior estimates for $\beta_n^*$ and $\gamma_n^*$. Following the guidelines from Theorems \ref{regret-nb} and \ref{regret-zip}, we choose $\lambda = 1$. For subsequent analyses, $\beta_{0}$ and $\gamma_{0}$ are set based on the estimates from the parallel data. }

\subsection{Clipped Regret Analysis}
\xueqing{
Due to the implementation of clipping, the traditional regret metric is not applicable. Instead, we assess the proposed algorithms against a clipped oracle - an oracle constrained to intervention probabilities within the range $[p_{\min}, p_{\max}]$ \citep{NIPS2017_4fa177df,pmlr-v149-yao21a}. Specifically, the frequentist regret relative to this clipped oracle is defined as follows:
\begin{eqnarray*}
       \mathcal{R}(T,\eta^*) =\sum_{t=1}^T \mathbb{E}\left[\max_{a,p^*}h(\phi(a,X_t)^{\top};\eta^*) - h(\phi(A_t,X_t)^{\top};\eta^*) \Big| \eta^*\right],
\end{eqnarray*}
where $a\in \mathcal{A}$ and $p^* \in [p_{\min}, p_{\max}]$. It can be shown that, with probability clipping, the proposed TS-Count algorithms will maintain the same regret bounds with respect to a clipped oracle. Here, we consider TS-NB as a specific example. Suppose 
Let $\bar{A}_t^*$ represent the optimal non-zero intervention chosen by the oracle and $\bar{A}_t$ is the corresponding choice made by the algorithm. Following \citet{NIPS2017_4fa177df}, the regret for times $t > \tau$ can be decomposed as:
\begin{eqnarray*}
    \operatorname{regret}_t &=& p^*\exp(\phi(\bar{A}_t^*, X_t)^\top \beta^*) + (1-p^*)\exp(\phi(0, X_t)^\top \beta^*)\\
    &&- \tilde{p}_t\exp(\phi(\bar{A}_t, X_t)^\top \beta^*) - (1-\tilde{p}_t)\exp(\phi(0, X_t)^\top \beta^*)\\
    &\leq&(p^*-\tilde{p}_t)[\exp(\phi(\bar{A}_t, X_t)^\top \beta^*) - \exp(\phi(0, X_t)^\top \beta^*)] \\
    &+& [\exp(\phi(\bar{A}_t^*, X_t)^\top \beta^*) - \exp(\phi(\bar{A}_t, X_t)^\top \beta^*)]
\end{eqnarray*}

Under Assumptions \ref{ass1} - \ref{ass3}, the cumulative regret $\mathcal{R}(T,\beta^*)$ can be bounded as follows:
\begin{eqnarray*}
    \mathcal{R}(T,\beta^*) \leq \tau \Delta_{\max} + \sum_{t=\tau+1}^T\mathbb{E}[\operatorname{regret}_t] \leq \tau \Delta_{\max} + \sum_{t=\tau+1}^T\mathbb{E}[T_1 + T_2]
\end{eqnarray*}
where $T_1 = (p^*-\tilde{p}_t)[\exp(\phi(\bar{A}_t, X_t)^\top \beta^*) - \exp(\phi(0, X_t)^\top \beta^*)]$ and $T_2 = [\exp(\phi(\bar{A}_t^*, X_t)^\top \beta^*) - \exp(\phi(\bar{A}_t, X_t)^\top \beta^*)]$. We bound $\mathbb{E}[T_1]$ based on Lemma 1 of \citet{NIPS2017_4fa177df} and bound $\mathbb{E}[T_2]$ using Lemma S2.1 from the Supplement. According to \citet{NIPS2017_4fa177df}, $\mathbb{E}[T_2]$ is the predominant component affecting the overall regret bound. Consequently, the clipped frequentist regret bound for TS-NB aligns with the standard unclipped scenario. This reasoning is similarly applicable to other variants of the TS-Count algorithms.
}

\subsection{Simulation Study Guided by Drink Less}
\xueqing{In this part, we build a simulation environment to mimic the Drink Less MRT and evaluate the candidate algorithms. We show that the proposed algorithms incur lesser user-averaged regret.}

\subsubsection{Simulation Environment}
\xueqing{We consider two different outcome generative models for the Drink Less study: (1) an OP model, as presented in Equation \eqref{gen1}, and (2) a ZIOP model, as presented in Equation \eqref{gen2}. We use the Drink Less MRT data to fit the outcome model. Each model was fitted individually for each user using the maximum a posteriori (MAP) estimation method. We used a prior $\beta_n^* \sim \mathcal{N}(0,I)$ for model (1) and an additional prior $\gamma_n^* \sim \mathcal{N}(0,I)$ for model (2). In our simulations, the context $X_{n,t}$ is directly observed from the data to mimic the practical scenario accurately. }



We generate $100$ independent simulated datasets using the aforementioned models, evaluating the algorithms based on their cumulative regret. In each simulated experiment, we randomly draw $N=349$ users, with each user participating in the study for over $200$ days. \xueqing{Additional simulations on scenarios with $p_{\min} = 0$ and $p_{\max} = 1$ and user-specific performances are included in Section S4.1 of the Supplement.}

\subsubsection{Results}
\xueqing{
Figure \ref{fig:mainfig5} illustrates the cumulative regret averaged over $349$ users and $100$ iterations of six different algorithms. The static strategy exhibits a consistently linear increase in regret, attributed to its lack of responsiveness to user feedback. As shown in Figure \ref{fig:subfig15}, the Linear TS algorithm with a logarithmic transformation initially demonstrates lower regret. However, with the accumulation of data throughout the experiment, the TS-Count variants begin to outperform the Linear TS, indicating superior performance in minimizing regret. This delayed efficacy may be due to the TS-Count variants requiring more data for learning so that their enhanced modeling capabilities become effective. Figure \ref{fig:subfig17} further highlights that in the early stages of the trial (where $T \leq 40$), Linear TS performs better than TS-ZIP, aligning with findings from \citet{trella2022designing} comparing these algorithms.


Overall, the findings suggest that the proposed approaches, particularly TS-Poisson, significantly improve user engagement with the app, particularly when $T$ is relatively large. For scenarios with a small $T$, Linear TS with a logarithmic transformation could still serve as a good learning algorithm.}


\begin{figure}[ht]
    \centering
    \subfloat[OP]{\includegraphics[width=0.5\textwidth]{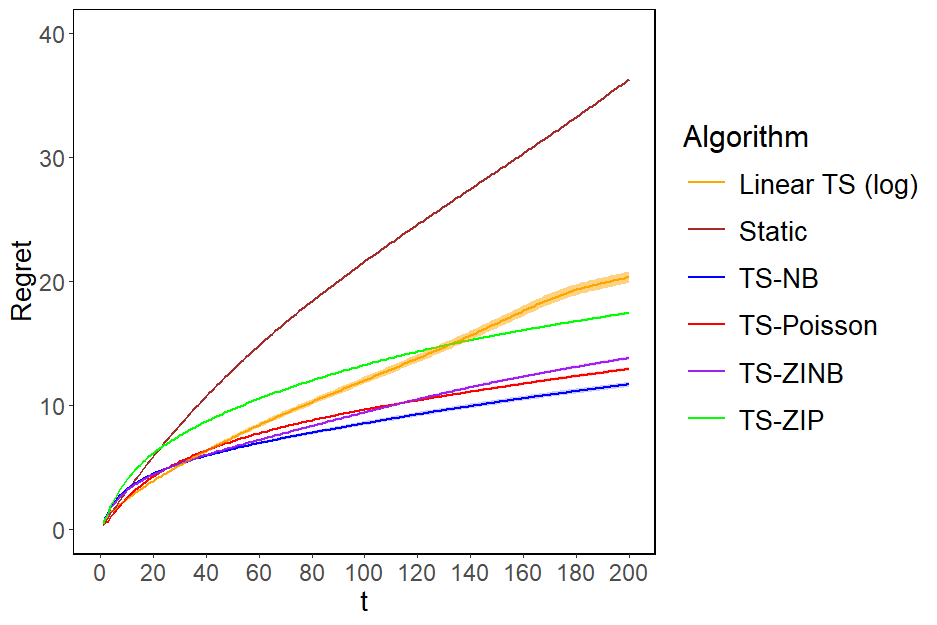}\label{fig:subfig15}}
    \hfill
    \subfloat[ZIOP]{\includegraphics[width=0.5\textwidth]{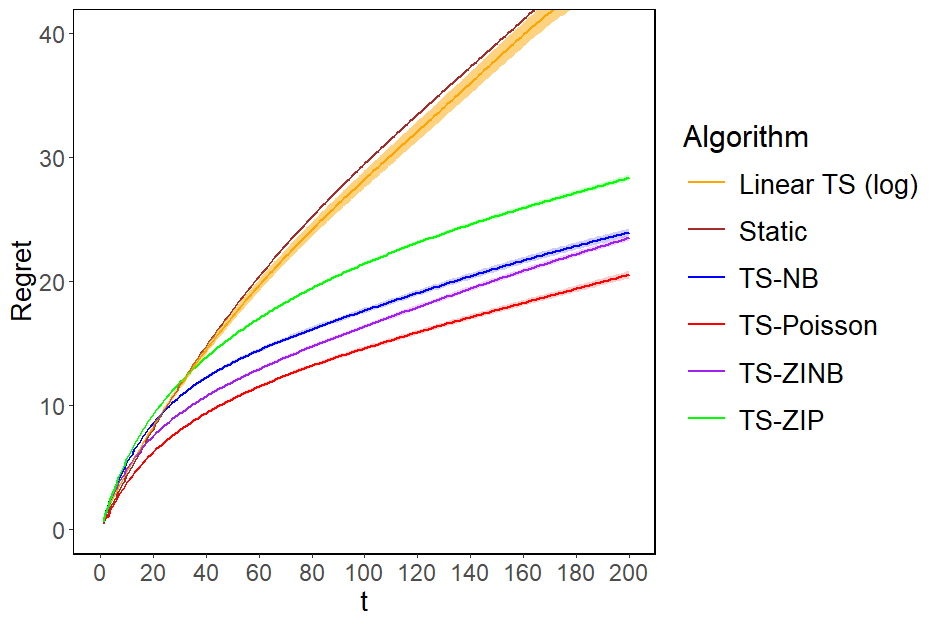}\label{fig:subfig17}}
    
    \caption{Results of the compared algorithms under the Drink Less study with $p_{\min} = 0.01$ and $p_{\max} = 0.99$.  The results are calculated from $100$ replications of the experiment. The solid lines indicate the mean values, while the shaded bands represent the standard error bounds across the independent replications.}
    \label{fig:mainfig5}
\end{figure}

\subsection{Off-Policy Evaluation on Drink Less}
\xueqing{
Next, we use off-policy evaluation methods to evaluate the candidate algorithms directly using real data from the Drink Less MRT. We demonstrate that the proposed algorithms lead to greater user-averaged expected rewards. }
\subsubsection{Methods}
\xueqing{
We applied the proposed algorithms to the real Drink Less MRT data, utilizing the method of \citet{Li2010a}.  Specifically, at each time $t$, given the collected history $D_{t}$, if the intervention selected by the algorithm matches the intervention assigned in Drink Less, the data tuple $(X_t, A_t, Y_t)$ is retained, and the algorithm is updated. Conversely, if the intervention selected by the algorithm differs from the intervention assigned in Drink Less, the data tuple is disregarded, and the algorithm proceeds to the next decision time. 

Using the collected data, we employ the self-normalized inverse probability weighting method to form an unbiased estimate of the per-trial expected reward achieved by each algorithm \citep{huch2023debiased}:
\begin{eqnarray*}
    \widehat{R} = \sum_{t=1}^T\frac{\frac{\pi_t(a_t|X_t)}{p_t(a_t|X_t)}Y_t}{\sum_{t=1}^T\frac{\pi_t(a_t|X_t)}{p_t(a_t|X_t)}},
\end{eqnarray*}
where $\pi_t(a_t|X_t) = 0.6$ represents the actual randomization probability in Drink Less and $p_t(a_t|X_t)$  represents the probability generated by the algorithm. 

We compute the reward improvement, defined as the expected reward of the algorithm minus the expected reward under the pre-specified Drink Less randomization strategy, for each candidate algorithm.  We perform this evaluation using $200$ bootstrap samples, running the algorithms separately for each user in each sample. The final results are averaged across all users and datasets.
}

\subsubsection{Results}
\xueqing{
We present the per-trial expected reward improvement of five algorithms in Figure \ref{fig:mainfig6},
which demonstrates that TS-Poisson outperforms the other algorithms, with TS-NB as the next most effective.
However, the brief duration of the Drink Less MRT data ($T=30$) may compromise the reliability of these expected reward estimates.
}

\begin{figure}[ht]
    \centering
    \includegraphics[width=0.4\textwidth]{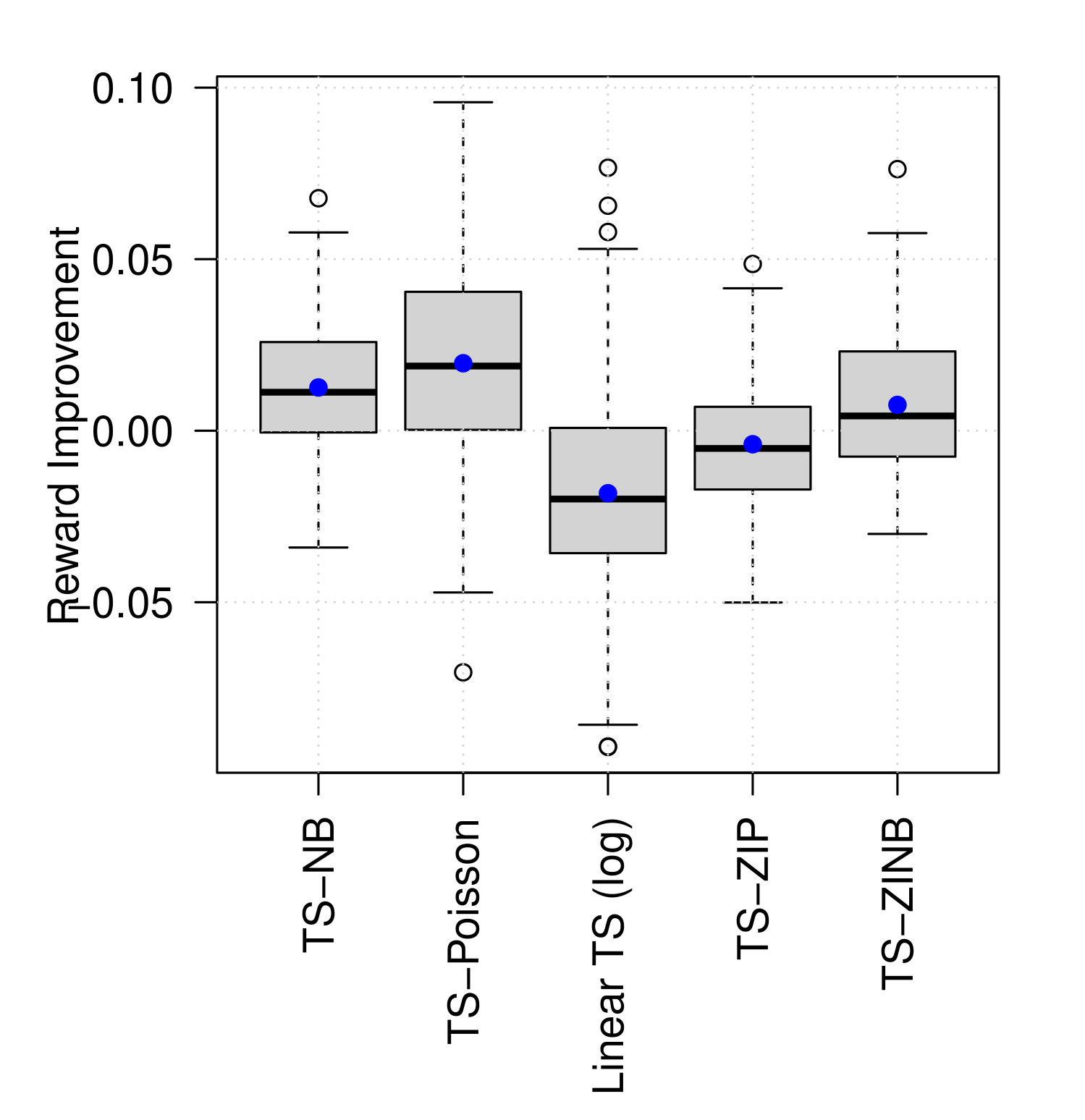}
    
    \caption{Boxplots of the estimates of the reward improvement for five competing algorithms. The results are averaged over $200$ Bootstrap samples. }
    \label{fig:mainfig6}
\end{figure}

\section{Discussion}
\label{conc}

Contextual bandits provide a suitable framework for tailoring mHealth interventions. However, in real-world scenarios, such as the Drink Less study \citep{bell2020notifications}, proximal outcomes are often collected as counts. This paper addresses the challenge of optimizing intervention delivery in mHealth applications with count proximal outcomes. Specifically, we propose methods that combine common count data models with the Thompson sampling strategy. We provide theoretical analyses on both standard regret and clipped regret for the proposed algorithms.

Within this paper, we have conducted comparisons across a suite of count models—including Poisson, NB, ZIP, and ZINB—considering varying degrees of overdispersion and zero-inflation within an online contextual bandit framework. Interestingly, our findings suggest that more complex models like the NB and ZINB regressions, which offer mechanisms to accommodate overdispersion unlike their Poisson counterparts, may not outperform in an online environment. This could be attributed to the limited data typically available online and the complexity introduced by the additional dispersion parameter that necessitates estimation.  Another reason might be that, despite overdispersion, the Poisson model can still yield consistent parameter estimates which are adequate for bounding the regret. 

Previous studies on contextual bandits based on generalized linear models have focused on a generic environment for online decision-making \citep{filippi2010, li2017provably}. However, they left unresolved statistical issues in applying contextual bandits to mHealth studies, including how to model overdispersed and zero-inflated count proximal outcomes in an online setting and whether a proper statistical model can improve the performance of online decision-making. Our work addresses these issues by providing both theoretical guarantees and evaluating their empirical performances. These algorithms, however, are not limited to mHealth studies and can be applied to other studies with a count outcome. 

There are opportunities to extend this work in several directions. For instance, many mHealth studies involve multiple users being intervened at each decision time. By incorporating mechanisms for information borrowing between users, as suggested by \citet{tomkins2021intelligentpooling}, our algorithms could learn more efficiently and accommodate user heterogeneity. Additionally, exploring alternative approximation methods for sampling from posterior distributions presents a substantial opportunity for improvement. Techniques such as stochastic gradient descent \citep{ding2021a} and Langevin Monte Carlo methods \citep{xu2022a} could be adapted for use in zero-inflated models. These methods promise more computationally efficient algorithms, a crucial factor in large-scale and real-time applications. Furthermore, a straightforward extension would be differentiating the feature vectors of the zero component and the non-degenerate distribution component (Poisson or NB) in ZIP and ZINB models. This modification would allow the models to more accurately reflect real-world scenarios where these components are likely to differ.


\begin{acks}[Acknowledgments]
Xueqing Liu is supported by PhD student scholarship from the Duke-NUS Medical School, Singapore. Lauren Bell is supported by a PhD studentship funded by the MRC Network of Hubs for Trials Methodology Research (MR/L004933/2- R18). Bibhas Chakraborty would like to acknowledge support from the grant MOE-T2EP20122-0013 from the Ministry of Education, Singapore. 
\end{acks}
\bibliographystyle{imsart-nameyear} 
\bibliography{Bibliography-MM-MC}       
\newpage
\appendix
\section*{}
This supplementary file includes the proof of the theorems and propositions presented in the main manuscript. \xueqing{Also, we provide additional details on the Drink Less study in S4.}  



To facilitate the proof, we introduce some notations. For $t = 1,\cdots,T$, let $\mathcal{F}_t$ denote the filtration which contains all available information up to time $t$. Let $[n]$ denote the set $\{1,2,\cdots, n\}$. Let $\|x\|_{V}$ represent the weighted Euclidean norm of vector $x$ under a positive semi-definite matrix $V$, i.e., $\sqrt{x^{\top}Vx}$. Let $\nabla f$ denote the gradient vector of a function $f(\cdot)$.
To simplify notation, we define $\mathbb{E}_t$ and $\mathbb{P}_t$ as the conditional expectation and conditional probability, given $\mathcal{F}_t$. For any positive semi-definite (PSD) matrix $M$, $\lambda_{\min}(M) \geq 0$ denotes the minimum eigenvalue of $M$.  For any $n \times n$ PSD matrices $M_1$ and $M_2$, $M_1 \preceq M_2$ if and only if $x^{\top}M_1x \leq x^{\top}M_2x$ for all $x \in \mathbb{R}^d$. In addition, recall that $G_t=\sum_{i=1}^{t-1} \phi(A_i,X_i) \phi(A_i,X_i)^{\top}$ is the unweighted Hessian matrix at decision time $t$. Additionally, $d$ denotes the dimension of the feature vector $\phi(a,X_t)$, $K$ denotes the number of intervention options, and $A_t^*$ denotes the optimal intervention option at time $t$. The expected reward is denoted by $R_t(a,\eta^{\ast})$ and we have that $R_t(a,\eta^{\ast})=\exp\left\{\phi(a,X_t)^{\top} \beta^\ast\right\}$ for TS-NB, and $R_t(a,\eta^{\ast})=\left[1-\operatorname{sigmoid} (\phi(a,X_t)^\top \gamma^*)\right]\exp\left\{\phi(a,X_t)^{\top} \beta^{\ast}\right\}$ for TS-ZIP and TS-ZINB. 

To derive regret bounds, we assume the following conditions hold.

\noindent \textbf{For TS-NB:}
\begin{itemize}
    \item[Condition 1.] 
    Let $$g_t(\beta) = \frac{r\left(r+ Y_t\right) \exp \{\phi(A_t,X_t)^{\top}\beta\} }{\left(r+ \exp \{\phi(A_t,X_t)^{\top} \beta\}\right)^2},$$ where $r$ denotes the inverse dispersion parameter.  There exists $\kappa_{\max} > \kappa_{\min} >0$ such that for any time $t$ 
    \[\kappa_{\min} \leq \inf_{\beta\in \mathbb{R}^d}g_t(\beta) \leq \sup_{\beta\in \mathbb{R}^d}g_t(\beta) \leq  \kappa_{\max}.\]
\end{itemize}

\noindent \textbf{For TS-ZIP:}
\begin{itemize}  
    \item[Condition 2.] 
    Let \[
    g_t(\beta) = I(Y_t=0)\frac{(1-p_t)\mu_t\exp(-\mu_t)}{p_t+(1-p_t)\exp(-\mu_t)}-I(Y_t=0)\frac{p_t(1-p_t)\mu_t^2\exp(-\mu_t)}{(p_t+(1-p_t)\exp(-\mu_t))^2} +I(Y_t>0)\mu_t.
    \]
    There exists $\kappa_{1,\max} > \kappa_{1,\min} >0$ such that for any time $t$
    \[\kappa_{1,\min} \leq \inf_{\beta\in \mathbb{R}^d}g_t(\beta) \leq \sup_{\beta\in \mathbb{R}^d}g_t(\beta) \leq  \kappa_{1,\max}.\]

    \item[Condition 3.]
    Let 
    \begin{align*}
    g_t(\gamma) &= -I(Y_t=0)\frac{(1-\exp(-\mu_t)(1-2p_t)p_t(1-p_t)}{p_t+(1-p_t)\exp(-\mu_t)} + I(Y_t=0)\frac{(1-\exp(-\mu_t)^2p_t^2(1-p_t)^2)}{(p_t+(1-p_t)\exp(-\mu_t)^2)}\\
    &\quad +I(Y_t>0)p_t(1-p_t).
    \end{align*}
    There exists $\kappa_{2,\max} > \kappa_{2,\min} >0$ such that for any time $t$
    \[\kappa_{2,\min} \leq \inf_{\gamma\in \mathbb{R}^d}g_t(\gamma) \leq \sup_{\gamma\in \mathbb{R}^d}g_t(\gamma) \leq  \kappa_{2,\max}.\]

\end{itemize}
    
\noindent \textbf{For TS-ZINB:}
\begin{itemize}
    \item[Condition 4.]  
    Let
    \begin{align*}
        g_t(\beta) &= I(Y_t=0)\frac{\mu_t^2(1-p_t)r\frac{r}{\mu_t+r}^{r-1}\frac{2r}{(\mu_t+r)^4}}{p_t+(1-p_t)\frac{r}{\mu_t+r}^r} +I(Y_t=0)\frac{\mu_t^2(1-p_t)r\frac{r^{2r}}{(\mu_t+r)^{2r+2}}}{(p_t+(1-p_t)\frac{r}{\mu_t+r}^r)^2}\\
        &\quad +I(Y_t=0)\frac{\mu_t(1-p_t)r\frac{r}{\mu_t+r}^{r-1}\frac{r}{(\mu_t+r)^2}}{p_t+(1-p_t)\frac{r}{\mu_t+r}^r} - I(Y_t=0)\frac{(1-p_t)r(r-1)\frac{r}{\mu_t+r}^{r-2}\frac{r^2}{(\mu_t+r)^4}}{p_t+(1-p_t)\frac{r}{\mu_t+r}^r}\\
        &\quad + I(Y_t>0)\frac{r^2+rY_t}{(\mu_t+r)^2}.
    \end{align*}
    There exists $\kappa_{1,\max} > \kappa_{1,\min} >0$ such that for any time $t$
    \[\kappa_{1,\min} \leq \inf_{\beta\in \mathbb{R}^d}g_t(\beta) \leq \sup_{\beta\in \mathbb{R}^d}g_t(\beta) \leq  \kappa_{1,\max}.\]
    
    \item[Condition 5.] 
    Let 
    \begin{align*}
        g_t(\gamma) &= -I(Y_t=0) \frac{(1-\frac{r}{\mu_t+r}^r)(1-2p_t)p_t(1-p_t)}{p_t+(1-p_t)\frac{r}{\mu_t+r}^r} +I(Y=0)\frac{(1-\frac{r}{\mu_t+r}^r)^2p_t^2(1-p_t)^2}{(p_t+(1-p_t)\frac{r}{\mu_t+r}^r)^2}\\
        &\quad +I(Y_t>0)p_t(1-p_t).
    \end{align*}
     There exists $\kappa_{2,\max} > \kappa_{2,\min} >0$ such that for any time $t$
    \[\kappa_{2,\min} \leq \inf_{\gamma\in \mathbb{R}^d}g_t(\gamma) \leq \sup_{\gamma\in \mathbb{R}^d}g_t(\gamma) \leq  \kappa_{2,\max}.\]
\end{itemize}
In essence,  the conditions necessitate the invertibility of the Fisher information matrix, along with adequate exploration, to be compatible with TS-count variants. They are verified under Assumptions 1-3 in the main manuscript.

\section{Bayesian Regret Bounds}
\subsection{Proof of Bayesian regret for TS-NB}

Based on the log-likelihood function of the NB regression (Equation (2) in the main manuscript), its score function at decision point $t$ is given by:
\begin{eqnarray*}
    S_t(\beta) &=& \sum_{i=1}^{t-1} \frac{r(Y_i-\exp(\phi(A_i,X_i)^{\top}\beta))}{\exp(\phi(A_i,X_i)^{\top}\beta) + r}\phi(A_i,X_i).
\end{eqnarray*} 

We introduce the following event to facilitate the proof:
$$E_{1, t}=\left\{\forall a \in \mathcal{A}: \left|\phi(a,X_t)^{\top} \hat{\beta}_t-\phi(a,X_t)^{\top} \beta^*\right| \leq c_1\left\|\phi(a,X_t)\right\|_{G_t^{-1}}\right\}.$$

The proof of the Bayesian regret bound for TS-NB is based on the following two lemmas.
\begin{lemma}
    \label{lem1}
    Let $c_1=Y_{\max} 
\kappa_{\min}^{-1} \sqrt{d \log (T/ d)-2 \log \delta}$ and $\tau$ be any time point such that the corresponding minimum eigenvalue $\lambda_{\min }\left(G_\tau\right) \geq 1$. Under Condition 1, for any $t \geq \tau$ and $\delta > 0$, 
    \begin{eqnarray*}
        \mathbb{P}\left(\bar{E}_{1, t}\right) \leq \delta,
    \end{eqnarray*}
where $\bar{E}_{1, t}$ denotes the complementary event of $E_{1,t}$.
\end{lemma}

\begin{lemma}
    \label{lem2}
    For any two parameters $u_t < M$ and $u_t^{\prime} < M$, where $M > 0$, the exponential function is locally Lipschitz continuous with constant $s$, i.e., 
    \begin{eqnarray*}
    |\exp(u_t) - \exp(u_t^{\prime})|  \leq s\left|u_{t}-u_{t}^{\prime}\right|.
    \end{eqnarray*}
    In particular, if $u_t > u_t^{\prime}$, we have
    \begin{eqnarray*}
    \exp(u_t) - \exp(u_t^{\prime})  \leq s(u_{t}-u_{t}^{\prime}).
    \end{eqnarray*}
\end{lemma}

\begin{proof}
In TS-NB, $\tilde{\beta}_t$ is independently drawn from a posterior distribution $Q_t(\beta)$ while the posterior belief in the true parameter $\beta^*$ is also represented by $Q_t(\beta)$ conditional on $\mathcal{F}_t$.  Hence, $\tilde{\beta}_t$ and $\beta^*$ are i.i.d. conditional on $\mathcal{F}_t$. Similarly, the optimal intervention $A^*_t$ and the selected intervention $A_t$ are also i.i.d. conditional on $\mathcal{F}_t$.

We define the upper confidence bound for the expected reward as
\begin{eqnarray*}
U_t(a, \hat{\beta}_t) =  \exp \left\{\phi(a,X_t)^{\top} \hat{\beta}_t+c_1\left\|\phi(a,X_t)\right\|_{G_t^{-1}}\right\},
\end{eqnarray*}
where $c_1$ is the confidence radius specified in Lemma \ref{lem1}. Then, we decompose the immediate regret using $U_t$.
\begin{eqnarray*}
\mathbb{E}_t[\Delta_{A_t}] &=& \mathbb{E}_t[R_t(A_t^{\ast},\beta^{\ast})-R_t(A_t,\beta^{\ast})]\\
&=& \mathbb{E}_t\left[R_t\left(A_t^*, \beta^*\right)-U_t\left(A_t^*, \hat{\beta}_t\right) \right]+\mathbb{E}_t\left[U_t\left(A_t^*, \hat{\beta}_t\right)-U_t\left(A_t, \hat{\beta}_t\right) \right]\\
&+& \mathbb{E}_t\left[U_t\left(A_t, \hat{\beta}_t\right)-R_t\left(A_t, \beta^*\right) \right].
\end{eqnarray*}

Note that $\mathbb{E}_t\left[U_t\left(A_t^*, \hat{\beta}_t\right)-U_t\left(A_t, \hat{\beta}_t\right)\right]=0$ since $U_t$ is a deterministic function and $A_t$ and $A^*$ are i.i.d. conditional on $\mathcal{F}_t$. Hence, we only need to bound the two quantities $\mathcal{R}_{\text {Bayes }}^{(1)}(T)$ and $\mathcal{R}_{\text {Bayes }}^{(2)}(T)$ for the Bayesian cumulative regret, as shown below:    
\begin{eqnarray*}
    \sum_{t=1}^T\mathbb{E}_t[\Delta_{A_t}] &=& \sum_{t=1}^{\tau}\mathbb{E}_t[\Delta_{A_t}] + \sum_{t=\tau+1}^T\mathbb{E}_t[\Delta_{A_t}] \\
    &\leq& \tau + \sum_{t=\tau+1}^T \mathbb{E}_t\left[\Delta_{A_t}\right]\\
    &=&\tau  + \underbrace{\sum_{t=\tau+1}^T \mathbb{E}_t\left[R_t\left(A_t^*, \beta^*\right)-U_t\left(A_t^*, \hat{\beta}_t\right)\right]}_{\mathcal{R}_{\text {Bayes }}^{(1)}(T)}+\underbrace{\sum_{t=\tau+1}^T \mathbb{E}_t\left[U_t\left(A_t, \hat{\beta}_t\right)-R_t\left(A_t, \beta^*\right) \right]}_{\mathcal{R}_{\text {Bayes}}^{(2)}(T)}.
    \end{eqnarray*}

In what follows, we present the upper bounds for $\mathcal{R}_{\text {Bayes }}^{(1)}(T)$ and $\mathcal{R}_{\text {Bayes }}^{(2)}(T)$, and then combine these results to establish the Bayesian regret bound for TS-NB.

\textbf{Bounding $\mathcal{R}_{\text {Bayes }}^{(1)}(T)$.} According to Lemma \ref{lem1}, we have $\left|\phi(a,X_t)^{\top} \hat{\beta}_t-\phi(a,X_t)^{\top} \beta^*\right| \leq c_1\left\|\phi(a,X_t)\right\|_{G_t^{-1}}$ with probability at least $1-1/t^2$ for all $a$ (setting $\delta$ to $1/t^2$). Hence, it follows that 
\begin{eqnarray*}
        \phi(a,X_t)^{\top} \beta^* - \left(\phi(a,X_t)^{\top} \hat{\beta}_t + c_1\left\|\phi(a,X_t)\right\|_{G_t^{-1}}\right) \leq 0
    \end{eqnarray*}
for all $a$ with probability at least $1-1/t^2$. Further, due to the monotonicity of $\exp(\cdot)$, we have 
\begin{eqnarray*}
        \sum_{t=\tau+1}^T \mathbb{E}_t\left[R_t\left(A_t^*, \beta^*\right)-U_t\left(A_t^*, \hat{\beta}_t\right)\right] \leq 0+ \sum_{t=\tau+1}^T \mathcal{O}(t^{-2}) =  \mathcal{O}(1).
\end{eqnarray*}

\textbf{Bounding $\mathcal{R}_{\text {Bayes }}^{(2)}(T)$.} In what follows, we first use Lemma \ref{lem2} to upper bound $\mathcal{R}_{\text{Bayes}}^{(2)}(T)$ by the expected difference in linear predictors. Then, we use Lemma \ref{lem1}, together with Cauchy-Schwarz inequality, to bound $\mathcal{R}_{\text{Bayes}}^{(2)}(T)$ by Equation \eqref{equ-r2}: 
\begin{eqnarray}
\sum_{t=\tau+1}^T \mathbb{E}_t\left[U_t\left(A_t, \hat{\beta}_t\right)-R_t\left(A_t, \beta^*\right) \right] & \leq& s\sum_{t=\tau+1}^T\mathbb{E}_t\left[\phi(A_t,X_t)^\top\hat{\beta}_t+c_1\left\|\phi(A_t,X_t)\right\|_{G_t^{-1}}-\phi(A_t,X_t)^\top \beta^* \right] \nonumber \\
& \leq& 2 s c_1 \sum_{t=\tau+1}^T \mathbb{E}_t\left[\left\|\phi(A_t,X_t)\right\|_{G_t^{-1}} \right]  + \mathcal{O}(1), 
\label{equ-r2}
\end{eqnarray}
where $\mathcal{O}(1)$ comes from the failure event of the concentration of $\hat{\beta}_t$ in Lemma \ref{lem1}. 
According to Lemma 2 in \citet{li2017provably}, as Assumption 1 in the main manuscript and $\lambda_
{\min}(G_{\tau}) \geq 1$ are satisfied, we have 
\begin{eqnarray*}
\sum_{t=\tau+1}^T \left\|\phi(A_t,X_t)\right\|_{G_t^{-1}} \leq \sqrt{2 d T \log \left(\frac{T}{d}\right)}.
\end{eqnarray*}

Combining the results, we can derive
\begin{eqnarray*}
\sum_{t=\tau+1}^T \mathbb{E}_t\left[U_t\left(A_t, \hat{\beta}_t\right)-R_t\left(A_t, \beta^*\right) \right] \leq 2s c_1\sqrt{2 d T \log \left(\frac{T}{d}\right)} + \mathcal{O}(1). 
\end{eqnarray*}
where $c_1  =  Y_{\max} 
\kappa_{\min}^{-1} \sqrt{d \log (T/ d)+4 \log T}$ by setting $\delta = 1/T^2$ since $T \geq t$.

\textbf{Combining $\mathcal{R}_{\text {Bayes }}^{(1)}(T)$ and $\mathcal{R}_{\text {Bayes }}^{(2)}(T)$.} Combining the bounds for $\mathcal{R}_{\text{Bayes}}^{(1)}(T)$ and $\mathcal{R}_{\text{Bayes}}^{(2)}(T)$, we have
\begin{eqnarray*}
\mathcal{R}_{\text{Bayes}}(T) \leq \tau  + \mathcal{O}(1) + \left[2s Y_{\max}
\kappa_{\min}^{-1} \sqrt{d \log (T/ d)+4 \log T}\right]\sqrt{2 d T \log \left(\frac{T}{d }\right)}.
\end{eqnarray*}
For completeness, we set $\tau = \sqrt{d}$ to derive the regret bound presented in Theorem 5.1. Since Algorithm 1 does not incorporate the initial exploration state for parameter estimation, it is feasible to optimize the choice of $\tau$ within the regret bound. 
\end{proof}

\subsection{Proof of Bayesian regret for TS-ZIP}
First, we introduce the following events:
\begin{align*}
    &E_{1, t}^{\beta}=\left\{\forall a \in \mathcal{A}:\left|\phi(a,X_t)^{\top} \hat{\beta}_t-\phi(a,X_t)^{\top} \beta^*\right| \leq c_1^{\beta}\left\|\phi(a,X_t)\right\|_{G_t^{-1}}\right\}\\
    &E_{1, t}^{\gamma}=\left\{\forall a \in \mathcal{A}:\left|\phi(a,X_t)^{\top} \hat{\gamma}_t-\phi(a,X_t)^{\top} \gamma^*\right| \leq c_1^{\gamma}\left\|\phi(a,X_t)\right\|_{G_t^{-1}}\right\}
\end{align*}

Notice that the score function of ZIP is
    \begin{eqnarray*}
        S_t(\beta) = \frac{\partial \ln L^{\operatorname{ZIP}}(\beta,\gamma)}{\partial \beta} &=&  \sum_{i=1}^{t-1}I(Y_i = 0) \frac{-(1-p_{i})\exp(-\mu_{i})}{p_{i}+(1-p_{i})\exp(-\mu_{i})} \mu_{i}\phi(A_i,X_i) \\
        && + \sum_{i=1}^{t-1}I(Y_i > 0) (Y_i-\mu_{i}) \phi(A_i,X_i)
    \end{eqnarray*}
    \begin{eqnarray*}
       S_t(\gamma) =  \frac{\partial \ln L^{\operatorname{ZIP}}(\beta,\gamma)}{\partial \gamma} &=& \sum_{i=1}^{t-1}I(Y_i=0) \frac{1-\exp(-\mu_{i})}{p_{i}+(1-p_{i})\exp(-\mu_{i})}p_{i}(1-p_{i})\phi(A_i,X_i) \\
        &&+ \sum_{i=1}^{t-1}I(Y_i > 0) \frac{1}{p_{i}-1}p_{i}(1-p_{i})\phi(A_i,X_i)
    \end{eqnarray*}

To derive the Bayesian regret bound for TS-ZIP, we use the following lemma.
\begin{lemma}
\label{lem3}
    Let $c_1^{\beta} = Y_{\max } \kappa_{1,\min}^{-1} \sqrt{d \log (T / d)-2 \log \delta}$ and $c_1^\gamma = \kappa_{2,\min}^{-1} \sqrt{d \log (T / d)-2 \log \delta}$. Let $\tau$ be any time point such that the corresponding minimum eigenvalue $\lambda_{\min}(G_{\tau}) \geq 1$. Under Conditions 2 and 3, for any $t\geq \tau$ and $\delta_1, \delta_2 >0$, 
    \begin{eqnarray*}
        \mathbb{P}\left(\bar{E}_{1, t}^{\beta} \right) &\leq& \delta_1,\\
        \mathbb{P}\left(\bar{E}_{1, t}^{\gamma}\right) &\leq& \delta_2.
    \end{eqnarray*}
\end{lemma}

\begin{proof}
In TS-ZIP, $\tilde{\gamma}_t$ and $\tilde{\beta}_t$ are independently drawn from the posteriors $Q_{t}(\gamma)$ and $Q_{t}(\beta)$, while the posterior belief in the true parameter$\gamma^*$ and $\beta^*$ is also represented by $Q_{t}(\gamma)$ and $Q_{t}(\beta)$ conditional on $\mathcal{F}_t$. Hence, $\tilde{\beta}_t$ and $\beta^*$ are i.i.d. conditional on $\mathcal{F}_t$, and  $\tilde{\gamma}_t$ and $\gamma^*$ are i.i.d. conditional on $\mathcal{F}_t$. Similarly, the optimal intervention $A_t^*$ and the selected intervention $A_t$ are also i.i.d. conditional on $\mathcal{F}_t$.

We define the upper confidence bound for the expected reward as
\begin{equation*}
    U_t\left(a, \hat{\gamma}_t, \hat{\beta}_t\right)= \left[1-\operatorname{sigmoid}(\phi(a,X_t)^{\top}\hat{\gamma}_t - c_1^{\gamma}\left\|\phi(a,X_t)\right\|_{G_t^{-1}})\right] \exp \left\{\phi(a,X_t)^{\top} \hat{\beta}_t+c_1^{\beta}\left\|\phi(a,X_t)\right\|_{G_t^{-1}}\right\}.
\end{equation*}
where $c_1^{\gamma}$ and $c_1^{\beta}$ are the confidence radius specified in Lemma \ref{lem3}. Then, we decompose the immediate regret using $U_t$,
\begin{eqnarray*}
\mathbb{E}_t[\Delta_{A_t}] &=& \mathbb{E}_t[R_t(A_t^{\ast},\gamma^{\ast},\beta^{\ast})-R_t(A_t,\gamma^{\ast},\beta^{\ast})]\\
&=& \mathbb{E}_t\left[R_t\left(A_t^*,\gamma^{\ast}, \beta^*\right)-U_t\left(A_t^*, \hat{\gamma}_t,\hat{\beta}_t\right) \right]+\mathbb{E}_t\left[U_t\left(A_t^*, \hat{\gamma}_t, \hat{\beta}_t\right)-U_t\left(A_t, \hat{\gamma}_t,\hat{\beta}_t\right) \right]\\
&+& \mathbb{E}_t\left[U_t\left(A_t,\hat{\gamma}_t, \hat{\beta}_t\right)-R_t\left(A_t, \gamma^{\ast},\beta^*\right) \right].
\end{eqnarray*}

Note that $\mathbb{E}_t\left[U_t\left(A_t^*, \hat{\gamma}_t, \hat{\beta}_t\right)-U_t\left(A_t, \hat{\gamma}_t,\hat{\beta}_t\right) \right] = 0$ since $U_t$ is a deterministic function; moreover, $A_t$ and $A^*$ are i.i.d. conditional on $\mathcal{F}_t$. Hence, we only need to bound the two quantities $\mathcal{R}_{\text {Bayes }}^{(1)}(T)$ and $\mathcal{R}_{\text {Bayes }}^{(2)}(T)$ for the Bayesian cumulative regret, as shown below:
    \begin{eqnarray*}
    \sum_{t=1}^T\mathbb{E}_t[\Delta_{A_t}] &=& \sum_{t=1}^{\tau}\mathbb{E}_t[\Delta_{A_t}] + \sum_{t=\tau+1}^T\mathbb{E}_t[\Delta_{A_t}] \\
    &\leq& \tau  + \sum_{t=\tau+1}^T \mathbb{E}_t\left[\Delta_{A_t}\right]\\
    &=&\tau+
    \underbrace{\sum_{t=\tau+1}^T \mathbb{E}_t\left[R_t\left(A_t^*,\gamma^*, \beta^*\right)-U_t\left(A_t^*, \hat{\gamma}_t, \hat{\beta}_t\right)\right]}_{\mathcal{R}_{\text {Bayes }}^{(1)}(T)}
    + \underbrace{\sum_{t=\tau+1}^T \mathbb{E}_t\left[U_t\left(A_t,  \hat{\gamma}_t,\hat{\beta}_t\right)-R_t\left(A_t, \gamma^*, \beta^*\right) \right]}_{\mathcal{R}_{\text {Bayes}}^{(2)}(T)}.
    \end{eqnarray*}

In what follows, we present the upper bounds for $\mathcal{R}_{\text {Bayes }}^{(1)}(T)$ and $\mathcal{R}_{\text {Bayes }}^{(2)}(T)$, and then combine these results to establish the Bayesian regret bound for TS-ZIP.

\textbf{Bounding $\mathcal{R}_{\text {Bayes }}^{(1)}(T)$.} 
According to Lemma \ref{lem3}, we have $\left|\phi(a,X_t)^{\top} \hat{\beta}_t-\phi(a,X_t)^{\top} \beta^*\right| \leq c_1^{\beta}\left\|\phi(a,X_t)\right\|_{G_t^{-1}}$  with probability at least $1-1/t^2$, and $\left|\phi(a,X_t)^{\top} \hat{\gamma}_t-\phi(a,X_t)^{\top} \gamma^*\right| \leq c_1^{\gamma}\left\|\phi(a,X_t)\right\|_{G_t^{-1}}$ with probability at least $1-1/t^2$. Due to the monotonicity of $\exp(\cdot)$ and the sigmoid function, we have
\begin{eqnarray*}
    \sum_{\tau +1}^T\mathbb{E}_t\left[R_t(A_t^*, \gamma^* ,\beta^*)-U_t(A_t^*, \hat{\gamma}_t ,\hat{\beta}_t)\right] \leq  0 + \sum_{\tau +1}^T\mathcal{O}(t^{-2}+t^{-2}-t^{-4}) = \mathcal{O}(1)
\end{eqnarray*}

\textbf{Bounding $\mathcal{R}_{\text {Bayes }}^{(2)}(T)$.} 
Let $f(x, y) = \frac{1}{1+\exp(x)}\exp(y)$. Using Taylor expansion, we have
\begin{eqnarray*}
    &&U_t\left(A_t, \hat{\gamma}_t, \hat{\beta}_t\right)-R_t\left(A_t, \gamma^*, \beta^*\right)\\
    &=& f(\phi(A_t,X_t)\hat{\gamma}-c_1^\gamma\|\phi(A_t,X_t)\|_{G_t^{-1}}, \phi(A_t,X_t)\hat{\beta}+c_1^\beta\|\phi(A_t,X_t)\|_{G_t^{-1}}) \\
    && -f(\phi(A_t,X_t)\gamma^*, \phi(A_t,X_t)\beta^*)\\
    &=& \nabla f(\phi(A_t,X_t)^\top\bar{\gamma}, \phi(A_t,X_t)^\top\bar{\beta})^\top \left(\begin{array}{cc}
        \phi(A_t,X_t)\hat{\gamma}-c_1^\gamma\|\phi(A_t,X_t)\|_{G_t^{-1}} -\phi(A_t,X_t)\gamma^*   \\
       \phi(A_t,X_t)\hat{\beta}+c_1^\beta\|\phi(A_t,X_t)\|_{G_t^{-1}} -\phi(A_t,X_t)\beta^*
    \end{array}\right)\\
    &\leq& \left\|\nabla f(\phi(A_t,X_t)^\top\bar{\gamma}, \phi(A_t,X_t)^\top\bar{\beta})\right\| \left\|\begin{array}{cc}
        \phi(A_t,X_t)\hat{\gamma}-c_1^\gamma\|\phi(A_t,X_t)\|_{G_t^{-1}} -\phi(A_t,X_t)\gamma^*   \\
       \phi(A_t,X_t)\hat{\beta}+c_1^\beta\|\phi(A_t,X_t)\|_{G_t^{-1}} -\phi(A_t,X_t)\beta^*
    \end{array}\right\|\\
    &\leq& \sup_{\|o_1\|\leq1,\|o_2\|\leq1} \left\|\nabla f(\bar{o}_1, \bar{o}_2)\right\| \\
    &&(\|\phi(A_t,X_t)\hat{\gamma}-c_1^\gamma\|\phi(A_t,X_t)\|_{G_t^{-1}} -\phi(A_t,X_t)\gamma^*\| \\
    &&+ \|\phi(A_t,X_t)\hat{\beta}+c_1^\beta\|\phi(A_t,X_t)\|_{G_t^{-1}} -\phi(A_t,X_t)\beta^*\|)\\
    &\leq& \sup_{\|o_1\|\leq1,\|o_2\|\leq1} \left\|\nabla f(\bar{o}_1, \bar{o}_2)\right\| (2c_1^\gamma\|\phi(A_t,X_t)\|_{G_t^{-1}} + 2c_1^\beta\|\phi(A_t,X_t)\|_{G_t^{-1}})
\end{eqnarray*}
where $q = \sqrt{2}e$ under Assumption 1.  

Similar to the case of TS-NB, we have
\begin{eqnarray*}
    \sum_{\tau +1}^T\mathbb{E}_t\left[U_t\left(A_t, \hat{\gamma}_t, \hat{\beta}_t\right)-R_t\left(A_t, \gamma^*, \beta^*\right)\right] &\leq& 2q\sum_{\tau +1}^T\mathbb{E}_t\left[c_1^\gamma\|\phi(A_i,X_i)\|_{G_t^{\gamma,-1}} + \|\phi(A_i,X_i)\|_{G_t^{\beta,-1}}\right] + \mathcal{O}(1)\\
    &\leq& 2qc_1^\gamma\sqrt{2 d T \log \left(\frac{T}{d}\right)}  +2qc_1^\beta\sqrt{2 d T \log \left(\frac{T}{d}\right)} + \mathcal{O}(1)
\end{eqnarray*}
where $ \mathcal{O}(1)$ comes from the failure event of the concentration of $\hat{\beta}_t$ and $\hat{\gamma}_t$ in Lemma \ref{lem3}. 

\textbf{Combining  $\mathcal{R}_{\text {Bayes }}^{(1)}(T)$ and $\mathcal{R}_{\text {Bayes }}^{(2)}(T)$.} 
Therefore, we have
\begin{eqnarray*}
    \mathcal{R}_{\text {Bayes }}(T) &\leq& \tau+\mathcal{O}(1)+\left[2 q Y_{\max } \kappa_{1,\min}^{-1} \sqrt{d \log (T / d)+4 \log T}\right] \sqrt{2 d T \log \left(\frac{T}{d}\right)}\\
    &+& \left[2 q \kappa_{2,\min}^{-1} \sqrt{d \log (T / d)+4 \log T}\right] \sqrt{2 d T \log \left(\frac{T}{d}\right)}
\end{eqnarray*}
For completeness, we set $\tau = \sqrt{d}$ to derive the regret bound presented in Theorem 5.2. Since Algorithm 2 does not incorporate the initial exploration state for parameter estimation, it is feasible to optimize the choice of $\tau$ within the regret bound. 
\end{proof}

\subsection{Proof of Bayesian regret for TS-ZINB}
First, the score functions of ZINB are
    \begin{eqnarray*}
        S_t(\beta) = \frac{\partial \ln L^{\operatorname{ZINB}}(\beta,\gamma)}{\partial \beta} &=&  \sum_{i=1}^{t-1}I(Y_i = 0) \frac{(1-p_i)r\frac{r}{\mu_i+r}^{r-1}\frac{-r}{(\mu_i+r)^2}}{p_i+(1-p_i)\frac{r}{\mu_i+r}^r}\mu_i\phi(A_i,X_i) \\
        &+& I(Y_i>0)\frac{r}{\mu_i+r}(Y_i-\mu_i)\phi(A_i,X_i)
    \end{eqnarray*}
    \begin{eqnarray*}
       S_t(\gamma) =  \frac{\partial \ln L^{\operatorname{ZINB}}(\beta,\gamma)}{\partial \gamma} &=& \sum_{i=1}^{t-1}I(Y_i=0) \frac{1-\frac{r}{\mu_i+r}^r}{p_i+(1-p_i)\frac{r}{\mu_i+r}^r}p_i(1-p_i)\phi(A_i,X_i) \\
        &-& \sum_{i=1}^{t-1}I(Y_i > 0)p_{i}\phi(A_i,X_i)
    \end{eqnarray*}

The derivation of Bayesian regret for TS-ZINB is similar to the case of TS-ZIP. We only need a slightly different Lemma \ref{lem4}, which mimics Lemma \ref{lem3}. 

\begin{lemma}
\label{lem4}
    Let $c_1^{\beta} = Y_{\max } \kappa_{1,\min}^{-1} \sqrt{d \log (T / d)-2 \log \delta}$ and $c_1^\gamma = \kappa_{2,\min}^{-1} \sqrt{d \log (T / d)-2 \log \delta}$. Let $\tau$ be any time point such that the corresponding minimum eigenvalue $\lambda_{\min}(G_{\tau}) \geq 1$. Under Conditions 4 and 5, for any $t\geq \tau$ and $\delta_1, \delta_2 >0$, 
    \begin{eqnarray*}
        \mathbb{P}\left(\bar{E}_{1, t}^{\beta} \leq 1\right) &\leq& \delta_1,\\
        \mathbb{P}\left(\bar{E}_{1, t}^{\gamma} \leq 1\right) &\leq& \delta_2.
    \end{eqnarray*}
\end{lemma}

\section{Frequentist Regret Bounds}
\subsection{Proof of frequentist regret for TS-NB}
First, we introduce the following events:
\begin{align*}
    &E_{2, t}=\left\{\forall a \in \mathcal{A}:\left|\phi(a,X_t)^{\top} \tilde{\beta}_t-\phi(a,X_t)^{\top} \hat{\beta}_t\right| \leq c_2\left\|\phi(a,X_t)\right\|_{G_t^{-1}}\right\}\\
    &E_{3, t}=\left\{\phi(A_t^*,X_t)^{\top} \tilde{\beta}_t-\phi(A_t^*,X_t)^{\top} \hat{\beta}_t>c_1\left\|\phi(A_t^*,X_t)\right\|_{G_t^{-1}}\right\}
\end{align*}
Here, $E_{2,t}$ indicates that $\tilde{\beta}_t \rightarrow \hat{\beta}_t$ , while $E_{3,t}$ argues that $\tilde{\beta}_t$ is sufficiently optimistic.

Below we provide an important lemma in the proof of frequentist regret for TS-NB.  
\begin{lemma}
    \label{lem5}
    Let $p_2 \geq \mathbb{P}_t\left(\bar{E}_{2, t}\right), p_3 \leq \mathbb{P}_t\left(E_{3, t}\right)$, and $p_3>p_2$. Then on event $E_{1, t}$,
    \begin{eqnarray*}
        \mathbb{E}_t\left[\Delta_{A_t}\right] \leq s\left(c_1+c_2\right)\left(1+\frac{2}{p_3-p_2}\right) \times \mathbb{E}_t\left[\left\|\phi(A_t,X_t)\right\|_{G_t^{-1}}\right]+ p_2.
    \end{eqnarray*}
\end{lemma}

We also need Lemma \ref{lem6} to bound $\mathbb{P}_t(\bar{E}_{2,t})$, and $\mathbb{P}_t(E_{3,t})$.
\begin{lemma}
    \label{lem6}
    Let 
    \begin{eqnarray*}
        \alpha = c_1\sqrt{\kappa_{\max}}, \ c_2 = c_1\sqrt{2\kappa_{\min}^{-1}\kappa_{\max}\log(KT)}.
    \end{eqnarray*}
    Under Condition 1, $\mathbb{P}_t\left(\bar{E}_{2, t}\right) \leq 1 / T$ and $\mathbb{P}_t\left(E_{3, t}\right) \geq 0.15$ hold. 
\end{lemma}

\begin{proof}
Let $p_1 \geq \mathbb{P}\left(\bar{E}_{1, t}\right)$. Through elementary algebra, we get
    \begin{eqnarray*}
\mathcal{R}(T,\beta^*) & \leq & \sum_{t=\tau+1}^T \mathbb{E}\left[\Delta_{A_t}\right]+\tau \Delta_{\max}  \\
& \leq & \sum_{t=\tau+1}^T \mathbb{E}\left[\Delta_{A_t} \mathds{1}\left\{E_{1, t}\right\}\right]+\left(\tau+p_1 T\right)\Delta_{\max}  \\
& = & \sum_{t=\tau+1}^T \mathbb{E}\left[\mathbb{E}_t\left[\Delta_{A_t}\right] \mathds{1}\left\{E_{1, t}\right\}\right]+\left(\tau+p_1 T\right)\Delta_{\max}.
    \end{eqnarray*}
To get $p_1 \leq 1 / T$, we set $c_1$ as in Lemma \ref{lem1} with $\delta=1/T$. Then we apply Lemma \ref{lem5} to $\mathbb{E}_t\left[\Delta_{A_t}\right] \mathds{1}\left\{E_{1, t}\right\}$ and get
\begin{eqnarray*}
    \mathcal{R}(T,\beta^*) \leq s\left(c_1+c_2\right)\left(1+\frac{2}{p_3-p_2}\right) \mathbb{E}\left[\sum_{t=\tau+1}^T \left\|\phi(A_t,X_t)\right\|_{G_t^{-1}}\right]+\left(\tau+\left(p_1+p_2\right) T\right)\Delta_{\max},
\end{eqnarray*}
where $\alpha$ and $c_2$ are set as in Lemma \ref{lem4}. For these settings, $p_2 \leq 1 / T$ and $p_3 \geq 0.15$. Finally, to bound $\sum_{t=\tau+1}^T \left\|\phi(A_t,X_t)\right\|_{G_t^{-1}}$, we use Lemma 2 in \citet{li2017provably}. As a result, we have
\begin{eqnarray*}
     \mathcal{R}(T,\beta^*) \leq s(c_1+c_2)\left(1+\frac{2}{0.15-1/T}\right)\times \sqrt{2 dT \log (T / d)}+(\tau+2)\Delta_{\max} .
\end{eqnarray*}
\end{proof}

\subsection{Proof of frequentist regret for TS-ZIP}
First, we introduce the following events:
\begin{align*}
    &E_{2,t}^\beta = \left\{\forall a \in \mathcal{A}:\left|\phi(a,X_t)^{\top} \tilde{\beta}_t-\phi(a,X_t)^{\top} \hat{\beta}\right| \leq c_1^{\beta}\left\|\phi(a,X_t)\right\|_{G_t^{-1}}\right\}\\
    &E_{2,t}^\gamma = \left\{\forall a \in \mathcal{A}:\left|\phi(a,X_t)^{\top} \tilde{\gamma}_t-\phi(a,X_t)^{\top} \hat{\gamma}\right| \leq c_1^{\gamma}\left\|\phi(a,X_t)\right\|_{G_t^{-1}}\right\}\\
    &E_{3,t}^\beta=\left\{\phi(A_t^*,X_t)^\top\tilde{\beta}_t - \phi(A_t^*,X_t)^\top\hat{\beta}_t > c_1^\beta\|\phi(A_t^*,X_t)\|\right\}\\
    &E_{3,t}^\gamma=\left\{\phi(A_t^*,X_t)^\top \tilde{\gamma}_t - \phi(A_t^*,X_t)^\top\hat{\gamma}_t <- c_1^\gamma\|\phi(A_t^*,X_t)\|\right\}
\end{align*}
Here, $E_{2,t}^\beta$ and $E_{2,t}^\gamma$ indicate that $\tilde{\beta}_t \rightarrow \hat{\beta}_t$ and $\tilde{\gamma}_t \rightarrow \hat{\gamma}_t$, while $E_{3,t}^\beta$ and $E_{3,t}^\gamma$ argue that $\tilde{\beta}_t$ and $\tilde{\gamma}_t$ are sufficiently optimistic.

Below we provide two useful lemmas for proving the frequentist regret of TS-ZIP. 
\begin{lemma}
\label{lem7}
    Let $p_2^\beta \geq \mathbb{P}_t(\bar{E}_{2,t}^\beta)$, $p_2^\gamma \geq \mathbb{P}_t(\bar{E}_{2,t}^\gamma)$, $p_3^\beta \leq \mathbb{P}_t(E_{3,t}^\beta)$, $p_3^\gamma \leq \mathbb{P}_t(E_{3,t}^\gamma)$, and $p_3^\beta p_3^\gamma  + p_2^\beta p_2^\gamma > p_2^\beta+p_2^\gamma$. Then on event $E_{1,t}^\beta$ and $E_{1,t}^\gamma$,
    \begin{align*}
        \mathbb{E}_t[\Delta_{A_t}] &\leq q(c_1^\beta+c_2^\beta + c_1^\gamma+c_2^\gamma)\left(1+\frac{2}{p_3^\beta p_3^\gamma - p_2^\beta-p_2^\gamma + p_2^\beta p_2^\gamma}\right)\times \mathbb{E}_t\left[\|\phi(A_t,X_t)\|_{G_t^{-1}}\right] \\
        &+ (\mathbb{P}_t(\bar{E}_{2,t}^\beta) + \mathbb{P}_t(\bar{E}_{2,t}^\gamma)- \mathbb{P}_t(\bar{E}_{2,t}^\beta)\mathbb{P}_t(\bar{E}_{2,t}^\gamma)).
    \end{align*}
\end{lemma}

\begin{lemma}
    \label{lem8}
    Let $\alpha_\beta = c_1^\beta\sqrt{\kappa_{1,\max}}$, $\alpha_\gamma = c_1^\gamma\sqrt{\kappa_{2,\max}}$, $c_2^\beta = c_1^\beta \sqrt{2 \kappa_{1, \min }^{-1} \kappa_{1, \max } \log (K T)}$, $c_2^\gamma = c_1^\gamma \sqrt{2 \kappa_{2, \min }^{-1} \kappa_{2, \max } \log (K T)}$.
    Under Conditions 2-3, $\mathbb{P}_t(\bar{E}_{2,t}^\beta) \leq 1/T$, $\mathbb{P}_t(\bar{E}_{2,t}^\gamma) \leq 1/T$, $\mathbb{P}_t(E_{3,t}^\beta) \geq 0.15$, and $\mathbb{P}_t(E_{3,t}^\gamma) \geq 0.15$ hold.
\end{lemma}

\begin{proof}
    Fix $\tau \in [T]$. Let $p_1 = 1-\mathbb{P}(E_{1,t}^\beta, E_{1,t}^\gamma)$. Through some trivial algebra, we have
    \begin{align*}
        \mathcal{R}(T,\gamma^*, \beta^*) &\leq \sum_{t=\tau+1}^T\mathbb{E}[\Delta_{A_t}] + \tau\Delta_{\max} \\
        &\leq \sum_{t=\tau+1}^T\mathbb{E}[\Delta_{A_t}\mathbbm{1}\left\{E_{1,t}^\beta, E_{1,t}^\gamma\right\}] + (\tau+p_1T)\Delta_{\max} \\
        &=\sum_{t=\tau+1}^T\mathbb{E}[\mathbb{E}_t[\Delta_{A_t}]\mathbbm{1}\left\{E_{1,t}^\beta, E_{1,t}^\gamma\right\}] + (\tau+p_1T) \Delta_{\max}
    \end{align*}
By setting $c_1^\beta$ and $c_1^\gamma$ as in Lemma \ref{lem3} with $\delta = 1/T$, we have $p_1 \leq 1-(1-1/T)^2 = 2/T-1/T^2$. Then, by applying Lemma \ref{lem7} to $\mathbb{E}_t[\Delta_{A_t}]\mathbbm{1}\left\{E_{1,t}^\beta, E_{1,t}^\gamma\right\}$, we have
\begin{align*}
     \mathcal{R}(T,\gamma^*, \beta^*) &\leq q(c_1^\beta+c_2^\beta + c_1^\gamma+c_2^\gamma)\left(1+\frac{2}{p_3^\beta p_3^\gamma - p_2^\beta-p_2^\gamma + p_2^\beta p_2^\gamma}\right)\times \mathbb{E}_t\left[\sum_{\tau+1}^T\|\phi(A_t,X_t)\|_{G_t^{-1}}\right]\\
     &+(\tau + (p_1 + p_2^\beta + p_2^\gamma - p_2^\beta p_2^\gamma)T)\Delta_{\max}
\end{align*}
By setting $\alpha^\beta$, $\alpha^\gamma$, $c_2^\beta$, and $c_2^\gamma$ as in Lemma \ref{lem8}, we have $p_2^\beta \leq 1/T$, $p_2^\gamma \leq 1/T$, $p_3^\beta \geq 0.15$, and $p_3^\gamma \geq 0.15$. Finally, to bound $\sum_{\tau+1}^T\|\phi(A_t,X_t)\|_{G_t^{-1}}$, we use Lemma 2 in \citet{li2017provably}. As a result, we have
\begin{align*}
    \mathcal{R}(T,\gamma^*, \beta^*) &\leq q(c_1^\beta+c_2^\beta + c_1^\gamma + c_2^\gamma) \left(1+\frac{2}{0.0225 - 2/T + 1/T^2}\right) \times \sqrt{2dTlog(T/d)}\\
    &+ (\tau + 4 - 2/T)\Delta_{\max}.
\end{align*}
    
\end{proof}

The derivation of the frequentist regret for TS-ZINB is similar to that of TS-ZIP, except that we substitute Conditions 2-3 with Conditions 4-5, respectively. Hence, we omit the proof.

\section{Technical Lemmas}
\subsection{Proof of Lemma \ref{lem1}}
\begin{proof}
 Let $\epsilon_t = (Y_t-\exp(\phi(A_t,X_t)^{\top}\beta^*))$ denote the noise at time $t$ corresponding to an intervention $A_t$.  We assume that $\epsilon_t$ is conditionally sub-Gaussian in the sense that there exists some $\sigma > 0$ such that for any $\gamma \geq 0 $, $t\geq 1$, 
\begin{eqnarray*}
    \mathbb{E}\left[\exp \left(\gamma \epsilon_t\right) \mid \mathcal{F}_{t}\right] \leq \exp \left(\frac{\gamma^2 \sigma^2}{2}\right) \quad \text { a.s. }
\end{eqnarray*}
It can be easily verified that $\epsilon_t$ is sub-Gaussian. Under Assumption 2 in the main manuscript, $\epsilon_t \in [e_t-Y_{\max}, e_t+Y_{\max}]$ holds almost surely for some $\mathcal{F}_t-$measurable random variable $e_t$. By applying Hoeffding's lemma, we obtain for all $\gamma \in \mathbb{R}$,
\begin{eqnarray*}
    \mathbb{E}\left[\exp \left\{\gamma \epsilon_t\right\} \mid \mathcal{F}_t\right] \leq \exp \left\{\gamma \mathbb{E}\left[\epsilon_t \mid \mathcal{F}_t\right]\right\} \exp \left\{\frac{4 Y_{\max }^2 \gamma^2}{8}\right\}=\exp \left\{\frac{\gamma^2 Y_{\max }^2}{2}\right\}.
\end{eqnarray*}
Therefore, $\epsilon_t$ satisfies the sub-Gaussian conditions with $\sigma = Y_{\max}$. In particular, this holds if $|\epsilon_t| \leq Y_{\max}$ almost surely.

Further, let $Z_t = \sum_{i=1}^{t-1}(Y_i - \exp(\phi(A_i,X_i)^{\top}\beta^*))\phi(A_i,X_i) = \sum_{i=1}^{t-1}\epsilon_i \phi(A_i,X_i)$. 
 Notice that
    \begin{eqnarray*}
        S_t(\beta^*) = \sum_{i=1}^{t-1} \frac{r\epsilon_i}{\exp(\phi(A_i,X_i)^{\top}\beta) + r}\phi(A_i,X_i),
    \end{eqnarray*}
    and $S_t(\hat{\beta}_t)=0$ as $\hat{\beta}_t$ is the minimizer of the log-likelihood function. 
    For any $\beta_1, \beta_2 \in \mathbb{R}^d$, we have that 
    \begin{eqnarray*}
S_t\left(\beta_1\right) - S_t\left(\beta_2\right) & =& \sum_{i=1}^{t-1}\left[\frac{r\phi(A_i,X_i)\left(Y_i - \exp (\phi(A_i,X_i)^{\top} \beta_1)\right)}{r+ \exp (\phi(A_i,X_i)^{\top} \beta_1)}-\frac{r\phi(A_i,X_i)\left(Y_i - \exp (\phi(A_i,X_i)^{\top} \beta_2)\right)}{r+ \exp (\phi(A_i,X_i)^{\top} \beta_2)}\right] \\
& =&\sum_{i=1}^{t-1}\left[\nabla \frac{r\left(Y_i - \exp (\phi(A_i,X_i)^{\top} \bar{\beta})\right)}{r+\exp (\phi(A_i,X_i)^{\top} \bar{\beta})}\right] \phi(A_i,X_i)\left(\beta_1-\beta_2\right) \\
& =&\left[-\sum_{i=1}^{t-1} \frac{r\left(r+ Y_i\right) \exp (\phi(A_i,X_i)^{\top} \bar{\beta}) }{\left(r+ \exp (\phi(A_i,X_i)^{\top} \bar{\beta})\right)^2}\phi(A_i,X_i) \phi(A_i,X_i)^{\top}\right]\left(\beta_1-\beta_2\right)\\
& \coloneqq & H_t \left(\beta_2-\beta_1\right). 
    \end{eqnarray*}
where $\bar{\beta} = c\beta_1+(1-c)\beta_2$ for $c\in [0,1]$ and $H_t$ is the Hessian matrix of the negative NB log-likelihood and is convex. The second equality holds because of the mean value theorem.

Now let us consider $\hat{\beta}_t$ and $\beta^*$. Using the preceding equation, we obtain
\begin{eqnarray*}
    S_t(\beta^*)-S_t(\hat{\beta}_t) =S_t(\beta^*)= H_t(\hat{\beta}_t-\beta^*).
\end{eqnarray*}
Under Condition 1, we have $\kappa_{\min} G_t \preceq H_t$. Fixing intervention $a$, we have
\begin{eqnarray*}
    \left|\phi(a,X_t)^{\top}\hat{\beta}_t -\phi(a,X_t)^{\top}\beta^* \right| &\leq& \|\hat{\beta}_t-\beta^*\|_{G_t} \|\phi(a,X_t)\|_{G_t^{-1}} = \sqrt{(\hat{\beta}_t-\beta^*)^{\top}G_t(\hat{\beta}_t-\beta^*)} \|\phi(a,X_t)\|_{G_t^{-1}}\\
    &=& \sqrt{S_t(\beta^*)^{\top}H_t^{-1}G_t H_t^{-1} S_t(\beta^*)} \|\phi(a,X_t)\|_{G_t^{-1}}\\
    &\leq& \kappa_{\min}^{-1}\|S_t(\beta^*)\|_{G_t^{-1}} \|\phi(a,X_t)\|_{G_t^{-1}},
\end{eqnarray*}
Here, the first inequality follows from the Cauchy-Schwarz inequality and others from the above argument. 

Using the definition of $S_t(\beta^*)$ and the fact that $\frac{r}{r+\exp(X_{i,A_i}^{\top}\beta^*)} < 1$, we can show that $\|S_t(\beta^*)\|_{G_t^{-1}} \leq \|Z_t\|_{G_t^{-1}}$. To bound $\|Z_t\|_{G_t^{-1}}$, we can use Theorem 1 in \citet{abbasi2011improved}, which states that if the noise is sub-Gaussian with parameter $\sigma$ (in our case, $\sigma = Y_{\max}$),  then with probability at least $1-\delta$, we have
\begin{eqnarray*}
    \|Z_t\|_{G_t^{-1}}^2 \leq 2\sigma^2 \log(\det(G_t)^{1/2}\det(G_{\tau}^{-1/2})/\delta).
\end{eqnarray*}

By Lemma 11 in \citet{abbasi2011improved} and Assumption 1 from the main manuscript, we have $\log \det(G_t) \leq d\log(T/d)$. By the choice of $\tau$, $\det(G_{\tau})^{-1} \leq 1$. It follows that
\begin{eqnarray*}
    \|S_t(\beta^*)\|_{G_t^{-1}}^2 \leq \|Z_t\|_{G_t^{-1}}^2  \leq Y_{\max}^2(d\log(T/d)-2\log \delta)
\end{eqnarray*}
for any $t \geq \tau$ with probability at least $1-\delta$. In this case, event $E_{1,t}$ is guaranteed to occur when $c_1 = Y_{\max}
\kappa_{\min}^{-1} \sqrt{d \log (T/ d)-2 \log \delta}$.  
\end{proof}

\subsection{Proof of Lemma \ref{lem2}}
\begin{proof}
    According to the mean value theorem, there exists $\bar{u}_t \in [u_t, u_t^{\prime}]$ such that 
    \begin{eqnarray*}
        \frac{\exp(u_t)-\exp(u_t^{\prime})}{u_t-u_t^{\prime}} = \exp(\bar{u}_t).
    \end{eqnarray*}
    As $u_t$ and $u_t^{\prime}$ are bounded by $M$, we have $\exp(\bar{u}_t) < \exp(M)$. Hence, there exists a constant $c$ such that $|\exp(u_t)-\exp(u_t^{\prime})| \leq s|u_t-u_t^{\prime}|$. If $u_t > u_t^{\prime}$, it is easy to see that $\exp(u_t) - \exp(u_t^{\prime})  \leq s(u_{t}-u_{t}^{\prime})$.
\end{proof}
\subsection{Proof of Lemma \ref{lem3}}
\begin{proof}
Recall that $p_{t} = \operatorname{sigmoid}(\phi(A_t,X_t)^\top\gamma)$ and $\mu_{t} = \exp \left(\phi(A_t,X_t)^{\top} \beta\right)$.
Let $\epsilon^{\beta}_t = I(Y_t = 0) \frac{-(1-p_{t})\exp(-\mu_{t})}{p_{t}+(1-p_{t})\exp(-\mu_{t})} \mu_{t} + I(Y_t > 0) (Y_t-\mu_{t})$ and $\epsilon^{\gamma}_t = I(Y_t=0) \frac{1-\exp(-\mu_{t})}{p_{t}+(1-p_{t})\exp(-\mu_{t})}p_{t}(1-p_{t}) - I(Y_t > 0)p_t$. Now we verify that $\epsilon^{\beta}_t$ and $\epsilon^{\gamma}_t$ are sub-Gaussian. Under Assumption 2, $\epsilon^{\beta}_t \in [e_t-Y_{\max}, e_t+Y_{\max}]$ holds almost surely for some $\mathcal{F}_t$-measurable random variable $e_t$. By applying Hoeffding's Lemma, we obtain for all $q \in \mathbb{R}$,
\begin{eqnarray*}
    \mathbb{E}[\exp\{q\epsilon^{\beta}_t\}|\mathcal{F}_t] \leq \exp\{q\mathbb{E}[\epsilon^{\beta}_t|\mathcal{F}_t]\} \exp\left\{\frac{Y_{\max}^2q^2}{2}\right\} = \exp\left\{\frac{Y_{\max}^2q^2}{2}\right\}.
\end{eqnarray*}
Hence, $\epsilon^{\beta}_t$ satisfies the sub-Gaussian conditions with $\sigma = Y_{\max}$. Similarly, we have $\epsilon^{\gamma}_t \in [e_t-1, e_t+1]$ holds almost surely for some $\mathcal{F}_t$-measurable random variable $e_t$. By applying Hoeffding's Lemma, we obtain for all $q \in \mathbb{R}$,
\begin{eqnarray*}
    \mathbb{E}[\exp\{q\epsilon^{\gamma}_t\}|\mathcal{F}_t] \leq \exp\{q\mathbb{E}[\epsilon^{\gamma}_t|\mathcal{F}_t]\} \exp\left\{\frac{q^2}{2}\right\} = \exp\left\{\frac{q^2}{2}\right\}.
\end{eqnarray*}
Hence, $\epsilon^{\gamma}_t$ satisfies the sub-Gaussian conditions with $\sigma = 1$.

Notice that $S_t(\beta) = \sum_{i=1}^{t-1}\epsilon^{\beta}_i \phi(A_i,X_i)$ and $S_t(\gamma) = \sum_{i=1}^{t-1}\epsilon^{\gamma}_i\phi(A_i,X_i)$. For any $\beta_1, \beta_2 \in \mathbb{R}^d$, we have
\begin{eqnarray*}
    S_t(\beta_1) - S_t(\beta_2) &=& \sum_{i=1}^{t-1}I(Y_i = 0) \frac{-(1-p_{i})\exp(-\mu_{i,1})}{p_{i}+(1-p_{i})\exp(-\mu_{i,1})} \mu_{i,1}\phi(A_i,X_i) +I(Y_i > 0) (Y_i-\mu_{i,1}) \phi(A_i,X_i)\\
    &-& I(Y_i = 0) \frac{-(1-p_{i})\exp(-\mu_{i,2})}{p_{i}+(1-p_{i})\exp(-\mu_{i,2})} \mu_{i,2}\phi(A_i,X_i) - I(Y_i > 0) (Y_i-\mu_{i,2}) \phi(A_i,X_i)\\
    &=&\sum_{i=1}^{t-1} \left[\nabla_{\beta = \bar{\beta}} I(Y_i = 0) \frac{-(1-p_{i})\exp(-\mu_{i})}{p_{i}+(1-p_{i})\exp(-\mu_{i})} \mu_{i} + I(Y_i > 0) (Y_i-\mu_{i})  \right]\phi(A_i,X_i)(\beta_1-\beta_2)\\
    &=& \sum_{i=1}^{t-1}  \left[\left(I(Y_i=0)\frac{-(1-p_i)\mu_i\exp(-\mu_i)}{p_i+(1-p_i)\exp(-\mu_i)}+I(Y_i=0)\frac{p_i(1-p_i)\mu_i^2\exp(-\mu_i)}{(p_i+(1-p_i)\exp(-\mu_i))^2} \right. \right.\\
    &-& \left. \left.I(Y_i>0)\mu_i\right)\phi(A_i,X_i)\phi(A_i,X_i)^{\top}\right](\beta_1-\beta_2)\\
    &=&H_t(\beta)(\beta_2-\beta_1)
\end{eqnarray*}
where $H_t(\beta)$ is the left upper block $d\times d$ entry of the Hessian matrix of the negative ZIP log-likelihood.

Then, we have
\begin{eqnarray*}
    S_t(\beta^*) - S_t(\hat{\beta}_t) = S_t(\beta^*) = H_t(\beta)\cdot(\hat{\beta}_t - \beta^*).
\end{eqnarray*}
Under Condition 2, we have $\kappa_{1,\min}G_t \preceq H_t$. Fixing intervention $a$, we have
\begin{eqnarray*}
    \left|\phi(a,X_t)^{\top} \hat{\beta}_t-\phi(a,X_t)^{\top} \beta^*\right| &\leq& \left\|\hat{\beta}_t-\beta^*\right\|_{G_t}\left\|\phi(a,X_t)\right\|_{G_t^{ -1}}=\sqrt{\left(\hat{\beta}_t-\beta^*\right)^{\top} G_t\left(\hat{\beta}_t-\beta^*\right)}\left\|\phi(a,X_t)\right\|_{G_t^{-1}}\\
    &=& \sqrt{S_t\left(\beta^*\right)^{\top} H_t^{-1}(\beta) G_t H_t^{-1}(\beta) S_t\left(\beta^*\right)}\left\|\phi(a,X_t)\right\|_{G_t^{-1}}\\
    &\leq& \kappa_{1,\min}^{-1} \left\|S_t\left(\beta^*\right)\right\|_{G_t^{-1}}\left\|\phi(a,X_t)\right\|_{G_t^{-1}}
\end{eqnarray*}
where the first inequality follows from the Cauchy-Schwarz inequality and others from the
above argument.

For any $\eta_1,\eta_2 \in \mathbb{R}^d$, we have
\begin{eqnarray*}
    S_t(\eta_1) - S_t(\eta_2) &=& \sum_{i=1}^{t-1}I(Y_i=0) \frac{1-\exp(-\mu_{i})}{p_{i,1}+(1-p_{i,1})\exp(-\mu_{i})}p_{i,1}(1-p_{i,1})\phi(A_i,X_i) - I(Y_i > 0) p_{i,1}\phi(A_i,X_i) \\
    &-& I(Y_i=0)\frac{1-\exp(-\mu_{i})}{p_{i,2}+(1-p_{i,2})\exp(-\mu_{i})}p_{i,2}(1-p_{i,2})\phi(A_i,X_i) + I(Y_i > 0) p_{i,2}\phi(A_i,X_i)\\
    &=& \sum_{i=1}^{t-1} \left[\nabla_{\eta = \bar{\eta}}I(Y_i=0) \frac{1-\exp(-\mu_{i})}{p_{i}+(1-p_{i})\exp(-\mu_{i})}p_{i}(1-p_{i}) - I(Y_i > 0) p_{i}\right]\phi(A_i,X_i)(\eta_1-\eta_2)\\
    &=&\sum_{i=1}^{t-1} \left[\left(I(Y_i=0)\frac{(1-\exp(-\mu_i)(1-2p_i)p_i(1-p_i)}{p_i+(1-p_i)exp(-\mu_i)} - I(Y_i=0)\frac{(1-\exp(-\mu_i)^2p_i^2(1-p_i)^2)}{(p_i+(1-p_i)\exp(-\mu_i)^2)}\right.\right.\\
    &-&\left.\left.I(Y_i=0)p_i(1-p_i)\right) \phi(A_i,X_i)\phi(A_i,X_i)^\top\right]\\
    &=&H_t(\gamma)(\beta_2-\beta_1)
\end{eqnarray*}
Similarly, we have $ S_t(\gamma^*) - S_t(\hat{\gamma}_t) = S_t(\gamma^*) = H_t(\gamma)\cdot(\hat{\gamma}_t - \gamma^*)$. Under Condition 3, we have
$\left|\phi(a,X_t)^{\top} \hat{\gamma}_t-\phi(a,X_t)^{\top} \beta^*\right| \leq \kappa_{2,\min}^{-1}\left\|S_t \left(\gamma^*\right)\right\|_{G_t^{-1}} \left\|\phi(a,X_t)\right\|_{G_t^{-1}}$. 

Following the same argument in the proof for TS-NB, we have 
\begin{eqnarray*}
    \left\|S_t\left(\beta^*\right)\right\|_{G_t^{-1}} &\leq& Y_{\max }^2(d \log (T / d)-2 \log \delta)\\
    \left\|S_t \left(\gamma^*\right)\right\|_{G_t^{-1}} &\leq& (d \log (T / d)-2 \log \delta)
\end{eqnarray*}
for any $t \geq \tau$ with probability a least $1-\delta$. In this case, event $E_{1,t}^\beta$ is guaranteed to occur when $c_1^\beta = Y_{\max } \kappa_{1,\min}^{-1} \sqrt{d \log (T / d)-2 \log \delta}$, and  event $E_{1,t}^\gamma$ is guaranteed to occur when $c_1^\gamma = \kappa_{2,\min}^{-1} \sqrt{d \log (T / d)-2 \log \delta}$.
\end{proof}

\subsection{Proof of Lemma \ref{lem4}}
We prove Lemma \ref{lem4} in a similar way to Lemma \ref{lem3}. 
\begin{proof}
Let $\epsilon^{\beta}_t = I(Y_t = 0) \frac{(1-p_t)r\frac{r}{\mu_t+r}^{r-1}\frac{-r}{(\mu_t+r)^2}}{p_t+(1-p_t)\frac{r}{\mu_t+r}^r}\mu_t + I(Y_t>0)\frac{r}{\mu_t+r}(Y_t-\mu_t)$ and $\epsilon^{\gamma}_t = I(Y_t=0)\frac{1-\frac{r}{\mu_t+r}^r}{p_t+(1-p_t)\frac{r}{r+\mu_t}^r}p_t(1-p_t) - I(Y_t > 0)p_t$. Similarly, we can verify that $\epsilon^{\beta}_t$ and $\epsilon^{\gamma}_t$ are sub-Gaussian. Specifically, $\epsilon^{\beta}_t$ satisfies the sub-Gaussian conditions with $\sigma = Y_{\max}$, and $\epsilon^{\gamma}_t$ satisfies the sub-Gaussian conditions with $\sigma = 1$.

Hence, $S_t(\beta) = \sum_{i=1}^{t-1}\epsilon^{\beta}_i \phi(A_i,X_i)$ and $S_t(\gamma) = \sum_{i=1}^{t-1}\epsilon^{\gamma}_i\phi(A_i,X_i)$. For any $\beta_1, \beta_2 \in \mathbb{R}^d$, we have
\begin{eqnarray*}
    S_t(\beta_1) - S_t(\beta_2) = H_t(\beta)(\beta_2-\beta_1)
\end{eqnarray*}
where $H_t(\beta)$ is the left upper block $d\times d$ entry of the Hessian matrix of the negative ZINB log-likelihood.

Then, we have
\begin{eqnarray*}
    S_t(\beta^*) = H_t(\beta)\cdot(\hat{\beta}_t - \beta^*).
\end{eqnarray*}
Under Condition 4, we have $\kappa_{1,\min} G_t \preceq H_t$. Fixing intervention $a$, we have
\begin{eqnarray*}
    \left|\phi(a,X_t)^{\top} \hat{\beta}_t-\phi(a,X_t)^{\top} \beta^*\right| &\leq& \left\|\hat{\beta}_t-\beta^*\right\|_{G_t}\left\|\phi(a,X_t)\right\|_{G_t^{ -1}}=\sqrt{\left(\hat{\beta}_t-\beta^*\right)^{\top} G_t\left(\hat{\beta}_t-\beta^*\right)}\left\|\phi(a,X_t)\right\|_{G_t^{-1}}\\
    &=& \sqrt{S_t\left(\beta^*\right)^{\top} H_t^{-1}(\beta) G_t H_t^{-1}(\beta) S_t\left(\beta^*\right)}\left\|\phi(a,X_t)\right\|_{G_t^{-1}}\\
    &\leq& \kappa_{1,\min}^{-1} \left\|S_t\left(\beta^*\right)\right\|_{G_t^{-1}}\left\|\phi(a,X_t)\right\|_{G_t^{-1}}
\end{eqnarray*}
where the first inequality follows from the Cauchy-Schwarz inequality and others from the
above argument.

For any $\eta_1,\eta_2 \in \mathbb{R}^d$, we have
\begin{eqnarray*}
    S_t(\eta_1) - S_t(\eta_2) 
    =H_t(\gamma)(\beta_2-\beta_1)
\end{eqnarray*}
Similarly, we have $ S_t(\gamma^*) = H_t(\gamma)\cdot(\hat{\gamma}_t - \gamma^*)$. Under Condition 5, we have
$\left|\phi(a,X_t)^{\top} \hat{\gamma}_t-\phi(a,X_t)^{\top} \gamma^*\right| \leq \kappa_{2,\min}^{-1}\left\|S_t \left(\gamma^*\right)\right\|_{G_t^{-1}} \left\|\phi(a,X_t)\right\|_{G_t^{-1}}$. 

Following the same argument in the proof for TS-NB, we have 
\begin{eqnarray*}
    \left\|S_t\left(\beta^*\right)\right\|_{G_t^{-1}} &\leq& Y_{\max }^2(d \log (T / d)-2 \log \delta)\\
    \left\|S_t \left(\gamma^*\right)\right\|_{G_t^{-1}} &\leq& (d \log (T / d)-2 \log \delta)
\end{eqnarray*}
for any $t \geq \tau$ with probability a least $1-\delta$. In this case, event $E_{1,t}^\beta$ is guaranteed to occur when $c_1^\beta = Y_{\max } \kappa_{1,\min}^{-1} \sqrt{d \log (T / d)-2 \log \delta}$, and  event $E_{1,t}^\gamma$ is guaranteed to occur when $c_1^\gamma = \kappa_{2,\min}^{-1} \sqrt{d \log (T / d)-2 \log \delta}$.
\end{proof}

\subsection{Proof of Lemma \ref{lem5}}
\begin{proof}
    Let $\tilde{\Delta}_a=\phi(A_t^*, X_t)^{\top} \beta^*- \phi(a, X_t)^{\top} \beta^*$ and $c=c_1+c_2$. Let 
    \begin{eqnarray*}
        \bar{S}_t=\left\{a \in \mathcal{A}: c\left\|\phi(a, X_t)\right\|_{G_t^{-1}} \geq \tilde{\Delta}_a\right\}
    \end{eqnarray*}
    be the set of undersampled arms at time point $t$. Note that $A_t^* \in \bar{S}_t$ by definition. We define the set of sufficiently sampled arms as $S_t=\mathcal{A} \backslash \bar{S}_t$. Let $J_t=\arg \min _{a \in \bar{S}_t}\left\|\phi(a, X_t)\right\|_{G_t^{-1}}$ be the least uncertain undersampled arm at time point $t$. 

    In what follows, we assume that event $E_{1, t}$ occurs. At time point $t$ on event $E_{2, t}$,
    \begin{eqnarray*}
    \Delta_{A_t} & \leq & s \tilde{\Delta}_{A_t}=s\left(\tilde{\Delta}_{J_t}+\phi(J_t, X_t)^{\top} \beta^*-\phi(A_t, X_t)^{\top} \beta^*\right)\\
    &\leq& s\left(\tilde{\Delta}_{J_t}+\phi(J_t, X_t)^{\top} \tilde{\beta}_t-\phi(A_t, X_t)^{\top} \tilde{\beta}_t+c\left(\left\|\phi(A_t, X_t)\right\|_{G_t^{-1}}+\left\|\phi(J_t, X_t)\right\|_{G_t^{-1}}\right)\right) \\ & \leq& s c\left(\left\|\phi(A_t, X_t)\right\|_{G_t^{-1}}+2\left\|\phi(J_t, X_t)\right\|_{G_t^{-1}}\right),
    \end{eqnarray*}
    where the first inequality follows from the mean value theorem, the second follows from the definition of events $E_{1, t}$ and $E_{2, t}$, and the last follows from the definitions of $A_t$ and $J_t$.

    We take the expectation on both sides and get
    \begin{eqnarray*}
    \mathbb{E}_t\left[\Delta_{A_t}\right]&=&\mathbb{E}_t\left[\Delta_{A_t} \mathds{1}\left\{E_{2, t}\right\}\right]+\mathbb{E}_t\left[\Delta_{A_t} \mathds{1}\left\{\bar{E}_{2, t}\right\}\right]\\
        &\leq& s c \mathbb{E}_t\left[\left\|\phi(A_t, X_t)\right\|_{G_t^{-1}}+2\left\|\phi(J_t, X_t)\right\|_{G_t^{-1}}\right]+\Delta_{\max } \mathbb{P}_t\left(\bar{E}_{2, t}\right).
    \end{eqnarray*}
    Then, we can replace $\mathbb{E}_t\left[\left\|\phi(J_t, X_t)\right\|_{G_t^{-1}}\right]$ with $\mathbb{E}_t\left[\left\|\phi(A_t, X_t)\right\|_{G_t^{-1}}\right]$ by the observation below:
    \begin{eqnarray*}
        \mathbb{E}_t\left[\left\|\phi(A_t, X_t)\right\|_{G_t^{-1}}\right] \geq \mathbb{E}_t\left[\left\|\phi(A_t, X_t)\right\|_{G_t^{-1}} \mid A_t \in \bar{S}_t\right] \mathbb{P}_t\left(A_t \in \bar{S}_t\right) \geq\left\|\phi(J_t, X_t)\right\|_{G_t^{-1}} \mathbb{P}_t\left(A_t \in \bar{S}_t\right),
    \end{eqnarray*}
    where the last inequality follows from the definition of $J_t$. Further, on event $E_{1, t}$, we have
    \begin{eqnarray*}
\mathbb{P}_t\left(A_t \in \bar{S}_t\right) & \geq& \mathbb{P}_t\left(\exists a \in \bar{S}_t: \phi(a, X_t)^{\top} \tilde{\beta}_t>\max _{b \in S_t} \phi(b, X_t)^{\top} \tilde{\beta}_t\right) \geq \mathbb{P}_t\left(\phi(A_t^*, X_t)^{\top} \tilde{\beta}_t > \max_{b \in S_t} \phi(b, X_t)^{\top} \tilde{\beta}_t\right) \\
& \geq& \mathbb{P}_t\left(\phi(A_t^*, X_t)^{\top} \tilde{\beta}_t > \max_{b \in S_t} \phi(b, X_t)^{\top} \tilde{\beta}_t, E_{2, t} \text { occurs }\right)\\
& \geq& \mathbb{P}_t\left(\phi(A_t^*, X_t)^{\top} \tilde{\beta}_t > \phi(A_t^*, X_t)^{\top} \beta^*, E_{2, t} \text { occurs }\right) \\
& \geq& \mathbb{P}_t\left(\phi(A_t^*, X_t)^{\top} \tilde{\beta}_t>\phi(A_t^*, X_t)^{\top} \beta^*\right)-\mathbb{P}_t\left(\bar{E}_{2, t}\right)\\
& \geq& \mathbb{P}_t\left(\phi(A_t^*, X_t)^{\top} \tilde{\beta}_t-\phi(A_t^*, X_t)^{\top} \hat{\beta}_t>c_1\left\|\phi(A_t^*, X_t)\right\|_{G_t^{-1}}\right)-\mathbb{P}_t\left(\bar{E}_{2, t}\right) .
    \end{eqnarray*}
The fourth inequality holds because on event $E_{1, t} \cap E_{2, t}$,
$$
\phi(b, X_t)^{\top} \tilde{\beta}_t \leq \phi(b, X_t)^{\top} \beta^*+c\left\|\phi(b, X_t)\right\|_{G_t^{-1}} < \phi(b, X_t)^{\top} \beta^*+\tilde{\Delta}_b = \phi(A_t^*, X_t)^{\top} \beta^*
$$
holds for any $b \in S_t$. The last inequality holds because $\phi(A_t^*, X_t)^{\top} \beta^* \leq \phi(A_t^*, X_t)^{\top} \hat{\beta}_t+c_1\left\|\phi(A_t^*, X_t)\right\|_{G_t^{-1}}$ holds on event $E_{1, t}$. Finally, we use the definitions of $p_2$ and $p_3$ to complete the proof.
\end{proof}

\subsection{Proof of Lemma \ref{lem6}}
\begin{proof}
    By the design of TS-NB,
    \begin{eqnarray*}
        \phi(a,X_t)^{\top} \tilde{\beta}_t-\phi(a,X_t)^{\top} \hat{\beta}_t \sim \mathcal{N}\left(0, \alpha^2\|\phi(a,X_t)\|_{\mathcal{I}_t^{-1}(\hat{\beta}_t)}^2\right)
    \end{eqnarray*}
    for any $a \in \mathcal{A}$ and any $t\in [T]$, where matrix $\mathcal{I}_t(\hat{\beta}_t) = \mathbb{E}(H_t(\hat{\beta}_t))$ is the Fisher information at time point $t$. 
    
    Let $U = \phi(a,X_t)^{\top} \tilde{\beta}_t-\phi(a,X_t)^{\top} \hat{\beta}_t$. Since $U \sim \mathcal{N}\left(0, \alpha^2\|\phi(a,X_t)\|_{\mathcal{I}_t^{-1}(\hat{\beta}_t)}^2\right)$ is a normal random variable, we have
    \begin{eqnarray*}
&  \mathbb{P}_t\left(U \geq \alpha\|\phi(a,X_t)\|_{\mathcal{I}_t^{-1}(\hat{\beta}_t)}\right)& \geq 0.15, \\
& \mathbb{P}_t\left(U \geq \pi\|\phi(a,X_t)\|_{\mathcal{I}_t^{-1}(\hat{\beta}_t)}\right)& \leq \exp \left[-\frac{\pi^2}{2 \alpha^2}\right],
    \end{eqnarray*}
    for any $c > 0$.

    According to Condition 1, $\mathcal{I}_t(\hat{\beta}_t) \preceq \kappa_{\max}G_t$. Thus, we have
    \begin{eqnarray*}
0.15 & \leq& \mathbb{P}_t\left(U \geq \alpha\|\phi(a,X_t)\|_{\mathcal{I}_t^{-1}(\hat{\beta}_t)}\right) \\
& \leq& \mathbb{P}_t\left(U \geq \alpha \sqrt{\kappa_{\max }^{-1}}\|\phi(a,X_t)\|_{G_t^{-1}}\right).
    \end{eqnarray*}
    For $\alpha = c_1\sqrt{\kappa_{\max}}$ and $a = A_t^*$, we get that event $E_{3,t}$ occurs with probability at least $0.15$.

    According to Condition 1, $\mathcal{I}_t(\hat{\beta}_t) \succeq \kappa_{\min} G_t$. Therefore,
    \begin{eqnarray*}
\exp \left[-\frac{\pi^2}{2 \alpha^2}\right] & \geq& \mathbb{P}_t\left(U \geq \pi\|\phi(a,X_t)\|_{\mathcal{I}_t^{-1}(\hat{\beta}_t)}\right) \\
& \geq& \mathbb{P}_t\left(U \geq \pi \sqrt{\kappa_{\min }^{-1}}\|\phi(a,X_t)\|_{G_t^{-1}}\right) .
    \end{eqnarray*}
    We set $\pi = \alpha\sqrt{2\log(KT)}$ and $a = A_t^*$. Hence, we have
    \begin{eqnarray*}
        \mathbb{P}_t\left(U \geq c_1\sqrt{2\kappa_{\min}^{-1}\kappa_{\max}\log(KT)}\|\phi(a,X_t)\|_{G_t^{-1}}\right) \leq \frac{1}{KT}.
    \end{eqnarray*}
        By the union bound over all K arms, we get that event $\bar{E}_{2, t}$ occurs with probability at most $1/T$ when $c_2 = c_1\sqrt{2\kappa_{\min}^{-1}\kappa_{\max}\log(KT)}$. 
\end{proof}

\subsection{Proof of Lemma \ref{lem7}}
\begin{proof}
Recall that
$\Delta_{a}=R_t(A_t^*, \gamma^*, \beta^*) - R_t(a,\gamma^*,\beta^*)=[1-\operatorname{sigmoid}(\phi(A_t^*,X_t)^\top\gamma^*)]\\\exp\{\phi(A_t^*,X_t)^\top\beta^*\} - [1-\operatorname{sigmoid}(\phi(a,X_t)^\top\gamma^*)]\exp\{\phi(a,X_t)^\top\beta^*\}$.

Let $\tilde{\Delta}_a^\beta = \phi(A_t^*,X_t)^\top\beta^* - \phi(a,X_t)^\top\beta^*$ and $\tilde{\Delta}_a^\gamma =  \phi(a,X_t)^\top\gamma^* - \phi(A_t^*,X_t)^\top\gamma^*$. Denote $c^\beta = c_1^\beta+c_2^\beta$ and $c^\gamma = c_1^\gamma + c_2^\gamma$. Let 
\begin{align*}
    \bar{S}_t = \left\{a\in\mathcal{A}:c^\beta\|\phi(a,X_t)\|_{G_t^{-1}}\geq \tilde{\Delta}_a^\beta \operatorname{and} c^\gamma\|\phi(a,X_t)\|_{G_t^{-1}}\geq \tilde{\Delta}_a^\gamma\right\}
\end{align*}
be the set of undersampled actions at time point $t$. Note that $A_t^* \in \bar{S}_t$ by definition. We define the set of sufficiently sampled arms as $S_t=\mathcal{A} \backslash \bar{S}_t$. Let $J_t = \arg \min_{a \in \bar{S}_t} \|\phi(a,X_t)\|_{G_t^{-1}}$ be the least uncertain undersampled action at time point $t$. 

In what follows, we assume that event $E_{1,t}^\beta$ and $E_{1,t}^\gamma$ occurs. Let $f(a,b) = \frac{1}{1+\exp(a)}\exp(b)$.
At time point $t$ on event $E_{2,t}^\beta$ and $E_{2,t}^\gamma$,
\begin{align*}
    \Delta_{A_t} &=  \nabla f(\bar{\phi}^\top\beta^*,\bar{\phi}^\top\gamma^*) \left(\begin{array}{cc}
    \tilde{\Delta}_{A_t}^\beta  \\
     - \tilde{\Delta}_{A_t}^\gamma
    \end{array}
    \right) =  \nabla f(\bar{\phi}^\top\beta^*,\bar{\phi}^\top\gamma^*) \left(\begin{array}{cc}
    \tilde{\Delta}_{J_t}^\beta + \phi(J_t,X_t)^\top\beta^* -\phi(A_t,X_t)^\top\beta^* \\
      \tilde{\Delta}_{J_t}^\gamma + \phi(J_t,X_t)^\top\gamma^* -\phi(A_t,X_t)^\top\gamma^*
    \end{array}
    \right)\\
    &\leq \nabla f(\bar{\phi}^\top\beta^*,\bar{\phi}^\top\gamma^*) 
    \left(\begin{array}{cc}
    \tilde{\Delta}_{J_t}^\beta + \phi(J_t,X_t)^\top\tilde{\beta}_t -\phi(A_t,X_t)^\top \tilde{\beta}_t + c^\beta(\|\phi(A_t,X_t)\|_{G_t^{-1}} + \|\phi(J_t,X_t)\|_{G_t^{-1}}) \\
      \tilde{\Delta}_{J_t}^\gamma + \phi(J_t,X_t)^\top\tilde{\gamma}_t -\phi(A_t,X_t)^\top\tilde{\gamma}_t + c^\gamma(\|\phi(A_t,X_t)\|_{G_t^{-1}} + \|\phi(J_t,X_t)\|_{G_t^{-1}}) 
    \end{array}
    \right)\\
    &\leq  \nabla f(\bar{\phi}^\top\beta^*,\bar{\phi}^\top\gamma^*) 
    \left(\begin{array}{cc}
    c^\beta(\|\phi(A_t,X_t)\|_{G_t^{-1}} + 2\|\phi(J_t,X_t)\|_{G_t^{-1}}) \\
    c^\gamma(\|\phi(A_t,X_t)\|_{G_t^{-1}} + 2\|\phi(J_t,X_t)\|_{G_t^{-1}}) 
    \end{array}
    \right)\\
    &\leq \sup_{\|o_1\|\leq1,\|o_2\|\leq1} \|\nabla f(o_1,o_2)\| \left\|\begin{array}{cc}
    c^\beta(\|\phi(A_t,X_t)\|_{G_t^{-1}} + 2\|\phi(J_t,X_t)\|_{G_t^{-1}}) \\
    c^\gamma(\|\phi(A_t,X_t)\|_{G_t^{-1}} + 2\|\phi(J_t,X_t)\|_{G_t^{-1}}) 
    \end{array}\right\|\\
    &\leq q \left[ c^\beta(\|\phi(A_t,X_t)\|_{G_t^{-1}} + 2\|\phi(J_t,X_t)\|_{G_t^{-1}}) + c^\gamma(\|\phi(A_t,X_t)\|_{G_t^{-1}} + 2\|\phi(J_t,X_t)\|_{G_t^{-1}}) \right]
\end{align*}
where $q = \sqrt{2}e$. Here, the first equality follows from the mean value theorem, the first inequality follows from the definition of events $E_{1,t}^\beta$, $E_{1,t}^\gamma$, $E_{2,t}^\beta$, and $E_{2,t}^\gamma$, the second inequality follows from the definitions of $A_t$ and $J_t$, and the rest follows from the Cauchy-Schwarz inequality. 

By taking expectations on both sides, we have
\begin{align*}
    \mathbb{E}_t(\Delta_{A_t}) &= \mathbb{E}_t\left[\Delta_{A_t} \mathbbm{1}\left\{E_{2, t}^\beta\cap E_{2, t}^\gamma\right\}\right]+\mathbb{E}_t\left[\Delta_{A_t} \mathbbm{1}\left\{\overline{E_{2, t}^\beta \cap E_{2, t}^\gamma}\right\}\right]\\
    &\leq s \left[ c^\beta(\|\phi(A_t,X_t)\|_{G_t^{-1}} + 2\|\phi(J_t,X_t)\|_{G_t^{-1}}) + c^\gamma(\|\phi(A_t,X_t)\|_{G_t^{-1}} + 2\|\phi(J_t,X_t)\|_{G_t^{-1}}) \right] \\
    &+ \Delta_{\max}(1-\mathbb{P}_t(E_{2, t}^\beta)\mathbb{P}_t(E_{2, t}^\gamma)).
\end{align*}

Further, we can replace $\mathbb{E}_t[\|\phi(J_t,X_t)\|_{G_t^{-1}}]$ with $\mathbb{E}[\|\phi(A_t,X_t)\|_{G_t^{-1}}]$ by the observation below:
\begin{align*} \mathbb{E}_t\left[\left\|\phi(A_t,X_t)\right\|_{G_t^{-1}}\right] \geq \mathbb{E}_t\left[\left\|\phi(A_t,X_t)\right\|_{G_t^{-1}} \mid A_t \in \bar{S}_t\right] \mathbb{P}_t\left(A_t \in \bar{S}_t\right) \geq\left\|\phi(J_t,X_t)\right\|_{G_t^{-1}} \mathbb{P}_t\left(A_t \in \bar{S}_t\right)
\end{align*}
where the last inequality follows from the definition of $J_t$. Further, on event $E_{1,t}^\beta$ and $E_{1,t}^\gamma$, we have
\begin{align*}
    &\mathbb{P}_t\left(A_t \in \bar{S}_t\right)\\
    &\geq \mathbb{P}_t\left(\exists a \in \bar{S}_t: \left[\frac{1}{1+\exp(\phi(a,X_t)^{\top} \tilde{\gamma}_t)}\right]\exp(\phi(a,X_t)^{\top} \tilde{\beta}_t)>\max _{b \in S_t} \left[\frac{1}{1+\exp(\phi(b,X_t)^{\top} \tilde{\gamma}_t)}\right]\exp(\phi(b,X_t)^{\top} \tilde{\beta}_t)\right)\\
    &\geq \mathbb{P}_t\left(\left[\frac{1}{1+\exp(\phi(A_t^*,X_t)^{\top} \tilde{\gamma}_t)}\right]\exp(\phi(A_t^*,X_t)^{\top} \tilde{\beta}_t)>\max _{b \in S_t} \left[\frac{1}{1+\exp(\phi(b,X_t)^{\top} \tilde{\gamma}_t)}\right]\exp(\phi(b,X_t)^{\top} \tilde{\beta}_t)\right)\\
    &\geq \mathbb{P}_t\left(\left[\frac{1}{1+\exp(\phi(A_t^*,X_t)^{\top} \tilde{\gamma}_t)}\right]\exp(\phi(A_t^*,X_t)^{\top} \tilde{\beta}_t)>\max _{b \in S_t} \left[\frac{1}{1+\exp(\phi(b,X_t)^{\top} \tilde{\gamma}_t)}\right]\exp(\phi(b,X_t)^{\top} \tilde{\beta}_t)\right., \\
    & \left.E_{2,t}^\beta \operatorname{and} E_{2,t}^\gamma \operatorname{occurs}\right)\\
    &\geq \mathbb{P}_t\left(\left[\frac{1}{1+\exp(\phi(A_t^*,X_t)^{\top} \tilde{\gamma}_t)}\right]\exp(\phi(A_t^*,X_t)^{\top} \tilde{\beta}_t)> \left[\frac{1}{1+\exp(\phi(A_t^*,X_t)^{\top} \tilde{\gamma}_t)}\right]\exp(\phi(A_t^*,X_t)^{\top} \tilde{\beta}_t)\right., \\
    & \left.E_{2,t}^\beta \operatorname{and} E_{2,t}^\gamma \operatorname{occurs}\right)\\
    & \geq \mathbb{P}_t\left(\left[\frac{1}{1+\exp(\phi(A_t^*,X_t)^{\top} \tilde{\gamma}_t)}\right]\exp(\phi(A_t^*,X_t)^{\top} \tilde{\beta}_t)> \left[\frac{1}{1+\exp(\phi(A_t^*,X_t)^{\top} \gamma_t^*)}\right]\exp(\phi(A_t^*,X_t)^{\top} \beta_t^*)\right)\\
    &- (1-\mathbb{P}_t(E_{2,t}^\beta)\mathbb{P}_t(E_{2,t}^\gamma)) \\
    &\geq \mathbb{P}_t(\phi(A_t^*,X_t)^{\top} \tilde{\beta}_t > \phi(A_t^*,X_t)^{\top} \beta_t^*) \times \mathbb{P}_t(\phi(A_t^*,X_t)^{\top} \tilde{\gamma}_t < \phi(A_t^*,X_t)^{\top} \gamma_t^*) \\
    &-1+ \mathbb{P}_t(E_{2,t}^\beta)\mathbb{P}_t(E_{2,t}^\gamma)\\
    &\geq \mathbb{P}_t(\phi(A_t^*,X_t)^{\top} \tilde{\beta}_t > \phi(A_t^*,X_t)^{\top} \hat{\beta}_t + c_1^\beta\|\phi(A_t^*,X_t)\|_{G_t^{-1}}) \times\\
    & \   \mathbb{P}_t(\phi(A_t^*,X_t)^{\top} \tilde{\gamma}_t < \phi(A_t^*,X_t)^{\top} \hat{\gamma}_t - c_1^\gamma \|\phi(A_t^*,X_t)\|_{G_t^{-1}}) - \mathbb{P}_t(\bar{E}_{2,t}^\beta) - \mathbb{P}_t(\bar{E}_{2,t}^\gamma) + \mathbb{P}_t(\bar{E}_{2,t}^\beta)\mathbb{P}_t(\bar{E}_{2,t}^\gamma)
\end{align*}

The fourth inequality holds because on event $E_{1,t}^\beta$, $E_{1,t}^\gamma$, $E_{2,t}^\beta$, and $E_{2,t}^\gamma$,
\begin{align*}
 &\phi(b,X_t)^{\top} \tilde{\beta}_t \leq \phi(b,X_t)^{\top} \beta_t^* + c^\beta\|\phi(b,X_t)^{\top}\|_{G_t^{-1}} \leq \phi(b,X_t)^{\top} \beta_t^* + \tilde{\Delta}_b^\beta = \phi(A_t^*,X_t)^{\top} \beta_t^*,\\
 &\phi(b,X_t)^{\top} \tilde{\gamma}_t \geq \phi(b,X_t)^{\top} \gamma_t^* - c^\gamma\|\phi(b,X_t)^{\top}\|_{G_t^{-1}} \geq \phi(b,X_t)^{\top} \gamma_t^* - \tilde{\Delta}_b^\gamma = \phi(A_t^*,X_t)^{\top} \gamma_t^*,
\end{align*}
holds for any $b \in S_t$. The last inequality holds because 
\begin{align*}
    &\phi(A_t^*,X_t)^{\top} \beta_t^* \leq \phi(A_t^*,X_t)^{\top} \hat{\beta}_t + c_1^\beta\|\phi(A_t^*,X_t)\|_{G_t^{-1}}\\
    &\phi(A_t^*,X_t)^{\top} \gamma_t^* \geq \phi(A_t^*,X_t)^{\top} \hat{\gamma}_t - c_1^\gamma \|\phi(A_t^*,X_t)\|_{G_t^{-1}}
\end{align*}
hold on event $E_{1,t}^\beta$ and $E_{1,t}^\gamma$. Finally, we use the definitions of $p_2^\beta$, $p_2^\gamma$, $p_3^\beta$, and $p_3^\gamma$ to finish the proof.
\end{proof}

\subsection{Proof of Lemma \ref{lem8}}
\begin{proof}
    By the design of TS-ZIP,
    \begin{align*}
        &\phi(a,X_t)^\top\tilde{\beta}_t - \phi(a,X_t)^\top\hat{\beta}_t \sim \mathcal{N}(0, {\alpha_{\beta}}^2\|\phi(a,X_t)^2\|_{\mathcal{I}_t^{-1}(\hat{\beta}_t)})\\
        &\phi(a,X_t)^\top\tilde{\gamma}_t - \phi(a,X_t)^\top\hat{\gamma}_t \sim \mathcal{N}(0, {\alpha_{\gamma}}^2\|\phi(a,X_t)^2\|_{\mathcal{I}_t^{-1}(\hat{\gamma}_t)})
    \end{align*}
    for any $a \in \mathcal{A}$ and any $t \in [T]$, where $\mathcal{I}_t(\hat{\beta}_t) = \mathbb{E}(H_t(\hat{\beta}_t))$ and $\mathcal{I}_t(\hat{\gamma}_t) = \mathbb{E}(H_t(\hat{\gamma}_t))$ are the fisher information at time point $t$.

    Let $U^\beta = \phi(a,X_t)^\top\tilde{\beta}_t - \phi(a,X_t)^\top\hat{\beta}_t$. Since $U^\beta \sim \mathcal{N}(0,\alpha_{\beta}^2\|\phi(a,X_t)^2\|_{\mathcal{I}_t^{-1}(\hat{\beta}_t)})$ is a normal random variable, we have
    \begin{align*}
        &\mathbb{P}_t(U^\beta\geq \alpha_\beta\|\phi(a,X_t)\|_{\mathcal{I}_t^{-1}(\hat{\beta}_t)}) \geq 0.15\\
        &\mathbb{P}_t(U^\beta\geq \pi\|\phi(a,X_t)\|_{\mathcal{I}_t^{-1}(\hat{\beta}_t)}) \leq \exp\left[-\frac{\pi^2}{2\alpha_{\beta}^2}\right]\\
    \end{align*}
    for any $\pi > 0$. 

    According to Condition 2, $\mathcal{I}_t(\hat{\beta}_t) \preceq \kappa_{1,\max}G_t$, we have
    \begin{align*}
        \mathbb{P}_t(U^\beta \geq \alpha_\beta \sqrt{\kappa_{1,\max}^{-1}}\|\phi(a,X_t)\|_{G_t^{-1}}) \geq 0.15.
    \end{align*}
    For $\alpha_\beta = c_1^\beta\sqrt{\kappa_{1,\max}}$ and $a=A_t^*$, we have that event $E_{3,t}^\beta$ occurs with probability at least $0.15$. 

    Furthermore, according to Condition 2, $\mathcal{I}_t(\hat{\beta}_t) \succeq \kappa_{1,\min}G_t$. Hence,
    \begin{align*}
        \mathbb{P}_t(U^\beta \geq \pi\sqrt{\kappa_{1,\min}^{-1}}\|\phi(a,X_t)\|_{G_t^{-1}}) \leq \exp\left[-\frac{\pi^2}{2\alpha_{\beta}^2}\right].
    \end{align*}
    By setting $\pi = \alpha_\beta\sqrt{2\log(KT)}$ and $a = A_t^*$, we have
    \begin{align*}
        \mathbb{P}_t\left(U \geq c_1^\beta \sqrt{2 \kappa_{1,\min }^{-1} \kappa_{1,\max } \log (K T)}\left\|\phi(a,X_t)\right\|_{G_t^{-1}}\right) \leq \frac{1}{K T}
    \end{align*}
    By the union bound over all $K$ arms, we get that event $\bar{E}_{2,t}^\beta$ occurs with probability at most $1/T$, with $c_2^\beta = c_1^\beta\sqrt{2 \kappa_{1,\min }^{-1} \kappa_{1,\max } \log (K T)}$.

    Similarly, under Condition 3, we have
    \begin{align*}
        &\mathbb{P}_t(U^\gamma \leq -\alpha^\gamma \sqrt{\kappa_{2,\max}^{-1}}\|\phi(a,X_t)\|_{G_t^{-1}}) \geq 0.15.\\
        &\mathbb{P}_t(U^\gamma \geq c_1^\gamma \sqrt{2 \kappa_{2,\min }^{-1} \kappa_{2,\max } \log (K T)}\|\phi(a,X_t)\|_{G_t^{-1}}) \leq \frac{1}{K T}.
    \end{align*}
     For $\alpha^\gamma = c_1^\gamma\sqrt{\kappa_{2,\max}}$ and $a=A_t^*$, we have that event $E_{3,t}^\gamma$ occurs with probability at least $0.15$. 
    By the union bound over all $K$ arms, we get that event $\bar{E}_{2,t}^\gamma$ occurs with probability at most $1/T$, with $c_2^\gamma = c_1^\gamma\sqrt{2 \kappa_{2,\min }^{-1} \kappa_{2,\max } \log (K T)}$.
\end{proof}

\section{Additional Details for the Drink Less Study}
\xueqing{
In this section, we provide additional details on the Drink Less study, as well as additional results on different settings. }
\subsection{Covariates Information}
\xueqing{
In the algorithms and environment models, we use a subset of the baseline and contextual data that are associated with the outcome. Table \ref{tab:variables} shows the features available in the Drink Less MRT dataset. 
\begin{table}[ht!]
\centering
\caption{Summary of variables used in the Drink Less study. The features used in $X_{n,t}$ and $S_{n,t}$ are denoted via a $\checkmark$ in the corresponding column, otherwise X. }
\label{tab:variables}
\begin{tabular}{@{}llcc@{}}
\toprule
Feature              & Description            & $X_{n,t}$ & $S_{n,t}$ \\ \midrule
Intercept           & Constant term          & $\checkmark$     & $\checkmark$    \\
Age                 & Age in years           & $\checkmark$     & X     \\
Gender              & Male (0), Female (1)   & $\checkmark$     & X     \\
AUDIT score         & Range 0--40            & $\checkmark$    & X     \\
Days since download & Range 0--30 days       & $\checkmark$    & $\checkmark$     \\ \bottomrule
\end{tabular}
\end{table}

Here, $S_{n,t}$ includes variables that interact with the intervention, potentially modifying the effect of sending a notification. We have included ``Days since download" based on insights from another study \citep{liu2023incorporating}.}

\subsection{Additional results}
\xueqing{
As stated in the manuscript, we conducted simulations to evaluate the algorithms with  $p_{\min} = 0$ and $p_{\max} = 1$. Figure \ref{fig:mainfig4} reveals that without probability clipping, the proposed TS-Count algorithms, notably TS-Poisson and TS-NB, outperform others in terms of regret. The comparison shows minimal differences in regret between scenarios with and without probability clipping. This suggests that, with probability clipping, the proposed algorithms maintain consistent regret bounds when compared to a clipped oracle.

\begin{figure}[ht!]
    \centering
    \subfloat[OP]{\includegraphics[width=0.5\textwidth]{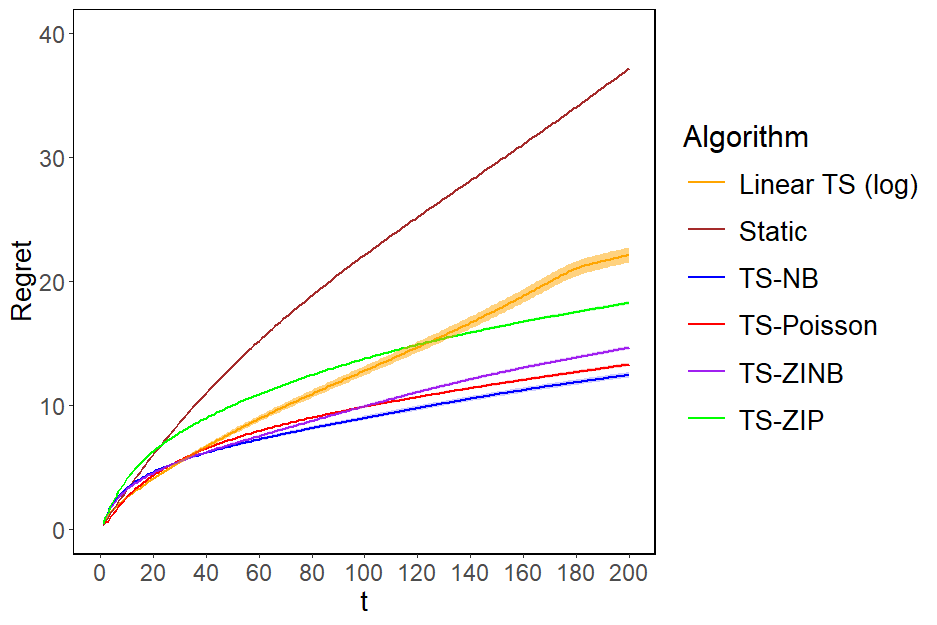}\label{fig:subfig11}}
    \hfill
    \subfloat[ZIOP]{\includegraphics[width=0.5\textwidth]{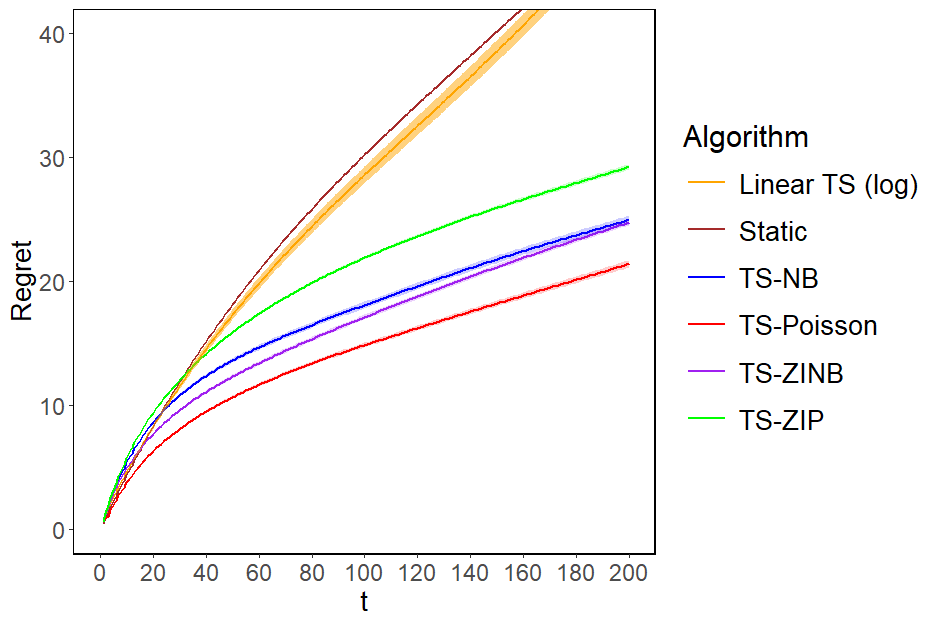}\label{fig:subfig13}}
    
    \caption{Results of the compared algorithms under the Drink Less study with $p_{\min} = 0$ and $p_{\max} = 1$.  The results are calculated from $100$ replications of the experiment. The solid lines indicate the mean values, while the shaded bands represent the standard error bounds across the independent replications.}
    \label{fig:mainfig4}
\end{figure}

\begin{figure}[ht!]
    \centering
    \subfloat[OP]{\includegraphics[width=0.5\textwidth]{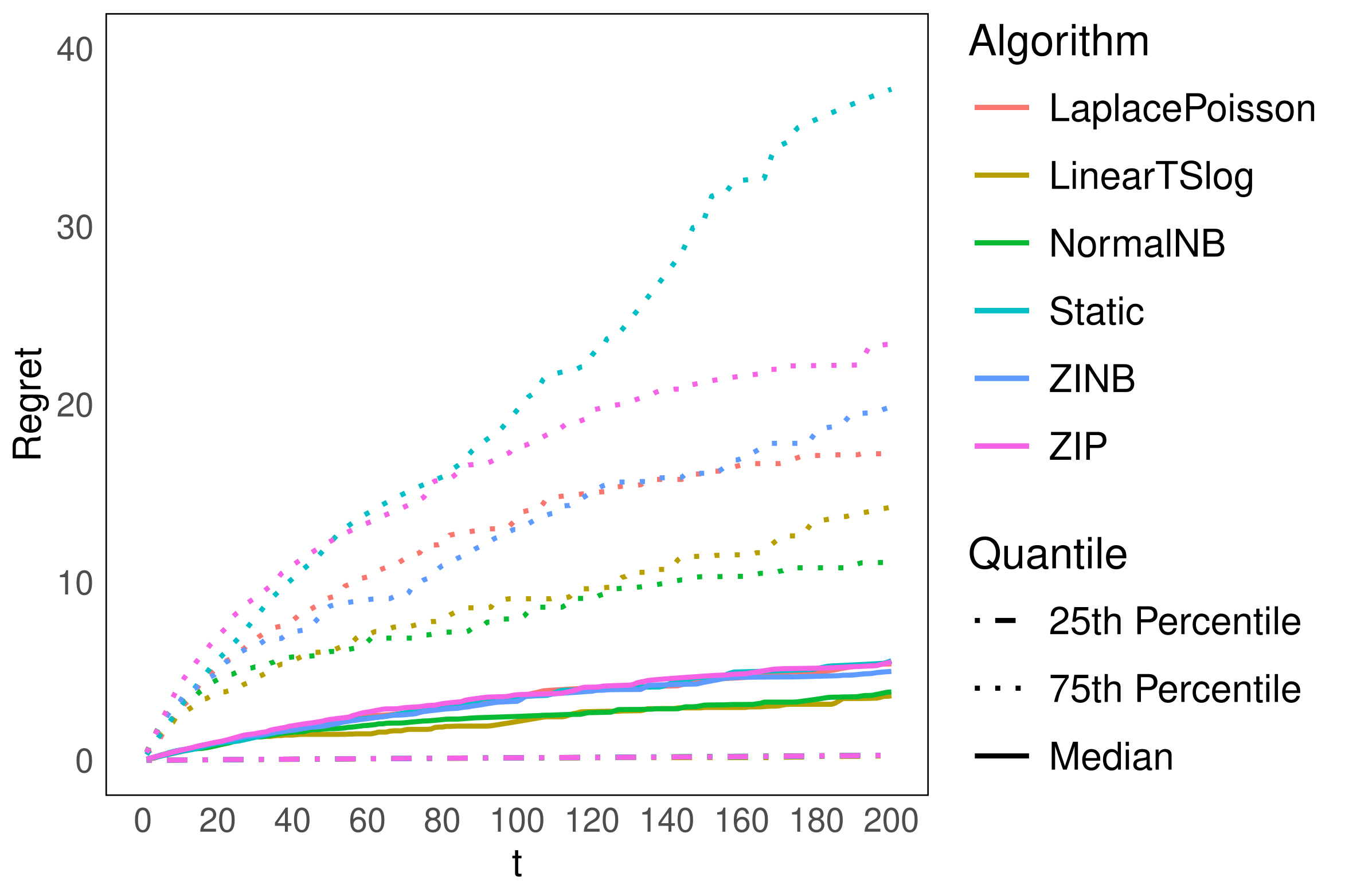}\label{fig:subfig1_app}}
    \hfill
    \subfloat[ZIOP]{\includegraphics[width=0.5\textwidth]{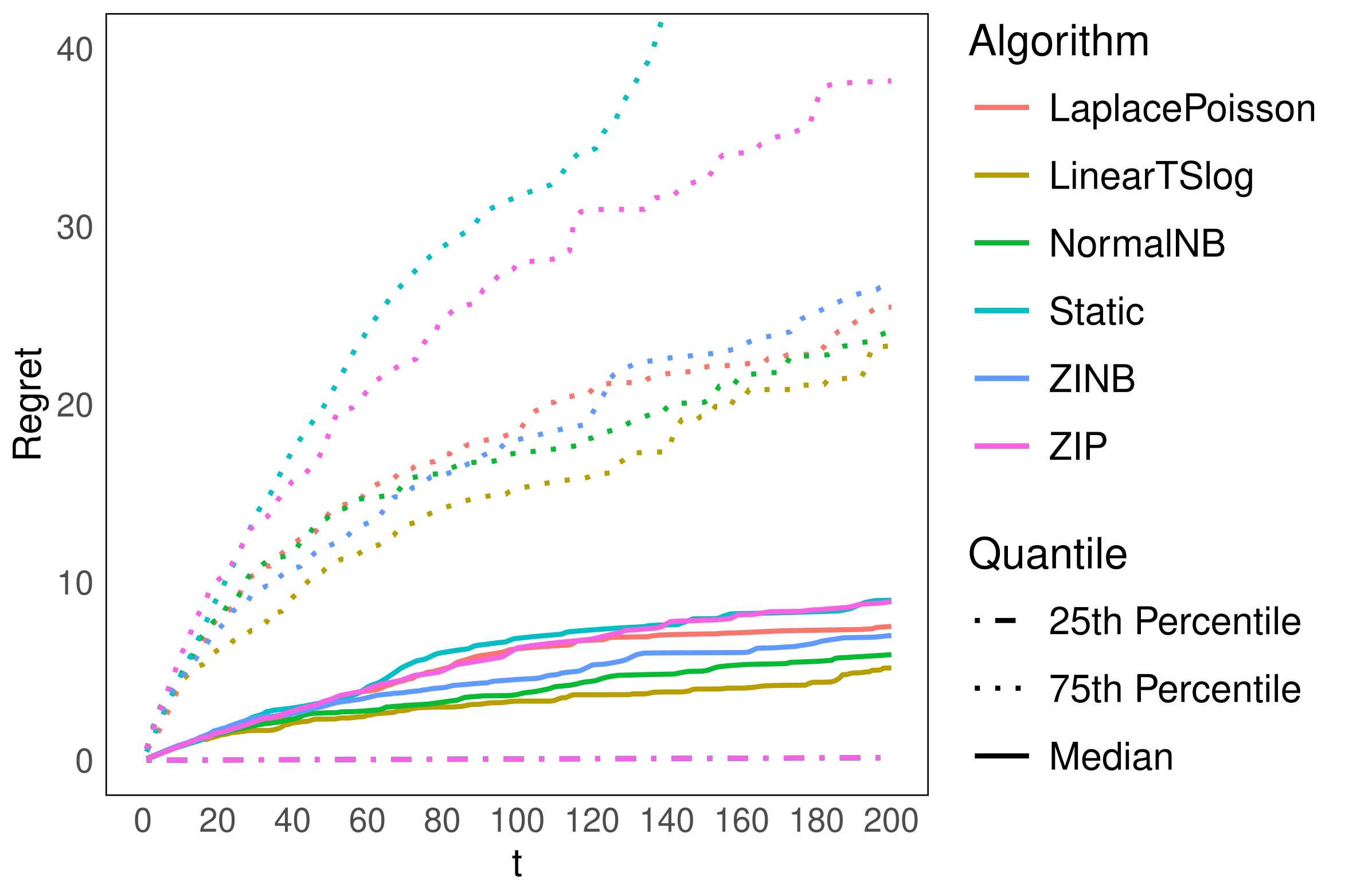}\label{fig:subfig2_app}}
    
    \caption{Quantile plots of the compared algorithms under the Drink Less study with $p_{\min} = 0.01$ and $p_{\max} = 0.99$.  The results are calculated from $1$ replication of the experiment. }
    \label{fig:mainfig8}
\end{figure}
Next, we want to study the individual variation in regret across different users. We utilized the same experimental setup as described in Section 7.3 of the manuscript, employing both overdispersed Poisson and zero-inflated overdispersed Poisson models for reward generation. 

In Supplement C, we include two interactive plots illustrating user-specific regret for 349 users across a single experiment replication. The plot named "regret\_per\_user\_op\_clipped.html" represents results from an overdispersed Poisson model, while "regret\_per\_user\_ziop\_clipped.html" is derived from a zero-inflated overdispersed Poisson model.

Furthermore, we depict the distribution of user regret across different quantiles in Figure \ref{fig:mainfig8}. This figure offers a visual analysis of the spread of regret, providing a deeper understanding of its distribution among users.}

\end{document}